\documentclass[review,onefignum,onetabnum]{siamart220329}
\nolinenumbers

\usepackage[subpreambles=true]{standalone}
\usepackage{import}
\usepackage{import}
\usepackage{braket,amsfonts}

\usepackage{array}

\usepackage[caption=false]{subfig}

\usepackage{caption}

\usepackage{pgfplots}

\newsiamremark{remark}{Remark}
\newsiamthm{assumption}{Assumption}
\newsiamthm{defn}{Definition}
\newsiamremark{hypothesis}{Hypothesis}
\crefname{hypothesis}{Hypothesis}{Hypotheses}

\usepackage{algorithmic}
\usepackage{graphicx,epstopdf}
\Crefname{ALC@unique}{Line}{Lines}
\usepackage{adjustbox}
\usepackage{comment}
\usepackage{booktabs}       
\usepackage{nicefrac}       
\usepackage{microtype}      
\usepackage{supertabular}
\usepackage{wrapfig}
\usepackage{tikz}
\usepackage{float}
\usepackage{mathtools}
\usepackage{multirow}

\newcommand{\vect}[1]{\mathbf{#1}}

\newcommand{\PB}[1]{{\color{red} Prasanna: {#1}}}

\newcommand{\norm}[1]{\left\lVert#1\right\rVert}
\usepackage{amsopn}

\usepackage{xspace}
\usepackage{bold-extra}
\usepackage[most]{tcolorbox}

\colorlet{texcscolor}{blue!50!black}
\colorlet{texemcolor}{red!70!black}
\colorlet{texpreamble}{red!70!black}
\colorlet{codebackground}{black!25!white!25}

\usepackage{lipsum}
\usepackage{amsfonts}
\usepackage{graphicx}
\usepackage{epstopdf}

\ifpdf
  \DeclareGraphicsExtensions{.eps,.pdf,.png,.jpg}
\else
  \DeclareGraphicsExtensions{.eps}
\fi

\headers{Graph Continual Learning}{R.~Krishnan, and P.~Balaprakash}

\title{Learning Continually on a Sequence of Graphs -- The Dynamical System Way\thanks{Submitted to the editors 5/14/2023.
}}

\author{Krishnan Raghavan\thanks{Mathematics and Computer Science, Argonne National Laboratory~(\email{ kraghavan@anl.gov})} \and Prasanna Balaprakash\thanks{Oak Ridge National Laboratory
  (\email{ pbalapra@ornl.gov}).}}

\usepackage{amsopn}

\ifpdf
\hypersetup{
  pdftitle={An Example Article},
  pdfauthor={D. Doe, P. T. Frank, and J. E. Smith}
}
\fi

\usepackage{pgfplots}
\pgfplotsset{compat=1.17}
\begin{document}
\maketitle
\begin{abstract}
Continual learning~(CL) is a field concerned with learning a series of inter-related task with the tasks typically defined in the sense of either regression or classification. In recent years, CL has been studied extensively when these tasks are defined using Euclidean data-- data, such as images, that can be described by a set of vectors in an n-dimensional real space. However, the literature is quite sparse, when the data  corresponding to a CL task is nonEuclidean-- data , such as graphs, point clouds or manifold, where the notion of similarity in the sense of Euclidean metric does not hold. For instance, a graph is described by a tuple of vertices and edges and similarities between two graphs is not well defined through a Euclidean metric. Due to this fundamental nature of the data, developing CL for nonEuclidean data presents several theoretical and methodological challenges.  In particular, CL for graphs requires explicit modelling of nonstationary behavior of vertices and edges and their effects on the learning problem. Therefore, in this work, we develop a adaptive dynamic programming viewpoint for CL with graphs. In this work, we formulate a two-player sequential game between the act of learning new tasks~(generalization) and remembering previously learned tasks~(forgetting). We prove mathematically the existence of a solution to the game and demonstrate convergence to the solution of the game. Finally, we demonstrate the efficacy of our method on a number of graph benchmarks with a comprehensive ablation study while establishing state-of-the-art performance.
\end{abstract}




\section{Introduction}
The problem of Continual learning~(CL) when the tasks are made up of a collection of graphs is referred as graph continual learning~(GCL). GCL is illustrated on the left of \cref{fig:cartoon} where the complete interval is divided into three tasks. Each task is signified by a learning problem~(graph classification, node classification, regression of graphs or link prediction) on a collection of a graphs. The goal at the onset of the first task is train a graph neural network~(a neural network that takes graphs as input) to achieve perfect performance on the first task. \begin{figure}
\includegraphics[width=\columnwidth]{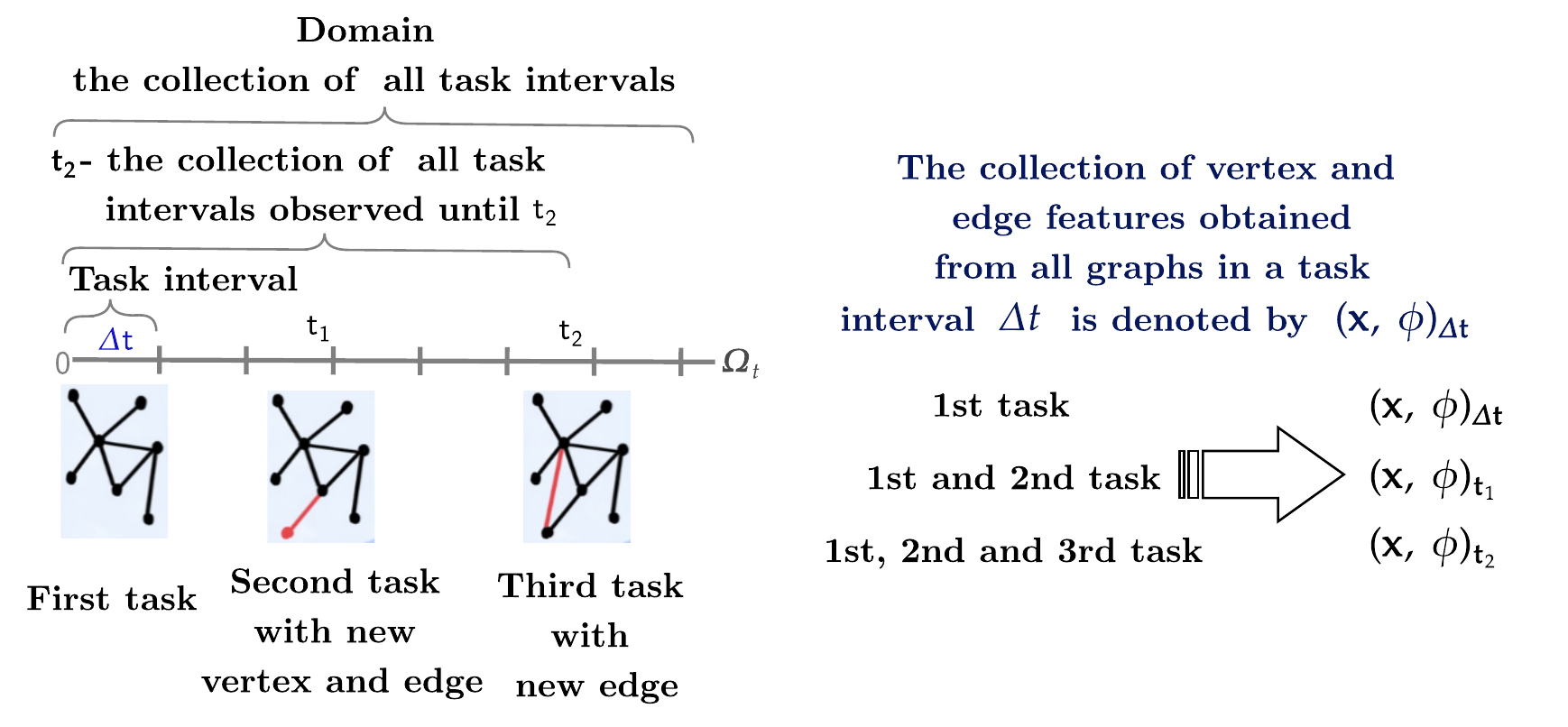}
\caption{Illustration of the GCL problem. At the end of each task interval, a collection of graphs is obtained. These tasks associated with a cost function forms a task. The illustration provides three tasks, the first task is obtained at $\Delta~t$, the second at $t_1$ and the third task at $t_2.$ Each subsequent task may have new edges or new vertices. The notation, $\vect{t}_{2}$ indicates the collection of all tasks in the interval $[0,t_2].$}
\label{fig:cartoon}
\end{figure} After the first tasks is finished, another task is observed and we seek to perform the second task well while remembering the first task. However, as this second task may present new vertices and/or edges along with new edge features and/or vertex features, updating the model naively using the second task will reduce the models' effectiveness on the first task~(previous task) due to a phenomenon known as catastrophic forgetting~\cite{lin1992self}.  This behaviour highlights an unusual dilemma in the CL domain, known as the stability-plasticity dilemma~\cite{carpenter1987massively}, where a model in pursuit of learning new experiences~(generalize to new tasks) loses long-term memory~(remember previous tasks). This dilemma provides a natural objective for GCL: \textit{balance forgetting and generalization.}  To facilitate this balance and achieve improved CL performance on graph tasks one must model the dynamics between forgetting and generalization -- something the literature does not do, more on this in section~\cref{sec:prior}

The first work to formalize these notion of balance in CL was presented in meta-experience replay~(MER)~\cite{riemer2018learning}, where the balance between forgetting and generalization is enforced with hyperparameters. Our prior work in \cite{ krishnan2020meta, krishnan2021continual} attempted to take an alternative adaptive dynamic programming viewpoint of CL to model the this balance and its dynamics with changing tasks. However, both \cite{riemer2018learning} and our prior work in \cite{krishnan2020meta, krishnan2021continual} does not naturally extend to \emph{graph neural network}~(neural networks that are specifically designed to take graphs as inputs~\cite{scarselli2008graph}). In recent years, graph neural networks~(GNNs) have facilitated unprecedented efficiency \cite{bronstein2017geometric,9046288,scarselli2008graph} for datasets with graph inputs. However, much of this research involves cases when all the training graphs are available upfront. In many applications, however, graphs are  observed in sequence, and the distribution of the graphs is nonstationary~\cite{trivedi2018representation}. Thus, we must train the GNN to adapt to the sequential nature of graph generation. While the GCL literature is sparse, several methods~\cite{zhang2022hierarchical, liu2021overcoming, galke2021lifelong, zhou2021overcoming, wang2020lifelong} have been proposed in recent years. 
There are two key issues with the current GCL literature, first, none of the methods in the GCL literature model the dynamics introduced into the GCL problem by the non-stationarity in graphs. 
Furthermore, the balance between generalization and forgetting and these methods are not modeled in any prior work and despite promising initial advances, the current GCL methods~\cite{zhang2022hierarchical, liu2021overcoming, galke2021lifelong, zhou2021overcoming, wang2020lifelong} do not provide any theoretical characterization of GCL.

\subsection{Prior work on GCL~\label{sec:prior} }

A rather complete and recent survey of continual learning when graphs are chosen as inputs is provided in \cite{yuan2023continual, febrinanto2023graph}. It can be seen from the survey that the prior work in GCL methods is sparse and the existing GCL methods can be grouped into three classes, representation learning strategies~\cite{zhang2022hierarchical, galke2021lifelong,  wang2020lifelong, tang2020graph}, regularization strategies~\cite{liu2021overcoming}, and  Replay-based strategies~\cite{zhou2021overcoming}.

\textit{Representation learning strategies~\cite{zhang2022hierarchical, galke2021lifelong,  wang2020lifelong, tang2020graph},} focus on learning representations across tasks to enable CL. For instance, Galke et al.~\cite{galke2021lifelong} incrementally learn unseen classes by estimating the change in the graph distribution characterized by the node features. Wang et al.~\cite{wang2020lifelong} translate a GNN to a feature graph network and apply CL method developed for convolutional neural network~(CNN). The work in \cite{tang2020graph} seeks to build a context graph to model connections and consolidate information across tasks  and demonstrates performance on image classification problems. The method in \cite{tang2020graph} does not provide any evidence for being directly relevant to GCL since all their experiments is on image data.  Zhang  et al.~\cite{zhang2022hierarchical} propose a hierarchical prototype network and provide certain theoretical insights.  In particular, Theorem 1 provides an upper bound on the number of prototypes, and Theorem 2 quantifies the prototype changes due to new tasks.  Furthermore, other recent works that utilize representation learning strategies includes SGNN-GR~\cite{wang2022streaming}~that identifies nodes in the network that are perturbed by the repeated training and retrains them; and MSCGL~\cite{cai2022multimodal} that seeks to obtain an optimal network architecture for each task through automated neural architecture search.    

\textit{Regularization strategies~\cite{liu2021overcoming}:} The key idea behind regularization strategies is to introduce constraints to the weights/output of the network such that the forgetting that is incurred because  a new task is minimized. The idea of regularization alleviates the need to construct different architectures for each problem. However, such a process is governed by preascertained hyperparameters that are more often than not blindly selected without the presence of past experiences~(previous tasks). 
For example, without the information from previous tasks, methods such as  \cite{liu2021overcoming} solely focus on increasing \textit{plasticity~(resistance to change due to the presence of a new task)} in the network. 

\textit{Replay-based strategies~\cite{zhou2021overcoming}:} The plasticity imposed by regularization strategies in \cite{liu2021overcoming} is obviated by ER-GNN~\cite{zhou2021overcoming}, where five pooling mechanisms are specifically developed to aggregate information with careful attention to the previous task through the use of an experience replay mechanism.  ER-GNN~\cite{zhou2021overcoming} empirically studied the effectiveness of these pooling techniques and demonstrated superior performance on vertex classification in the CL setting. However, no attempt was  made to ascertain balance required for efficient CL, nor does ER-GNN provide any theoretical insights. In summary, a key gap in the literature is the lack of theoretical foundations that enable the study and analysis of GCL for a variety of graph tasks involving dynamic graphs without the need for building specific architectures for different applications~\cite{liu2021overcoming, zhou2021overcoming} or blindly setting hyper-parameters~\cite{liu2021overcoming}. 

It is important to note from the the summary of prior work on GCL and the recent surveys~\cite{yuan2023continual, febrinanto2023graph} that; while rudimentary investigations have been performed and several methods have been proposed, there are two fundamental  oversights in the field of continual learning with Graphs. First, there is no characterization of the dynamics of learning, that is, how does the solution to the CL problem change when new tasks are introduced. Second, there is no study, theoretical or empirical, that proves the feasibility of a new balance point between forgetting and generalization whenever a shift in the input distribution is observed. 

To obviate these shortcomings, we develop the theoretical foundations of GCL where we model the progression of value function~(the best CL cost over all possible tasks) with respect to tasks in the GCL setting as a dynamical system. To this end, we extend our prior work in \cite{krishnan2020meta, krishnan2021continual}  to dynamic graphs and present our insights both when tasks are observed continuously and when the tasks are observed discretely. Towards this pursuit, we model the non-stationarity in graphs through continuous~(CT) and discrete~(DT) stochastic processes leveraging the vertex edge random graphs (VERGs)~\cite{verg} formalism. Subsequently, we define the GCL task as a realization of VERG with the corresponding loss function. We utilize an adaptive dynamic programming viewpoint to formulate GCL, wherein the GCL cost~($L$) over all the possible tasks~(domain of the CL problem) is given by an integral of the generalization~(cost on the new task) and forgetting cost~(cost on all previous tasks) over the domain. The solution to the GCL problem is then be obtained by minimizing $L$ to get the value function $L^*.$ However, obtaining this solution is not straightforward because much of the domain is unknown.  To circumvent this issue, we utilize Bellman's principle of optimality and derive an ordinary differential equation~(ODE)  which models the progression of value function as a function of tasks. That is, models, how each new task affects the value function. To solve the ODE, we discretize the domain and formulate a sequential two player min-max game with the goal of achieving balance between generalization and forgetting---one player maximizes generalization, and another minimizes forgetting. We solve this game utilizing mini-batch stochastic gradient ascent-descent. 

We present a comprehensive theoretical analysis that proves the existence of at least one balance point between these two players for each new task~(\cref{theory:theorem1}). We also show convergence of our algorithm~(\cref{theory:theorem2}). We demonstrate these results while considering the effect of dynamic graph with assumptions of Lipschitz continuity. To substantiate our approach empirically, we demonstrate a 44\% improvement over the state of the art on vertex classification in a CL setting. Furthermore, we use large-scale hyperparameter search experiments to demonstrate the robustness of our method to different hyperparameters and initialization of weights on the graph classification. We also present an ablation study and demonstrate a 21\% improvement  over the naive experience replay method. 

\subsection{Summary of Contributions}
In summary, we build on~\cite{krishnan2020meta, krishnan2021continual} to initiate the investigation into, how tasks impact the graph continual learning problem when the input is provided by graphs. Our work provides the first theoretical characterization and analysis of the balance between forgetting and generalization in the GCL setting where we model the non-Euclidean nature of the graphs including the stochasticity presented by the dynamic nodes, the edge attributes and edge connectivity~(including dynamic connectivity) and introduce a theoretical framework to study this. Furthermore, utilizing this framework, we develop a simple methodology that can be shown theoretically to converge and provides empirical advantages. Finally, we provide a comprehensive ablation study where we illustrate the effect of changes in hyperparameter configurations, which is quite unique in the ML setting.

\subsection{Notations}
We denote scalars by lowercase letters, i.e., $x,$ vectors by lowercase bold letters, i.e., $\vect{x}.$ A matrix is denoted by uppercase bold alphabets, i.e., $\vect{X}$ and use $\vect{w}$ to collectively denote a column vector of all parameters. Let the interval $\vect{\Omega} = [0,\Omega]$ for $\Omega>0$ represent the sample space for all tasks instances (the total interval in which tasks may be observed) and $\mathbb{R}^+$ referring to the set of real positive numbers. Task intervals in $\Omega$ are given by $\vect{t} \subseteq \vect{\Omega}\;|\;\vect{t} = [0,t].$ In a special case, when $\vect{\Omega} \subset \mathbb{N},$ with $\mathbb{N}$ denoting the set of natural numbers, we will use $k$ to denote discrete task instances in $\vect{\Omega}$. Any partial derivatives are given by $\partial_{y}(x)$---the partial derivative of $x$ with respect to $y.$ An asymptotic value is indicated by a superscript $N$, optimal value by a superscript $*$ and iterations are indicated by a superscript in parenthesis. For instance, $L^{*}_{t}$ refers to the optimal value of $L$ with respect to the task $t$ and $\vect{w}^{(n)}$ provides the value of $\vect{w}$ after $n$ iterations. 

\subsection{Outline}
There are two major goals in this paper. First, we seek to derive the dynamics of the GCL value function. Next, we derive our two player min-max game to provide updates to a neural network to find the balance between generalization and forgetting. We begin with preliminaries from \cite{verg} in  and describe a task in GCL in \cref{sec:GCL}. Next, we describe the dynamical system modelling of GCL in \cref{sec:dyn_sys} and \cref{sec:GCL_ode} and formulate the two player game and provide the algorithm in \cref{sec:balance}. The theoretical analysis presented in \cref{sec:theory} and we substantiate our contributions empirically in \cref{sec: results}.

\section{Graph Continual Learning}\label{sec:GCL}
GCL is the problem of learning a sequence of tasks, where each task is defined as a collection of graphs associated with a loss function according to a GNN. 
Therefore, we must first define a collection of graphs that captures the non-stationarity in vertices and edges within a graph. In typical literature, a collection of graphs  $G \in \mathcal{G}(k)$ for  $k \in \vect{\Omega}$ is a tuple of vertices $\vect{V} \subset \boldsymbol{\mathcal{V}}$, edges $\vect{E}$, and features $\vect{F}.$  This collection of graphs is typically considered as samples from random graphs -- in particular, a probability space  $(\vect{\Omega}, \sigma(\vect{\Omega}), \mathbb{P}).$ Here, $\vect{\Omega}$ is a sample space equipped with a sigma algebra $\sigma(\vect{\Omega})$ and  $\mathbb{P}$ is a probability space. Furthermore, $\vect{V} \subset \mathbb{N}$ is the vertex set, $\vect{E} \in \mathbb{R}^{ |\vect{V}| \times |\vect{V}|}$ is the edge set with $\vect{\mathcal{X}} \subset \mathbb{R}^{n}$ being the sample space for the features $\vect{F}.$  In the context of GCL, graphs within the collection $\mathcal{G}(k)$ can be dynamic for different tasks: they can have different vertices, edges, and corresponding features. We formalize this dynamic behavior using vertex-edge random graph~(VERG)~\cite{verg} where both the vertices and edges of a random graph are formalized as stochastic processes. 

\subsection{VERG for GCL}
 According to \cite{verg}, a vertex edge random graph~(VERG) is a probability space over $(\vect{x},\phi)$ graphs. Subsequently, an $(\vect{x},\phi)$ graph is a tuple of vertex features, represented by $\vect{x}$ and edge features, represented by  $\phi.$ The stochastic processes $\vect{x}$ and $\vect{\phi}$ can be thought as measurable functions the assign real values to vertex and an edge. Consequently, a VERG is of the form $(\mathcal{G}_{\vect{V}}, P)$, where $\mathcal{G}_{V}$ is a collection of $(\vect{x}, \phi)$ graphs with a probability measure $P.$ 

To adapt this generic notion of VERG must be adapted to GCL we let $\boldsymbol{\mathcal{V}}$ be the set of all vertices over which the graph data can be collected and that $\boldsymbol{\mathcal{V}}$ is endowed with a compact neighborhood topology. Then, for each interval $[t, t+ \Delta t] \subset \vect{\Omega}| t, \Delta t \in \vect{\Omega},$  the collection of graphs is defined over a particular task dependent vertex set $\vect{V}(t) \subset \boldsymbol{\mathcal{V}}$ such that VERG for GCL is formally given as a probability space over dynamic graphs.

\begin{defn}[VERG for GCL]\label{def:stp} 
For any $t \in \vect{\Omega}, v \in \boldsymbol{\mathcal{V}}$ define $ \vect{x}(v, t) :  \boldsymbol{\mathcal{V}} \times \vect{\Omega} \rightarrow \mathbb{R}^{n}$ and $\phi : \boldsymbol{\mathcal{V}} \times \boldsymbol{\mathcal{V}} \times \vect{\Omega} \rightarrow \mathbb{R}.$ Then, $\vect{x}_{\vect{V}}(t) = \{ \vect{ {x}}(v, t), \quad \forall v \in \vect{N}_{v}, \forall \vect{N}_{v} \subset \vect{V}(t)\}_{\vect{V}(t)}$ is the vertex stochastic process with $\phi_{\vect{V}}(t) = \{ \phi(i, j, t), \forall i, j \in \vect{N}_{v}, \vect{N}_{v}  \subset \vect{V}(t)\}_{\vect{V}(t)}$ being the edge stochastic process. Then, VERG is a probability space of the form $(\mathcal{G}_{\vect{V}}(t), P_{\vect{x} \times  \phi} )$ where $\mathcal{G}_{\vect{V}}(t)$ is a collection of a graphs $\left( \vect{x}_{\vect{V}}(t), \phi_{\vect{V}}(t) \right)$ with probability measure $P_{\vect{x} \times  \phi}.$ 
\end{defn}

\subsection{Learning Objective:} When dynamic graphs are generated, the associated problem is to learn an unknown function $g(.): \mathbb{R}^{n} \times \mathbb{R} \rightarrow \mathbb{R}^p$ such that  $\label{eq:prob} y(t) = \rho(\vect{x}(t), \phi(t))$
with $y(t) = \partial_{t} (.)$ for regression and $y(t)$ representing probabilities in the context of classification~(node and edge classification as well as link prediction).  We seek to approximate this unknown function using a graph neural network~(GNN) parameterized by some weight parameters. Therefore, we define a compact parameter space $\mathcal{W} \subset \mathbb{R}^{m}$ and a vector-valued function that  provides an $\mathbb{R}^{m}$ weight vector for each $t \in \vect{\Omega}$ where $\vect{w}(t) : \vect{\Omega} \rightarrow \mathcal{W}.$ The vector-valued function can be essentially thought as a policy that defines the set of rules according to which a weight vector may be obtained corresponding to each $t$ such that 
$y(t) = g(.,., \vect{w}(t))$ represents the GNN. Then, we define a loss function $\ell: \mathbb{R}^n \rightarrow \mathbb{R},$ which provides the learning objective. 

\subsection{GCL task}
The GCL task is given as
\begin{defn}[GCL task in the continuous sense]\label{defn:task} For $t, \Delta t \in \vect{\Omega}$, define the interval $[t, t + \Delta t]$ and let $(\mathcal{G}_{V}(t), P_{\vect{x}(t) \times  \phi(t)})$  represent a VERG associated with GCL. Denote the GNN model as $g(., ., \vect{w}(t)): \vect{\Omega} \rightarrow \mathbb{R}^{n}$ with a loss function given as $\ell: \mathbb{R}^p \rightarrow \mathbb{R}.$ Let $J(\vect{x}([t, t+\Delta t]), \phi([t, t+\Delta t]), \vect{w}([t, t+\Delta t])) = \int_{\tau = t}^{t+\Delta t} \ell(\vect{x}_{\vect{V}}(\tau), \phi_{\vect{V}}(\tau), \vect{w}(\tau))]$ be the forgetting and generalization cost over the interval $[t, t+\Delta t].$  Then, a GCL task $\mathcal{T}([t, t+\Delta t])$ is described by the tuple \begin{multline} \Big( \vect{x}([t, t+\Delta t]), \phi([t, t+\Delta t]),  J(\vect{x}([t, t+\Delta t]), \phi([t, t+\Delta t]), \vect{w}([t, t+\Delta t]))  \Big)\end{multline} with $ \vect{x}([t, t+\Delta t]) = \{\vect{x}_{\vect{V}}(\tau) \forall \tau \in [t, t+\Delta t] \}_{\vect{V}(t)}$ and  $\phi([t, t+\Delta t]) = \{\vect{\phi}_{\vect{V}}(\tau) \forall \tau \in [t, t+\Delta t] \}_{\vect{V}(t)}$
\end{defn}
A GCL task has been defined as a stochastic process over $\vect{V}(t)$ of $\boldsymbol{\mathcal{V}}.$ As a GNN operates by performing operations in a neighborhood and the vertex set need not be fixed and the loss function corresponding to a GNN is simply an integral across all neighborhoods in the vertex set $\vect{V}(t)$. As $\boldsymbol{\mathcal{V}}$ is endowed with a compact neighborhood topology, the vertex set is decomposable into overlapping neighborhood and the cost $J(\vect{x}([t, t+\Delta t]), \phi([t, t+\Delta t]), \vect{w}([t, t+\Delta t]))$ can be defined as an integral over all vertex sets in the interval $[t, t+\Delta t]$. This key feature in the formulation of a task allows the flexibility to consider distinct vertex sets for distinct graphs is not present in other GCL methods in the literature. The only restriction in this setting is that the total number of vertices in a graph may not exceed the cardinality of $\boldsymbol{\mathcal{V}}.$ 

For simplicity of notations, we will denote the task as  $\mathcal{T}_{[ t, t+\Delta t]}$ which is described by the tuple  $\left( (\vect{x}, \phi)_{[ t, t+\Delta t]}, J_{[t, t+\Delta t]} (\vect{x}, \phi, \vect{w}) \right)$ where the subscript indicates the interval over which the task is defined. This notation, easily extends to a collection of tasks. For instance, all the tasks in the interval $[0,t]$ are collectively provided by $\mathcal{T}_{[ 0, t]}.$ As we use $\vect{t}$ to represent the interval $[0,t],$ it follows that $\mathcal{T}_{[ 0, t]}$ is rewritten as $\mathcal{T}_{\vect{t}}$ which represents all tasks in the interval $[0,t]$ where $\mathcal{T}_{\vect{t}} = \left( (\vect{x}, \phi)_{\vect{t}}, J_{\vect{t}} (\vect{x}, \phi, \vect{w}) \right) $ with $ J_{\vect{t}} (\vect{x}, \phi, \vect{w}) =  \int_{\tau = 0}^{t} \ell(\vect{x}(\tau), \phi(\tau), \vect{w}(\tau)).$ This notation naturally extends to the case when $\vect{\Omega}$ is comprised of discrete instance as well. In this case, we will set $\Delta t =1$ and replace $t$ by $k$ such that $[0,t] = [0,k] = [0,1,2,3,\cdots, k] = \vect{k}$ Furthermore, the collection of all tasks in the interval $[0,k]$ is given by $\mathcal{T}_{\vect{k}}$ with $J_{\vect{k}}(\vect{x} , \phi, \vect{w})= \sum_{\tau = 0}^{k} \ell(\vect{x} (\tau), \phi (\tau), \vect{w}(\tau))$ where $g(\vect{x}(k), \phi(k))$ is the parametric map and $\ell(\vect{x} (k), \phi (k), \vect{w}(k))$ is the corresponding loss with the VERG defined by the probability space $(\mathcal{G}(k), P_{\vect{x}(k) \times  \phi(k)})), k \in \vect{\Omega}.$  

\subsection{Dynamical Systems Modelling}\label{sec:dyn_sys}
Consider now an example for GCL, where at time $t_{i},$  the goal is to perform well on all tasks in the interval $[0, t_i],$ the collection of which is provided by $\mathcal{T}_{ \vect{t}_i}.$  The conventional CL approach involves constructing a cost function as 
$J_{\vect{t}_{1}} (\vect{x} , \phi , \vect{w}) =  \int_{\tau = 0}^{t_i} \ell(\vect{x} (\tau), \phi (\tau), \vect{w}(\tau))$ with all the available tasks and minimizing it. Therefore, the GNN parameter set $\vect{w}( \vect{t}_1)$ is sought such that  $$\underset{\vect{w}( \vect{t}_1)}{min} J_{\vect{t}_{1}} (\vect{x} , \phi , \vect{w}).$$ Next, when at $\vect{t}_2$, new set of tasks are observed and the GNN must be updated such that the GNN parameter set $\vect{w}( \vect{t}_2)$ minimizes $J_{\vect{t}_{2}} (\vect{x} , \phi , \vect{w})$ where \[J_{\vect{t}_{2}} (\vect{x} , \phi , \vect{w}) = J_{\vect{t}_{1}} (\vect{x} , \phi , \vect{w}) + J_{[\vect{t}_{1},  \vect{t}_{2}]} (\vect{x} , \phi , \vect{w}).\] That is, we seek to solve for  both $J_{\vect{t}_{1}}$ and $J_{[\vect{t}_{1},  \vect{t}_{2}]}.$ This naive method of solving GCL leads to two challenges. First, the number of tasks that must be performed increase with each subsequent tasks and the increase in the number of tasks leads to a shrinkage in the feasible region of the solution space. Second, whenever the new set of tasks are observed, $\vect{w}$ obtained for the previous tasks is the starting point and this starting point biases the GNN towards the future tasks. 

To elaborate on these challenges, let there be three graph learning tasks and let the solution space of the  be $\mathcal{W}_1, \mathcal{W}_2$ and $\mathcal{W}_3$ with $\vect{w}_1^*, \vect{w}_1^*$ and $\vect{w}_3^*$ being the ideal solutions~(centers of these feasible regions). Say, the first task is solved perfectly~(solution $w_1^*$ is attained). Therefore, the search for second task begins from $\vect{w}_1^*.$ This process would move the solution from the feasible region $\mathcal{W}_1$ to  $\mathcal{W}_1 \cap \mathcal{W}_2$ because the solution that is optimal for both task 1 and 2 can only be found on this intersection. If the two tasks are identical or very similar, then $\mathcal{W}_1 \cap \mathcal{W}_2$ is similar to $\mathcal{W}_1 \cup \mathcal{W}_2.$ However, if the two tasks are not identical, the intersection space is smaller.

With this behavior, three scenarios may arise. First, with more and more dissimilar tasks, the solution space~(intersection of the feasible regions) will shrink and it is possible that a solution that will perform all tasks equally well may not exist. Second,  if the new task pulls the solution unreasonably towards its own solution space, we will incur large forgetting.  Third, if the new task does not influence the solution at all, we will not learn the new task at all, no generalization.  Due to these reasons, we cannot naively minimize a cost as done by traditional CL methods and it is imperative to update the GNN in such a way that a solution from the previous tasks is neither unreasonably influenced by the new task~(which will reduce forgetting), nor ignore the new task completely~(no generalization). That is -- it is necessary to achieve \textit{balance forgetting and generalization.} 

Even if it was possible to achieve a balance point, a bias exists. For example, say the first task has been learned and we search for the solution of the second task beginning from $\vect{w}_1^*.$  Due to this, $\vect{w}_1^*$ biases the solution for the second task and even the tasks in the future because, the distance between $\vect{w}_1^*$ and $\mathcal{W}_1 \cap \mathcal{W}_2$ determines how likely we will be able to attain an optimal solution for both task one and two and the quality of this solution. As each subsequent new task accumulates this bias, a multi-stage decision making stochastic process must be considered where the weight parameter at each stage influences the subsequent stages. 

Therefore, we need to find a solution that is not only a saddle point between generalization and forgetting but also converges to an optimal solution that is optimal irrespective of the bias. To model the bias, we must characterize ``How does the solution to the GCL problem behave when each new task is observed?" To find the balance, we seek to answer the question ``Can we force this solution to achieve a saddle point that enforces balance between generalization and forgetting?" We will begin by modelling the behaviour~(dynamics) of the solution as a function of each new task. 

In this pursuit, we will take an adaptive dynamic programming viewpoint and leverage tools for optimal control literature~\cite{lewis2012optimal}. Give a collection of tasks, the generalization and catastrophic forgetting cost is
\begin{equation}
    \begin{aligned}
        J_{\vect{t}}(\vect{x}, \phi , \vect{w} )= \int_{\tau = 0}^{t} \ell_{\tau }(\vect{x}, \phi , \vect{w} ) d\tau.
    \end{aligned}
    \label{eq:eq_CL cost}
\end{equation}
Consider the example in \cref{fig:cartoon} where the forgetting and generalization cost at $\vect{t}_1$ is given as $J_{\vect{t}_1}(\vect{x}, \phi , \vect{w} )$ with the weights of the GNN initialized at $\vect{w} (\vect{t}_1).$ As we go through the updates, we obtain obtain $\vect{w}^*(\vect{t}_1)$ by minimizing $J_{\vect{t}_1}(\vect{x}, \phi , \vect{w} ).$  However, when the next set of previous and new tasks are observed,  the quality of $\vect{w}^*(\vect{t}_2)$ is biased  by $\vect{w}^*(\vect{t}_1).$  This bias can be removed if $\vect{w}^*(\vect{t}_1)$ is obtained considering its effect of $\vect{w}^*(\vect{t}_2).$ At $\vect{t}_1,$ if the total cost $L(\vect{t})_1$ is $J_{\vect{t}_1}(\vect{x}, \phi , \vect{w} ) + J_{\vect{t}_2}(\vect{x}, \phi , \vect{w} ) + J_{\vect{t}_3}(\vect{x}, \phi , \vect{w} )$ is minimized, then, we have considered the complete effect of $\vect{w}(\vect{t}^1)$ on ever task in the future. Generalizing this idea over the interval $\vect{\Omega},$ we obtain the GCL problem considering the effect of GNN over all tasks is given as
\begin{equation}
    \begin{aligned}\label{eq:problem_GCL}
    L^*_{\vect{t}} =  \underset{ \vect{w}(\vect{t}) }{min} \int_{\tau=t}^{\Omega}   J_{\boldsymbol{\tau} }(\vect{x}, \phi , \vect{w} ),
    \end{aligned}
\end{equation}
where the  outer integral summarizes the cumulative effect of $\vect{w}(t)$ on all future tasks and the effect of $\vect{w}(t)$ on all the previous and the new task is considered through  $J_{\boldsymbol{\tau} }(\vect{x}, \phi , \vect{w} ).$ Solving \cref{{eq:problem_GCL}} results in $L^*( \vect{t} )$ as the minimum which is obtained considering the effects of the GCL solution over all the tasks in the past, present and the future -- a holistic unbiased solution of GCL. Therefore, our goal of GCL is to find a holistic unbiased policy $\vect{w}(\vect{t})$ which results in $L^*( \vect{t} ).$ 

\subsection{GCL Ordinary Differential Equation~(ODE)}\label{sec:GCL_ode}
It is intractable to calculate the integral in \cref{eq:problem_GCL} as the large part of the domain is unknown since future tasks are not available. We will therefore simplify this problem using Bellman's principal of optimality~\cite{krishnan2020meta,krishnan2021continual} where we will approximate the information regarding future tasks using what information is available right now. With this setup we will now derive the dynamics of the graph continual learning problem through the following proposition.

\begin{proposition} \label{prop:GCL_dynamics}
For $t \in \vect{\Omega}$ and a vertex set $\vect{V}(t):  \vect{\Omega} \rightarrow  \boldsymbol{\mathcal{V}} | \vect{V}(t) \subset \boldsymbol{\mathcal{V}},$ define the CL task as in  \cref{defn:task}. For each $\vect{t} \subset \vect{\Omega}$ let
\begin{equation}
    \begin{aligned}
      L^*_{\vect{t}} =  \underset{ \vect{w}(\vect{t}) }{min} \int_{\tau=t}^{\Omega}   J_{\boldsymbol{\tau} }(\vect{x}, \phi , \vect{w} )
    \end{aligned}
\end{equation}
as the GCL problem. Assume $  J_{\vect{t}}(\vect{x}, \phi , \vect{w} )$ to be smooth with respect to all its arguments. Under the assumption that $R(\vect{t})$ denotes all the higher order terms in a Taylor series expansion, the following is true
\begin{equation}
    \begin{aligned} \label{eq:sp_ODE}
       - \left( \partial_{ \vect{t} } L^*_{\vect{t}} \right)^T \Delta_{\vect{t}} & =  \underset{ \vect{w}( \vect{t} ) }{min} \Big[ J_{\vect{t}}(\vect{x}, \phi , \vect{w} ) 
       + \left( \partial_{\vect{x}_{\vect{t}}} L^*_{\vect{t}} \right)^T \Delta_{\vect{x} }  \\ 
       & + \left( \partial_{ \phi_{\vect{t}} } L^*_{\vect{t}} \right)^T \Delta_{ \phi} 
       + \left( \partial_{\vect{w}_{\vect{t}} } L^*_{\vect{t}} \right)^T \Delta_{\vect{w} } + R_{\vect{t}}, \Big]
    \end{aligned}
\end{equation}
where $\Delta{x}$ refers to tiny perturbation.
\end{proposition}
\begin{proof}
The proof is provided in the supplementary files.
\end{proof}
\cref{prop:GCL_dynamics} summarizes the dynamics of GCL as a function of tasks. The partial derivative of the value function on the left-hand side of the equation describes the change in value function introduced by a change in $\vect{t}$-- the introduction of a new task. The terms in the right-hand side describe which different element quantify this change in the value function. The first term summarizes the cost of generalization and forgetting. The second term characterizes the change due to $\vect{x}$---``how  change in the vertex  features modifies the value function." The third term describes  `` how  the change in the edge features impacts the value function". The fourth term summarizes the impact of model parameters on the value function and the fifth term summarize all the higher-order terms. One of the key features of this ODE is that the change in connectivity of the graphs is modelled by the last two terms of this ODE. The second and the third term are defined on $\vect{V}(t)$ which describes the change in the vertex set. As the vertex set may change as a function of task, both the edge feature and vertex features may change. This ability to theoretically model edge-connectivity provides a unique approach to GCL.

To verify whether the equation GCL ODE is correct or not, we will derive the ODE for when the vertex set has cardinality one, with unit neighborhood size and the edge properties being non-existant. Intuitively, this equation should be equivalent to \cite{krishnan2020meta}.
\begin{lemma} \label{lem:eq_Neurips}
Let the graph ODE be given as
\begin{equation}
    \begin{aligned}
         - \left( \partial_{ \vect{t} }L^*(\vect{t}) \right)^T \Delta_{\vect{t}} & =  \underset{ \vect{w}( \vect{t} ) }{min} \Big[ J_{\vect{t}}(\vect{x}, \phi , \vect{w} ) 
    + \left( \partial_{\vect{x}_{\vect{t}}} L^*_{\vect{t}} \right)^T \Delta_{\vect{x} }  \\ &+ \left( \partial_{ \phi_{\vect{t}} } L^*_{\vect{t}} \right)^T \Delta_{ \phi} 
    + \left( \partial_{\vect{w}_{\vect{t}} } L^*_{\vect{t}} \right)^T \Delta_{\vect{w} } + R_{\vect{t}}, \Big]
    \end{aligned}
\end{equation}
Let $ \vect{t}=t$ and therefore $\Delta t=t.$ Assuming $G = \vect{x}$, where the edge properties have been removed. Define a task in this context as $\mathcal{T}(t)= ( \vect{x}(t), \ell(\vect{w}(t))$ and assume that $R_{t} = 0,$ Then the dynamics of continual learning are governed by
\begin{equation}
    \begin{aligned}
       \partial_{ \vect{t} } L^*(t )^T \Delta_{t} &= -\underset{ \vect{w}(t) }{min} \Big[ J_{\vect{t}}(\vect{x} , \vect{w} )  + \left( \partial_{\vect{x}_{\vect{t}}} L^*_{\vect{t}} \right)^T \Delta_{\vect{x} }  + \left( \partial_{\vect{w}_{\vect{t}} } L^*_{\vect{t}} \right)^T \Delta_{\vect{w} }\Big].
    \end{aligned}
\end{equation}
where $\partial_{y} x$ refers to the partial derivative of $x$ with respect to $y$ and $\Delta_{x}$ refers to the first difference in $x.$
\end{lemma}
\begin{proof}
The end result is obtained by simply substituting $\Delta t = t, \vect{t} = t$ and $ \left( \partial_{ \phi_{\vect{t}} } L^*_{\vect{t}} \right)^T \Delta_{ \phi}  = 0.$\end{proof}
The preceding lemma proves that the work in \cite{krishnan2020meta} is a special case of the dynamics provided here. This reinforces the fact that the present work is the most general form of continual learning that can work for both Euclidean and nonEuclidean data. The ODE in \cref{prop:GCL_dynamics} models the behavior of the GCL solution $L^*(\vect{t})$ as a function of the tasks. Next, we seek solution to the GCL problem where we formulate a saddle point problem to balance generalization and forgetting.

\section{Balancing Forgetting and Generalization}\label{sec:balance}
In the preceding section, we formulated the problem of GCL where we obtained ODE that models the behavior of $L^*$-- the optimal solution of the GCL problem, as a function of the tasks. We now seek to update to the policy using the differential equation in \cref{prop:GCL_dynamics}. An update to our policy can be directly obtained by solving the ODE in Eq.~\eqref{eq:eq_CL cost}. However, solving the ODE fully may be quite tedious. Fortunately, it is not necessary to solve the ODE explicitly. The important thing is that when any new task is revealed, we must find an update to the policy such that the new task induces minimum change~(zero change in an ideal case) to the value function. This will allow us to find a stable balance point which is not biased by the tasks. This insight provide the following optimization problem that need to be solved:
\begin{gather}
       \underset{ \vect{w}( \vect{t} ) }{min} \Big[ J_{\vect{t}}(\vect{x}, \phi , \vect{w} ) 
    + \left( \partial_{\vect{x}_{\vect{t}}} L^*_{\vect{t}} \right)^T \Delta_{\vect{x} }  + \left( \partial_{ \phi_{\vect{t}} } L^*_{\vect{t}} \right)^T \Delta_{ \phi} 
    + \left( \partial_{\vect{w}_{\vect{t}} } L^*_{\vect{t}} \right)^T \Delta_{\vect{w} }, \Big]  \\
       \text{ subject to }   
       - \left( \partial_{ \vect{t} }L^*(\vect{t}) \right)^T \Delta_{\vect{t}} + R_{\vect{t}} \leq \kappa,  
       \forall \kappa \in [0,1]. \nonumber
\end{gather}
In many GCL settings, the tasks may arrive at discrete time instances. Therefore, for each $k \in \vect{\Omega}, \Delta_{\vect{t}} = 1$ the optimization problem is given as
\begin{gather}
    \underset{ \vect{w}( \vect{k} ) }{min} \Big[ J_{\vect{k}}(\vect{x}, \phi , \vect{w} ) 
    + \left( \partial_{\vect{x}_{\vect{k}}} L^*_{\vect{t}} \right)^T \Delta_{\vect{x} }  + \left( \partial_{ \phi_{\vect{k}} } L^*_{\vect{k}} \right)^T \Delta_{ \phi} 
    + \left( \partial_{\vect{w}_{\vect{k}} } L^*_{\vect{t}} \right)^T \Delta_{\vect{w} }, \Big]  \\
    \text{ subject to }   
    - \partial_{ \vect{t} }L^*(\vect{k})  + R_{\vect{k}} \leq \kappa,  
    \forall \kappa \in [0,1]. \nonumber
\end{gather}
With the assumption that $R(\vect{k}) \leq \kappa \quad \forall \kappa \in \mathbb{R}$, we write the optimization problem as 
\begin{gather}
    \underset{ \vect{w}_{\vect{k}}}{min} \quad  \mathcal{H}_{\vect{k}}(\Delta_{\phi}, \Delta_{\vect{x} }, \Delta_{\vect{w} }, \vect{x}, \phi ,\vect{w} ),
\end{gather}
where we have gathered all terms in the bracket. In summary, the GCL problem is that of finding a policy~($ \vect{w}$) that minimizes $ \mathcal{H}_{\vect{k}}(\Delta_{\phi}, \Delta_{\vect{x} }, \Delta_{\vect{w} }, \vect{x}, \phi ,\vect{w} )$ whenever a new task is introduced. Typically, at the onset of each new task, there is a change in the graph indicated by $\Delta_{\phi}$ and $\Delta_{\vect{x}}.$ Subsequently, corresponding to this change in the graph, there is a particular $\Delta_{\vect{w}}$ that provides a change in the policy just due to the new task. This change due to the new task experienced by the policy is the process of generalization. Therefore, the larger the value of the delta's, the larger the shift in policy, the larger the generalization. As a  large shift in policy will erase the knowledge of previous tasks ~(as the policy moves away from the previous tasks), large forgetting follows large generalization. Since, large forgetting due to a new task is not desirable, this shift in policy must be regularized. However, such a regularizing in only possible, if the exact value of $\Delta_{\vect{w}}$ corresponding to a given value of $\Delta_{\phi}, \Delta_{\vect{x}}$ is known.  In the absence of such information, which is the case in most GCL/CL applications, we must first ascertain the delta's~($\Delta_{\phi}, \Delta_{\vect{x}}, \Delta_{\vect{w}}$) to estimate generalization and then, update $\vect{w}$ to guarantee minimum forgetting. In this work, we simulate generalization by updating delta's to introduce maximum change in the value function and  update the policy to reduce the effect of worst case generalization. 
Thus, we formulate a two player min-max game where the worst case generalization and the corresponding forgetting is implicitly determined through iterative updates~\cite{dantzig1956constructive} with $\Delta$'s updated through gradient ascent and $\vect{w}$ is updated through gradient descent.

To indicate these iterative updates, we will introduce two iteration indices $i$ and $j$, respectively. We use $i = 1, 2, 3, \cdots, \zeta$ to indicate updates for the $\Delta$'s and  $j=1, 2, 3, \cdots, \rho$ to indicate updates on $\vect{w}.$ Then, we rewrite the optimization problem as $\underset{ \vect{w}^{(j)}_{\vect{k}} }{min} \quad H_{\vect{k}}(\Delta^{(i)}_{\phi}, \Delta^{(i)}_{\vect{x} }, \Delta^{(i)}_{\vect{w} }, \vect{x}, \phi,\vect{w}^{(j)}),$  where we denote the task with the subscript $\vect{k}$ with the min and max iteration at each task denoted by superscript $i$ and $j$, respectively. Next, we define three compact sets $\mathcal{W}, \vect{\mathcal{X}}, \vect{\Phi}$ such that the search space for the optimization problem is described by the triplet $ \Delta_{\vect{w}}^{(i)}  \times  \Delta^{(i)}_{\phi} \times \Delta^{(i)}_{\vect{x}} \in  \mathcal{W}  \times \vect{\Phi} \times \vect{\mathcal{X}}, \quad  \vect{w}^{(i)}\in  \mathcal{W}.$ We approximate the value function using $J_\vect{k},$ as given by the following proposition
\begin{proposition}
Let $\vect{k} \in \vect{\Omega}$ and define $\mathcal{W}, \vect{\mathcal{X}}, \vect{\Phi}$ such that $ \Delta_{\vect{w}}^{(i)} \times  \Delta^{(i)}_{\phi} \times \Delta^{(i)}_{\vect{x}} \in  \mathcal{W}  \times \vect{\Phi} \times \vect{\mathcal{X}}$ and assume that
\begin{subequations} \begin{align}
    sup_{\vect{\mathcal{X}} } L^*_{\vect{k}} &\leq inf_{\vect{\mathcal{X}}} J_{\vect{k}}(\vect{x} , \phi , \vect{w}^{(j)} ) \leq max_{\vect{\Delta}^{(i)}_{\vect{x}_{\vect{k}}} \in \vect{\mathcal{X}} } J_{\vect{k}}(\vect{x}+ \vect{\Delta}_{\vect{x}}^{(i)} , \phi , \vect{w}^{(j)} )      \\
    sup_{\vect{\Phi}} L^*_{\vect{k}} &\leq  inf_{\vect{\Phi}} J_{\vect{k}}(\vect{x}, \phi , \vect{w}^{(j)} ) \leq max_{\Delta^{(i)}_{\phi_{\vect{k}}} \in  \vect{\Phi} } J_{\vect{k}}(\vect{x} , \phi + \Delta^{(i)}_{\phi}  , \vect{w}^{(j)}) \\
     sup_{\mathcal{W}} L^*_{\vect{k}} &\leq inf_{\mathcal{W}} J_{\vect{k}}(\vect{x}, \phi , \vect{w}^{(j)} ) \leq max_{\Delta^{(i)}_{\vect{w}} \in \mathcal{W}} J_{\vect{k}}(\vect{x} , \phi , \vect{w}^{(j)} +\Delta^{(i)}_{\vect{w}}) \\ Inf_{\vect{\mathcal{X}} } L^*_{\vect{k}} \geq 0, 
    Inf_{\vect{\Phi}} L^*_{\vect{k}} &\geq 0, 
     Inf_{\mathcal{W}} L^*_{\vect{k}} \geq 0. \end{align} \end{subequations}
Then, define $\mathcal{H}_{}(\Delta^{(i)}_{\phi}, \Delta^{(i)}_{\vect{x} },\vect{x} , \phi  ,\vect{w}^{(j)} ) =  J_{\vect{k}}(\vect{x}, \phi , \vect{w}^{(j)} ) + \left( \partial_{\vect{x}_{\vect{k}} } L^*_{\vect{k}}\right)^T \Delta^{(i)}_{\vect{x}_{\vect{k}}}     \\ + \left( \partial_{ \phi_{\vect{k}} } L^*_{\vect{k}} \right)^T \Delta_{ \phi_{\vect{k}} } + \left( \partial_{\vect{w}^{(j)}_{\vect{k}}} L^*_{\vect{k}} \right)^T \Delta^{(i)}_{\vect{w}_{\vect{k}}}$ and the following approximation is true
\begin{equation}
	\begin{aligned}
	& H_{\vect{k}}(\Delta^{(i)}_{\phi}, \Delta^{(i)}_{\vect{x} },  \Delta^{(i)}_{\vect{w} },\vect{x} , \phi  ,\vect{w}^{(j)} )  \leq max_{\Delta_{\vect{w}_{\vect{k}}}^{(i)} \times  \Delta^{(i)}_{\phi_{\vect{k}}} \times \Delta^{(i)}_{\vect{x}_{\vect{k}}} \in  \mathcal{W}  \times \vect{\Phi} \times \vect{\mathcal{X}}} \\  
      &\Big[ J_{\vect{k}}(\vect{x}, \phi , \vect{w}^{(j)} ) + \beta_{1} J_{\vect{k}}(\vect{x}+ \vect{\Delta}_{\vect{x}}^{(i)} , \phi , \vect{w}^{(j)} )  + \beta_{2} J_{\vect{k}}(\vect{x}, \phi+\Delta^{(i)}_{\phi} , \vect{w}^{(j)} )  \\ 
      &+ \beta_3 J_{\vect{k}}(\vect{x}, \phi , \vect{w}^{(j)} +\Delta^{(i)}_{\vect{w}} ),  \Big] 
	\end{aligned}
	\label{eq:eq_opt}
\end{equation}
where $\beta_k  \in \mathbb{R} \cup [0,1], \forall k$ and $\zeta \in \mathbb{N}$ indicates finite difference updates.
\label{prop:finit_app}
\end{proposition}
\begin{proof}   The proof can be found in the supplementary material. \end{proof}
Using \cref{prop:finit_app}, the upper bound to our optimization problem is
\begin{align}
 & \underset{ \vect{w}^{(j)}_{\vect{k}} }{min} \; H_{\vect{k}}(\Delta^{(i)}_{\phi}, \Delta^{(i)}_{\vect{x} }, \Delta^{(i)}_{\vect{w}},\vect{x} , \phi  ,\vect{w}^{(j)} ) \leq \underset{ \vect{w}^{(j)}_{\vect{k}} }{min} 
 \underset{\Delta_{\vect{w}_{\vect{k}}}^{(i)} \times  \Delta^{(i)}_{\phi_{\vect{k}}} \times \Delta^{(i)}_{\vect{x}_{\vect{k}}} \in  \mathcal{W}  \times \vect{\Phi} \times \vect{\mathcal{X}}}{max}  \nonumber  \\ & \Big[ J_{\vect{k}}(\vect{x}, \phi , \vect{w}^{(j)} )     + \beta_{1} J_{\vect{k}}(\vect{x}+ \vect{\Delta}_{\vect{x}}^{(i)} , \phi , \vect{w}^{(j)})    + \beta_{2} J_{\vect{k}}(\vect{x}, \phi+\Delta^{(i)}_{\phi} , \vect{w}^{(j)} )  + \beta_3 \\ &  J_{\vect{k}}(\vect{x}, \phi , \vect{w}^{(j)} +\Delta^{(i)}_{\vect{w}} )  \Big]       \leq \underset{ \vect{w}^{(j)}_{\vect{k}} }{min} \underset{\Delta_{\vect{w}_{\vect{k}}}^{(i)} \times  \Delta^{(i)}_{\phi_{\vect{k}}} \times \Delta^{(i)}_{\vect{x}_{\vect{k}}} }{max} \mathcal{H}_{\vect{k}}(\Delta^{(i)}_{\phi}, \Delta^{(i)}_{\vect{x} }, \Delta^{(i)}_{\vect{w}},\vect{x} , \phi  ,\vect{w}^{(j)}).\nonumber
\end{align} 
For notational simplification, we will pool all the maximizing parameters using a block column vector $\vect{u}$ and write \begin{align} min_{\vect{w}^{(j)}_{\vect{k}} } \quad max_{\vect{u}_{\vect{k}}^{(i)}} \quad \mathcal{H}_{\vect{k}}(\vect{u}^{(i)}, \vect{w}^{(j)} ), \label{eq:minmax_prob} \end{align} with   $ \mathcal{H}_{\vect{k}}(\vect{u}^{(i)}, \vect{w}^{(j)} ) = J_{\vect{k}}(\vect{w}^{(j)} )+ \beta_{1} J_{\vect{k}}(\vect{u}_{0}^{(i)}, \vect{w}^{(j)} )+ \beta_{2} J_{\vect{k}}(\vect{u}_{1}^{(i)} , \vect{w}^{(j)} ) + \beta_3 J_{\vect{k}}(\vect{u}_{2}^{(i)}, \vect{w}^{(j)}),$ where $\vect{u}_l$ denotes the $l^{th}$ element in the block-column vector $\vect{u}.$ Generalization is simulated through $\vect{u},$~(the delta's and the first player) and the forgetting through the policy $\vect{w}$~(the second player). The two players, $\vect{u}$ and  $\vect{w}$ chose adversarial strategies and introduce dynamics by increasing and decreasing $\mathcal{H}_{\vect{k}}.$ Many different strategies are possible but, in this work, we choose stochastic gradient ascent-descent. Over different iterations, the two players introduce the dynamics of a game using the gradient of the
$\mathcal{H}_{\vect{k}}$ to either perform ascent updates as in the case of $\vect{u}$ or descent update as in the case of $\vect{w}.$ This push-pull play will converge when the gradient of $\mathcal{H}_{\vect{k}}$ will approach zero at which point, the two players have no incentive to move and an equilibrium point~(saddle point between two players) will be achieved. In the context of GCL, this equilibrium state is known as the balance between forgetting and generalization. Thus, two theoretical questions arise, ``Is there such a balance point?," and ``Can this balance be achieved?" Later in the theoretical analysis section, we demonstrate that the answer to these questions is indeed yes. However, first, we detail an algorithm through which this two player game will be played.

Specifically, we define two datasets: $D^{P}_{k}$, the previous tasks dataset, and $D^{N}_{k}$, the new task dataset. With this setup, we define $\mathcal{W}  \times \vect{\mathcal{X}} \times \vect{\Phi}$ as $\mathcal{U}$ to be the compact search space for $\vect{u}^{(i)}$ and $\mathcal{W}$ as the search space for $\vect{w}^{(j)}.$ The learning problem is 
\begin{equation}
    \begin{aligned} \label{eq: Game_m}
    \underset{ \vect{w}^{(j)} \in \mathcal{W} }{min} \quad \underset{ \vect{u}^{(i)} \in \mathcal{U}  }{max} \quad \mathcal{H}_{\vect{k}}( \vect{u}^{(i)}, \vect{w}^{(j)} ).
    \end{aligned}
\end{equation}
This algorithm follows a two-step strategy and is detailed in \cref{alg1a}.
\begin{algorithm}
     \caption{Graph Continual Learning \label{alg1a}}
  \begin{algorithmic}[1]
\STATE{Initialize ${\vect{w}^{(j)}}, D^{N}_{k}, D^{P}_{k}.$} \\
\WHILE{$k=1,2,3, . K$}
    \STATE{j = 0}  \\
    \WHILE{$j < \rho$}
        \STATE{Fix  $\vect{w}^{(j)}$ } \\
        \WHILE{$i+1 <= \zeta$}
            \STATE{  Sample $\vect{b}_{N} \in D^{N}_{k}, \vect{b}_{P} \in D^{P}_{k}$ and get $\vect{b}_{PN} = \vect{b}_{P} \cup \vect{b}_{N}.$ 
            i = 0.}\\
            \STATE{ Update~$\vect{u}^{(i)}$ through gradient ascent on $\mathcal{H}_{\vect{k}}(\vect{u}^{(i)}, \vect{w}^{(j)} ).$} \\
            \STATE{ i = i+1. }
        \ENDWHILE
        \STATE{ Fix $\vect{u}^{(\zeta)}$}
        \STATE{ Sample $\vect{b}_{N} \in D^{N}_{k}, \vect{b}_{P} \in D^{P}_{k}$ and get $\vect{b}_{PN} = \vect{b}_{P} \cup \vect{b}_{N}.$} \\
        \STATE{ Update~$\vect{w}^{(j)}$ using gradient descent on $\mathcal{H}_{\vect{k}}(\vect{u}^{(\zeta)}, \vect{w}^{(j)} ).$} \\
        \STATE{ j= j+1.}
      \ENDWHILE
    \STATE{Update ${D_{P}}$ with  ${D_{N}}$}
    \ENDWHILE
     \end{algorithmic}
\end{algorithm}

We first update $\vect{u}^{(i)}$ and attain $\vect{u}^{*}$ in $\mathcal{U}.$ With a fixed solution in $\mathcal{U}$, we find $\vect{w}^*$ in $\mathcal{W}$ by updating $\vect{w}^{(j)}.$ With repeated iterations, we converge to the equillibrium point~$(\vect{u}^*, \vect{w}^*).$ The existence and the convergence is guaranteed through theoretical analysis in the following section. 

\section{Theoretical Analysis}\label{sec:theory}
In this section, we will define all the expected values respect to the joint probability measure $P_{\vect{x} \times \phi}$ as described in \cref{defn:task}. To perform continual learning according to \cref{alg1a}, we define $\mathcal{U}$ to be the search space for $\vect{u}^{(i)}$ and $\mathcal{W}$ is the search space for $\vect{w}^{(j)}.$ Then, the learning problem is provided as 
\begin{equation}
    \begin{aligned} \label{eq: Game}
    \underset{ \vect{w}^{(j)} \in \mathcal{W} }{min} \quad \underset{ \vect{u}^{(i)} \in \mathcal{U}  }{max}  \mathcal{H}_{\vect{k}}( \vect{u}^{(i)}, \vect{w}^{(j)} ).
    \end{aligned}
\end{equation}
We will denote $\vect{g}^{(j)}_{\vect{w}}$ and $\vect{g}^{(i)}_{\vect{u}}$ as the derivative of $H_{\vect{k}}( \vect{u}^{(i)}, \vect{w}^{(j)} )$ with respect to $\vect{w}$ and $\vect{u}$ respectively. For a minibatch sampled according to the distribution $P_{\vect{x} \times \phi}$, we will define a minibatch estimate of the gradients as $\hat{\vect{g}}_{\vect{u}}^{(i)}$ and $\hat{\vect{g}}_{\vect{w}}^{(j)}$ respectively.  We will use $\vect{g}^{(j)}_{\vect{w}}(m)$ indicating gradients with respect to the $m^{th}$ datapoint in the mini-batch. At this point, we make the following assumptions 
\begin{assumption} \label{ass:lip_bounded}
The function $J_{\vect{k}}$ is Lipschitz continuous, that is  
\begin{subequations} \begin{align}
    \|\nabla_{ \vect{u}^{(i+1)}} J_{\vect{k}} -\nabla_{ \vect{u}^{(i)} } J_{\vect{k}}  \| &\leq M \|\vect{u}^{(i+1)} - \vect{u}^{(i)}\|      \\
    \|\nabla_{ \vect{w}^{(i+1)}} J_{\vect{k}} -\nabla_{ \vect{w}^{(i)} } J_{\vect{k}}  \| &\leq L_w \|\vect{w}^{(i+1)} - \vect{w}^{(i)}\|  \\
    \forall \vect{u}^{(i+1)}, \vect{u}^{(i)} \in \mathcal{U}, & \forall \vect{w}^{(i+1)}, \vect{w}^{(i)} \in \mathcal{W}. \nonumber
\end{align} \end{subequations} 
Furthermore, the gradient is bounded with respect to all its arguments.
\begin{equation} \begin{aligned}
    \|\nabla_{ \vect{u}_{0}^{(i)}} J_{\vect{k}} \|  \leq G_{\vect{x}} &, \|\nabla_{ \vect{u}_{1}^{(i)}} J_{\vect{k}} \| \leq G_{\phi} 
    \|\nabla_{ \vect{u}_{2}^{(i)}} J_{\vect{k}} \|  \leq G_{\vect{w}} &,   \|\nabla_{ \vect{w}^{(i)}} J_{\vect{k}} \| \leq G \\ 
    \forall \vect{u}^{(i)} \in \mathcal{U}, & \forall  \vect{w}^{(i)} \in \mathcal{W}.
\end{aligned} \end{equation} 
\end{assumption}
Before presenting our results, we will bound the expected value of gradients~($\hat{\vect{g}}_{\vect{u}}^{(i)}, \hat{\vect{g}}_{\vect{w}}^{(i)}$) based on the assumptions described here.
\begin{lemma}
Let  \cref{ass:lip_bounded} be true and let the size of $D^{P}_{k} \cup D^{N}_{k}$, be described by $N$ with the batch size given by $b.$ Assume that a minibatch is obtained by sampling uniformly from the dataset and define 
$H_{\vect{k}}( \vect{u}^{(i)}, \vect{w}^{(j)} ) = [ J_{\vect{k}}(\vect{w}^{(j)} )+ \beta_{1} J_{\vect{k}}(\vect{u}_{0}^{(i)}, \vect{w}^{(j)} )+ \beta_{2} J_{\vect{k}}(\vect{u}_{1}^{(i)} , \vect{w}^{(j)} ) + \beta_3 J_{\vect{k}}(\vect{u}_{2}^{(i)}, \vect{w}^{(j)}) ]$ with $\bar{G}=[ G_{\phi}+ G_{\vect{x}} +  G_{\vect{w}} ].$ Then,  the following inequalities are true.
\begin{align} 
\|\vect{g}^{(i)}_{\vect{u}} \|                         &\leq  \beta \bar{G}   & \| \vect{g}^{(j)}_{\vect{w}} \| &\leq  \Big[(1+3\beta) G   \Big] \nonumber \\ 
\mathbb{E}\left[\| \vect{g}^{(i)}_{\vect{u}} \|\right] &\leq  \beta \frac{b\bar{G}}{N} & \mathbb{E}\left[\| \vect{g}^{(i)}_{\vect{w}} \|\right]  &\leq  \beta \frac{b(1+3\beta) G}{N} \\ 
         Var(\vect{g}^{(i)}_{\vect{u}})  &\leq \frac{2b\beta^2\left(N^2+b^2\right)}{N^3} \bar{G}^2 &  Var(\vect{g}^{(j)}_{\vect{w}}) & \leq \left[\frac{bN^2+b^3}{N^3} \right] \left[ G^2 (1+3\beta)^2\right].  \nonumber
\end{align}
    where $\|.\|$ refers to the $\ell_{2}$ norm. \label{lem:bound_grad}
\end{lemma}
In this work, the goal is to prove that a equillibrium point for the two player game exists and can be reached. Since, the two player game is sequential, we seek a local min-max point or Stackleberg equillibrium. The following definitions are adapted from \cite{jin2020local} which provides approximate conditions for sequential minmax games as 
\begin{defn}\label{def:stack_eq}.
Let there be two compact sets $\mathcal{U}$ and $\mathcal{W}$ and assume $\mathcal{H}_{\vect{k}}$ to be twice differentiable, then  $(\vect{u}^{*}, \vect{w}^{*}) \in  \mathcal{U} \times \mathcal{W}$ is said to be a local minimax point or at Stackleberg equillibrium  for $\mathcal{H}_{\vect{k}}$, if 
the following is true
\begin{equation} \begin{aligned} \label{cond_stack_eq}
    \norm{ \mathbb{E}[ \mathcal{H}_{\vect{k}}(\vect{u}^{*}, \vect{w}^{*}) - \mathcal{H}_{\vect{k}}(\vect{u}^{(i)}, \vect{w}^{*})] }&\leq \epsilon(\delta_u),       \\
\norm{ \mathbb{E} \left[ \mathcal{H}_{\vect{k}}(\vect{u}^{*}, \vect{w}^{*} ) - \underset{\vect{u}^{(i)} \in \mathcal{\vect{U}} }{max}\mathcal{H}_{\vect{k}}(\vect{u}^{(i)}, \vect{w}^{(j)}) \right] } &\leq \epsilon(\delta_w, \delta_u).
\end{aligned} \end{equation}
for every $(\vect{u}^{(i)}, \vect{w}^{(j)} ) \in   \mathcal{U} \times \mathcal{W} $ such that for any  $\delta_u, \delta_w  \in \mathbb{R}^+$ with $ \mathbb{E}[\|\vect{w}^{(j)} - \vect{w}^{*}] \| \leq \delta_w, \quad \mathbb{E}[\|\vect{u}^{(i)} - \vect{u}^{*}] \| \leq \delta_u, $ with $\epsilon(\delta_u, \delta_w),\epsilon(\delta_u)  \in \mathbb{R^{+}}.$
\end{defn}
In what follows, we will show that a local min-max point~(which is equivalent to the Stackleberg equilibrium in a two-player setting~\cite{jin2020local}) exists~( definition \cref{def:stack_eq}) and that the algorithm converges.

\subsection{Theorem 1, Existence of the minmax point}\label{theory:theorem1}
We will first show the existence of Stackleberg equillibrium. It suffices to show that there exists a $(\vect{u}^{*}, \vect{w}^{*})$ such that the conditions from definition \cref{def:stack_eq} are satisfied. This is given as follows.
\begin{theorem}[Existence of an Equilibrium Point]
For each task $k,$  fix $\vect{w}^{*} \in\mathcal{W}$ and construct $\mathcal{M}= \{\vect{\mathcal{U}},\vect{w}^{*} \}.$  Let assumption \cref{ass:lip_bounded} be true, define a dataset $D_{P} \cup D_{N}$ of size $N>0$ and sample uniformly a mini-batch  of size $b.$ Next, define
$ \mathcal{H}_{\vect{k}}( \vect{u}^{(i)}, \vect{w}^{(j)} ) =  J_{\vect{k}}(\vect{w}^{(j)} )+ \beta_{1} J_{\vect{k}}(\vect{u}_{0}^{(i)}, \vect{w}^{(j)} ) + \beta_{2} J_{\vect{k}}(\vect{u}_{1}^{(i)} , \vect{w}^{(j)} ) + \beta_3 J_{\vect{k}}(\vect{u}_{2}^{(i)}, \vect{w}^{(j)}) $ with $\beta_1, \beta_2, \beta_3 \leq \beta > 0.$ Let the inequalities from \cref{lem:bound_grad} be true. Under assumption that $\alpha^{(i)}_{u}>0,  b>0, \beta>0, N>0,$ then, there conditions in \cref{cond_stack_eq} are satisfied with  \begin{subequations}
    \begin{align}  \epsilon(\delta_u) & = \frac{M+1}{2}\delta_u^2 + \bar{G}^2 \left( \frac{1}{2}\left( \frac{b\bar{G}}{N}\right)^2 + \frac{2b\beta^2\left(N^2+b^2\right)}{N^3} \right), \\
    \epsilon(\delta_u, \delta_w) &= (\frac{L_w+1}{2})\delta_w^2 +  G^2 \left( \frac{\left[ (1+3\beta)^2\right]}{2} +  \left[\frac{bN^2+b^3}{N^3} \right] \left[ (1+3\beta)^2\right] \right) 
\\ &+ \frac{M+1}{2}\delta_u^2 + \bar{G}^2 \left( \frac{1}{2}\left( \frac{b}{N}\right)^2 + \frac{2b\beta^2\left(N^2+b^2\right)}{N^3} \right) \nonumber.\end{align}
\end{subequations} 
Here $\| \vect{u}^{(i)} - \vect{u}^{*}\| \leq \delta_u, \forall i$  and $\| \vect{w}^{(j)} - \vect{w}^{*}\| \leq \delta_w, \forall j.$ Furthermore, there exists a $(\vect{u}^{*}, \vect{w}^{*}) \in \mathcal{M} \cup \mathcal{N}$ such that  $(\vect{u}^{*}, \vect{w}^{*})$ is a local minimax point or a Stackleberg equillibrium point according to \cref{def:stack_eq}.
\end{theorem}

To show a local min-max point, we make the assumption that the two players are initialized close to the equilibrium point. This assumption is due to the lack of convexity in the learning problem. Since there is no unique equilibrium point in the nonconvex case, the best we can claim is that one converges to a local minimum. We note, however, that local minima are typically good in the sense of performance for neural networks and that  any initialization strategy such as the one in \cite{glorot2010understanding} can facilitate this local minimum. Next, we show that our algorithm converges.

\subsection{Theorem 2, Convergence to the equillibrium point}\label{theory:theorem2}
The proof of this theorem requires us to first prove that the maximizing player converges. Next, we show that, provided the maximizing player provides a strategy, the minimizing player converges. 
In our algorithm, we perform $\zeta$ gradient ascent updates for each $\vect{w}^{(j)}.$ Therefore, we will first show that the with many gradient ascent steps, our gradient goes to zero. 
\begin{theorem}[Gradient of $\mathcal{H}_{\vect{k}}$ with respect to $\vect{u}^{(i)}$ converges to zero] \label{lem: convergence_max}
For each task $k,$  fix $\vect{w}^{(j)} \in\mathcal{W}$ and construct $\mathcal{M}= \{\vect{w}^{(j)}, \vect{\mathcal{U}} \}.$ Let \cref{ass:lip_bounded} be true and sample uniformly a minibatch of size $b$ from the dataset $D$ of size $N$. Define 
$ \mathcal{H}_{\vect{k}}( \vect{u}^{(i)}, \vect{w}^{(j)} ) =  J_{\vect{k}}(\vect{w}^{(j)} )+ \beta_{1} J_{\vect{k}}(\vect{u}_{0}^{(i)}, \vect{w}^{(j)} ) + \beta_{2} J_{\vect{k}}(\vect{u}_{1}^{(i)} , \vect{w}^{(j)} ) + \beta_3 J_{\vect{k}}(\vect{u}_{2}^{(i)}, \vect{w}^{(j)}) $ with $\beta_1, \beta_2, \beta_3 \leq \beta > 0.$ Let the inequalities from \cref{lem:bound_grad} be true. Choose $\alpha^{(i)}_{u} = \frac{\alpha_{u}}{\sqrt{\zeta}},$ then $\sum_{i} \alpha^{(i)}_{u} = \sum_{i} \frac{\alpha_{u}}{ \sqrt{\zeta}} = \alpha_{u} \sqrt{\zeta}.$ Similarly, $ \sum_{i} (\alpha^{(i)}_{u})^2 =  \alpha^2_{u}$ such that
$\sum_{i} (\alpha^{(i)}_{u}-\frac{M (\alpha^{(i)}_{u})^2 }{2}) = \frac{ 2\alpha_{u} \sqrt{\zeta} - M\alpha^2_{u} }{2} = S_n$. Now, denote $\Delta_{(i)} = \mathcal{H}_{\vect{k}}(\vect{u}^{(i+1)}, \vect{w}^{(j)})-\mathcal{H}_{\vect{k}}(\vect{u}^{*}, \vect{w}^{(j)}),$ then the following is true  
\[ \underset{i=0,\cdots, \zeta}{min} E\left[ \| \vect{g}^{(i)}_{\vect{u}} \|^2 \right] \leq  \frac{2 \mathbb{E}[\| \vect{u}^{(\zeta)} - \vect{u}^{*}\|] }{2\alpha_{u} \sqrt{\zeta} - M\alpha^2_{u} } + \frac{2 M(\alpha_{u})^2 b\beta^2\bar{G}^2 )}{N\left(2\alpha_{u} \sqrt{\zeta} - M\alpha^2_{u}\right) }+\frac{2M(\alpha_{u})^2(b^3\beta^2 \bar{G}^2 )}{N^3\left( 2\alpha_{u} \sqrt{\zeta} - M\alpha^2_{u} \right) }.\] Provided $M\alpha^2_{u} << \alpha_{u} \sqrt{\zeta}$, we have $\lim_{\zeta \rightarrow \infty } \underset{i=0,\cdots, \zeta}{min} E\left[ \| \vect{g}^{(i)}_{\vect{u}} \|^2 \right]   \rightarrow 0$ with the rate $\frac{1}{\sqrt{\zeta}}.$
\end{theorem}
With the result that the gradient will converge to zero for the maximizing player, we are now ready to show the convergence of our algorithm in the sense of gradients approaching zero. At which point, the two players will have no incentive to change. Therefore, the Stackleberg equillibrium is attained. For the following result, we will assume that, for each $j$ the maximizing player has already played and provides with a $\vect{u}^{(\zeta)}$ as given by the preceeding theorem.
\begin{theorem}[Convergence in gradients] \label{theorem: grad_max}
For each task $k,$  construct $\mathcal{N}= \{\vect{\mathcal{U}}, \vect{\mathcal{W}} \}.$  Let \cref{ass:lip_bounded} be true and define a dataset $D$ of size $N>0.$ Assume that a minibatch of size $b$ is obtained by uniformly sampling from $D$ and define 
$ \mathcal{H}_{\vect{k}}( \vect{u}^{(i)}, \vect{w}^{(j)} ) =  J_{\vect{k}}(\vect{w}^{(j)} )+ \beta_{1} J_{\vect{k}}(\vect{u}_{0}^{(i)}, \vect{w}^{(j)} ) + \beta_{2} J_{\vect{k}}(\vect{u}_{1}^{(i)} , \vect{w}^{(j)} ) + \beta_3 J_{\vect{k}}(\vect{u}_{2}^{(i)}, \vect{w}^{(j)}) $ with $\beta_1, \beta_2, \beta_3 \leq \beta > 0$ Consider the updates for $\vect{u}^{(i)}$ as $\alpha^{(i)}_{u} \vect{\hat{g}}^{(i)}_{\vect{u}}$ and the updates for updates for $\vect{w}^{(j)}$ as $\alpha^{(j)}_{w} \vect{\hat{g}}^{(i)}_{\vect{w}}$ and let the inequalities from \cref{lem:bound_grad} provides the bounds on the variance and expected values of these gradients. By \cref{theorem: grad_max}, we obtain that the maximizing player $\vect{u}^{(i)}$ converges to $\vect{u}^{(\zeta)}$ Furthermore, assume that $\alpha^{(i)}_{u}>0,  b>0, \beta>0, N>0, \norm{\vect{u}^{(\zeta)}-\vect{u}^{*}}^2 \leq \delta_w^2,$ and that $\underset{j = 1, 2, 3,   \cdots \rho}{max} Var(\vect{g}^{(j)}_{\vect{w}} )$ and $\underset{j = 1, 2, 3,   \cdots \rho}{max} Var(\vect{g}^{(j)}_{\vect{w}} )$ are upper bounded by the bound provided by \cref{lem:bound_grad}. Furthermore, choose, $\alpha^{(i)}_{w} = \frac{\alpha_{w}}{ \sqrt{\rho}},$ then $\sum_{i} (\alpha^{(i)}_{w} = \sum_{i} \frac{\alpha_{w}}{\sqrt{\rho}} = \alpha_{w} \sqrt{\rho}.$ Similarly, $ \sum_{i} (\alpha^{(i)}_{w})^2 =  \alpha^2_{w}.$ Then the minimum value of the gradient is bounded as 
\begin{equation}  \begin{aligned}
   \underset{j = 1, 2, 3,  \cdots \rho}{min}  \mathbb{E}\left[  \| \vect{g}^{(j)}_{\vect{w}} \|^2\right]  &\leq  \frac{ 2\rho \mathbb{E} \left[\delta_u^2\right] (M+1) + 2(1+3\beta)G \delta_{w}2}{2\alpha_{w} \sqrt{\rho} - L_w\alpha^2_{w} } + \frac{L ( \alpha_{\vect{w}})^2 b}{ N (2 \alpha_{w} \sqrt{\rho} - L\alpha^2_{w} )} \\ & +  \frac{ \beta  \rho b^2\bar{G}^2}{N^2 (2 \alpha_{w} \sqrt{\rho} - L_w\alpha^2_{w}) } + \frac{G^2 (1+3\beta)^2 L_w (\alpha_{\vect{w}})^2 b^3}{(N^3 2\alpha_{w} \sqrt{\rho} -L_w\alpha^2_{w} ) }. 
\end{aligned} \end{equation}
where $G$ and $\bar{G}$ are provided by \cref{lem:bound_grad} with 
and the gradient converges asymptotically to zero with the rate $\frac{1}{\sqrt{\rho}}$ under the assumption that $2\alpha_{w} \sqrt{\rho} >>L_w\alpha^2_{w}$ . 
\end{theorem}

Note that the convergence again depends on how effectively the parameters are initialized. Moreover, the equilibrium is not exact but  approximate; that is, the pair $(\vect{u}^{(i)}, \vect{w}^{(j)})$ reaches within a ball around a local equilibrium point $(\vect{u}^{*}, \vect{w}^{*}).$ The size of this ball is dependent on the batch size, learning rate $\alpha_w$, size of $D^{P}_{k}$, and  number of updates $\zeta$ and  $\rho$. These ideas cater to the usual intuition that the larger the dataset, the better the convergence. Similarly, a large number of updates lead to convergence, and better initialization always allows a network to approach a better minimum. The theorems and proofs presented in this section are the first convergence results for using stochastic gradient ascent-descent strategies in the GCL literature.

\section{Experiments}\label{sec: results}
We consider the Cora, CiteSeer, and Reddit datasets~\cite{zhou2021overcoming} for vertex  classification problems and consider Mutag and Proteins~\cite{morris2020tudataset} for graph classification. We compare our method with the state-of-the-art experience replay-driven method in the GCL literature ~\cite{zhou2021overcoming}. Note that \cite{zhou2021overcoming} reports results only on vertex  classification problems. We adopted the experimental setting and datasets from the current state-of-the-art papers \cite{zhou2021overcoming} and \cite{liu2021overcoming}. In this way we can ensure a fair comparison between two experience-replay-based methods. Furthermore, we fix the size of the memory buffer to be 500. Given the lack of large CL benchmarks for graphs, we conducted a unique large-scale hyperparameter search (with 1,000 high-performing hyperparameter configurations) to confirm the robustness of our approach. Moreover, we conducted ablation studies to establish the effectiveness of different components of our method. With all these experimental results, we seek to mitigate the dearth of large standard GCL benchmarks for empirical comparison. We compared our method with other methods on vertex classification datasets. We utilized graph classification datasets to study and analyze the stability of our approach to different hyperparameters. All experiments were conducted in Python 3.4 using the pytorch $1.7.1$ library with the NVIDIA-A100 GPU.

\begin{table}[tbhp] 
\scriptsize
    \begin{center}
        \caption{Vertex classification problem}
        \subfloat[Performance mean~(PM) \label{tab:vertexPM_class}]{
    \begin{tabular}{c|c|c|c}
    \toprule
     &\multicolumn{1}{c}{\textbf{Cora}}& \multicolumn{1}{|c}{\textbf{CiteSeer}}   &\multicolumn{1}{|c}{\textbf{Reddit}}   \\ \hline
                            &PM$\uparrow$        &PM$\uparrow$.    &PM$\uparrow$  \\
                              \cline{2-4}
        DeepWalk            &85.63       &64.79     &76.93  \\ 
        node2Vec            &85.99       &65.18     &78.24\\
        GraphSage           &94.15       &81.26     &95.01\\
        GIN                 &90.17       &74.92     &93.75\\
        GCN                 &93.62       &80.63     &94.43  \\
        SGC                 &93.06       &78.18     &94.01  \\
        GAT                 &94.19       &81.48     &93.84 \\ \hline \hline
        ER-GAT-MF           &94.15       &80.03     &94.18  \\
        ER-GAT-MF*          &94.23       &81.83     &94.63 \\
        ER-GAT-CM           &93.98       &78.78     &93.33  \\ 
        ER-GAT-CM*          &94.25       &80.86     &94.23  \\ 
        ER-GAT-IM           &$\mathbf{95.66}$  &80.85  &$\mathbf{95.36}$   \\ \hline
        Ours                &$91.51$ &$\mathbf{90.34}$  &$94.32$   \\ 
    \bottomrule 
    \end{tabular} }
    \subfloat[Forgetting metric~(FM) \label{tab:vertexFM_class} ]{
    \begin{tabular}{c|c|c}
    \toprule
     \multicolumn{1}{c}{\textbf{Cora}}& \multicolumn{1}{|c}{\textbf{CiteSeer}}   &\multicolumn{1}{|c}{\textbf{Reddit}}   \\ \hline
                        FM$\downarrow$        &FM$\downarrow$    &FM$\downarrow$  \\ \cline{1-3}
                34.51          &25.92              &33.24  \\ 
               35.46          &24.87              &34.66  \\
                  37.73          &28.06              &40.06 \\
                   33.81          &27.42              &36.28  \\
                      31.90          &25.47              &35.17   \\
                   33.93          &28.31              &38.59  \\
                 30.84          &23.73              &32.79 \\ \hline \hline
             22.49          &17.96              &26.44  \\
              21.88          &17.83              &23.54 \\
               22.14          &18.03              &26.17  \\ 
        21.03          &17.86              &23.15  \\ 
          21.14          &17.08              &23.09 \\ \hline
          $\mathbf{7.58}$ &$\mathbf{3.64}$  &$\mathbf{14.21}$ \\ 
    \bottomrule 
    \end{tabular} }
    \end{center}
\end{table}

\textbf{Metrics:} We used the same metrics as in \cite{zhou2021overcoming} for comparison: performance mean~(PM) and forgetting mean~(FM). These metrics use either accuracy ($acc$) or micro-F1 score ($f_1$) based on the dataset. Whenever a new task was observed, we recorded two quantities: (i) \textit{task $acc$ or $f_1$} of the model on the new task and (ii) \textit{forgetting $acc$ or $f_1$}---the difference in the $acc$ or $f_1$ of the model before and after the new task was observed. Once all the tasks are observed, FM is observed as the average \textit{forgetting $acc$ or $f_1$}, and PM is observed as the average \textit{task $acc$ or $f_1$}. For PM, the higher score is better; for FM, the lower score is better. Note, however, note that the value of PM is not upper bounded by the FM value since the PM values are the measure of pure generalization, whereas FM is the metric of both generalization and forgetting. A methodology is high performing if it achieves low FM and high PM. 

\textbf{Vertex  classification:} The baselines for the vertex classification problem are directly taken from \cite{zhou2021overcoming}. They are  Deepwalk~\cite{perozzi2014deepwalk}, node2Vec~\cite{grover2016node2vec}, graph convolutional networks~\cite{kipf2016semi}~(GCNs), GraphSAGE~\cite{hamilton2017inductive}, Graph Attention~\cite{velivckovic2017graph}~(GAT), Simple Graph Convolutions~\cite{wu2019simplifying}~(SGC), and  Graph Isomorphism Network~\cite{xu2018powerful}~(GIN). Furthermore, the work in \cite{zhou2021overcoming} introduces five new models:  ER-GAT-MF, ER-GAT-MF*, ER-GAT-CM, ER-GAT-CM*, and ER-GAT-IM, where MF is the mean of the attributes, MF* is the mean of embeddings, CM is the  attribute space coverage maximization, CM is  the embedding space maximization, and IM is  the influence maximization. As in \cite{zhou2021overcoming} we constructed three 2-way tasks, namely, two classes per task for the Cora and CiteSeer datasets. For Reddit, we constructed eight 5-way tasks. For the network, we utilized two layers of GAT with two layers of dropouts, similar to what is used in \cite{zhou2021overcoming}. We used the Adam optimizer learning rate of $10^{-03}$ for gradient descent and $10^{-07}$ for gradient ascent with $\rho = 1000$ and $\zeta=10$; and we utilized the 80-20 training-testing split. We use double precision for all our simulations and therefore, even $10^{-07}$ introduces changes in the weights. We summarize these results in \cref{tab:vertexFM_class} and \cref{tab:vertexPM_class}. We evaluate FM and PM on $acc$ for Cora while evaluating these metrics on $f_1$ score for Reddit and CiteSeer, as in \cite{zhou2021overcoming}. We report the mean and standard deviation of these metrics over  100 runs, where a distinct random seed was utilized for each run, and mark the best scores in bold.

Except for our results, all other numbers are taken directly from \cite{zhou2021overcoming} where no standard deviation numbers or results on PubMed were reported. The results show that our approach achieves superior performance in the FM when compared with all methods discussed in \cite{zhou2021overcoming}. To quantify, for Cora, our method obtains an FM of $7.58$, which is 67\% improvement over ER-GAT-IM with $21.14$: $\left((21.14-7.58)/21.14) \times 100\right)$.  Similarly, 78\% and 44\% improvements are observed in CiteSeer and Reddit datasets, respectively. 

Our method also does  well on the PM scale: the PM values achieved by our method are either comparable to or better than all the others. These results are  reflected in the last two rows of \cref{tab:vertexPM_class}. Earlier, we claimed that a good GCL methodology must achieve a balance between generalization and forgetting. A low FM value coupled with a high PM value supports this claim and supports the narrative of Theorem 1, where it was shown that such a balance point exists. 

\textbf{Stability analysis:} Next we demonstrate the stability of our approach to supporting the claims of Theorem 2. We utilized the \textit{graph classification} datasets Mutag and Protein for this study. In both these cases, we used the 60-20-20 split for training-testing and validation. We are  able to report scores only for our method in  \cref{tab:graph_classification} since the only method available for comparison~\cite{zhou2021overcoming} does not report any $acc/f_1$ on graph classification problems. In the first row of \cref{tab:graph_classification}, we report the performance using 1-FM that is achieved by  training a model on all the tasks together~(assuming all the data from these tasks is available); we call this \textit{joint training}.  For any continual learning approach, the accuracy achieved by joint training is an absolute upper bound. In what follows we study the impact of hyperparameters on the effectiveness of our proposed GCL method. To that end, we utilized DeepHyper~(DH)~\cite{egele2021agebo}---a scalable software package for hyperparameter tuning. Leveraging DH, we ran a hyperparameter search where we selected hyperparameter configurations~($nlays$--number of layers, $drop$-- dropout rate, $hc$-- number of hidden channels in graph attention layers with  $\alpha, \rho$ and $\zeta$) that improve the objective~$(FM).$ Once DH has run for a sufficient number of iterations, we selected  the hyperparameters that provide top 30\% quantile of objective values (FM) distribution.

\begin{wrapfigure}[20]{r}{0.60\textwidth}
\includegraphics[width = 0.20\columnwidth,height=5.7cm]{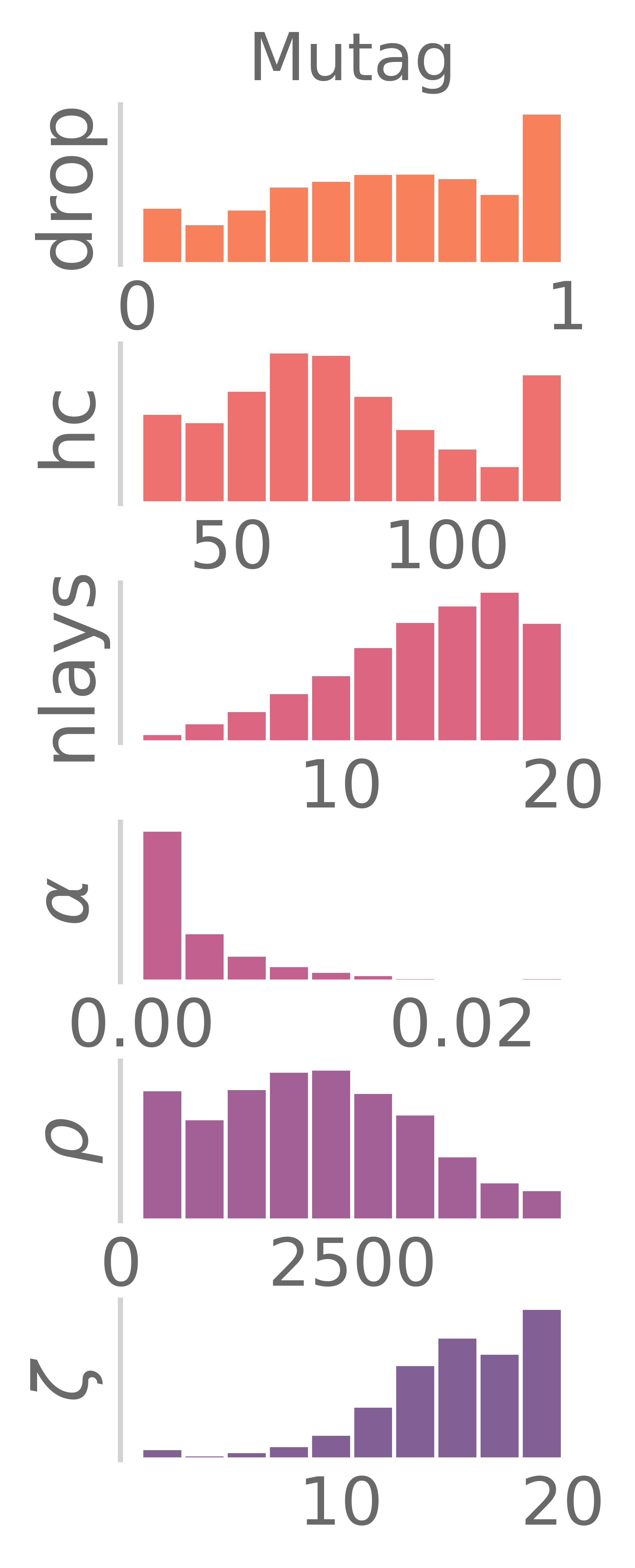}
\includegraphics[width = 0.20\columnwidth,height=5.7cm]{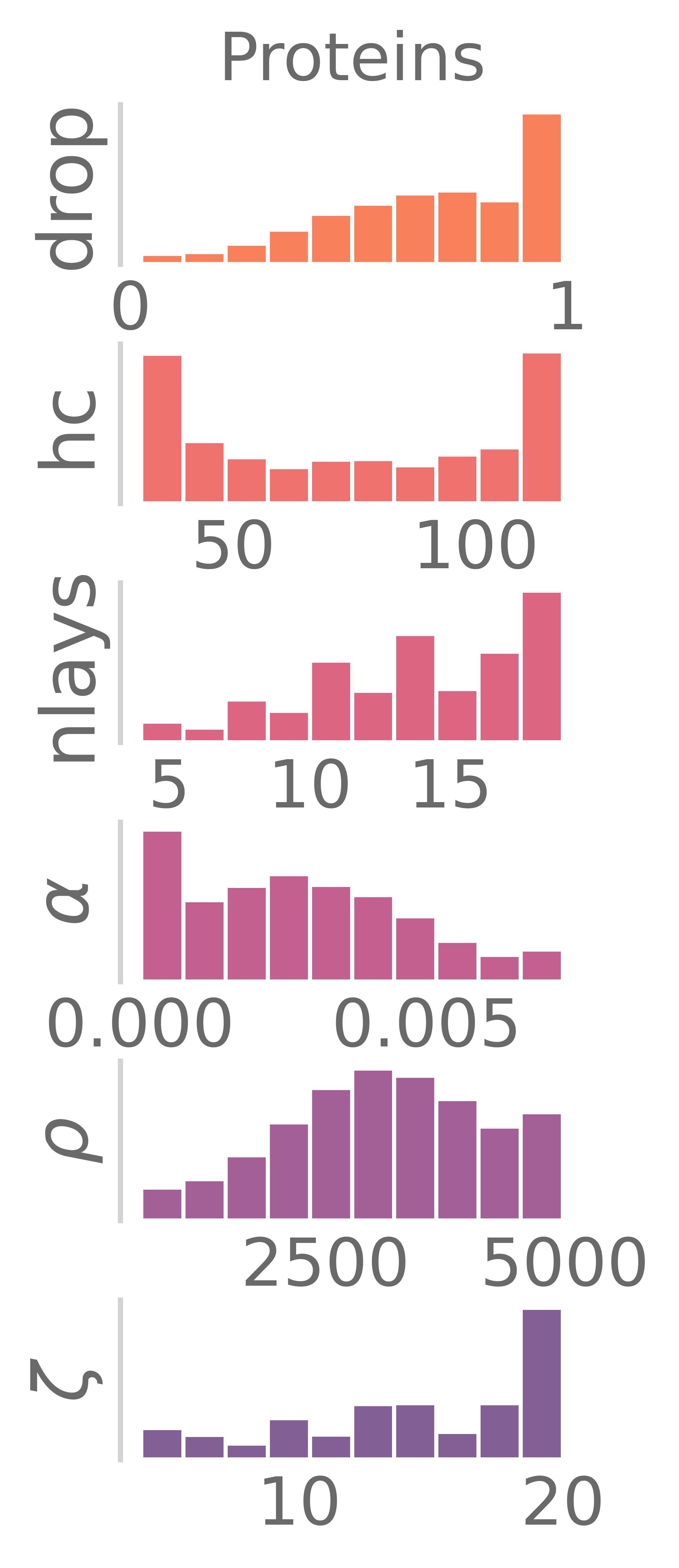}
\includegraphics[width = 0.20\columnwidth,height=5.7cm]{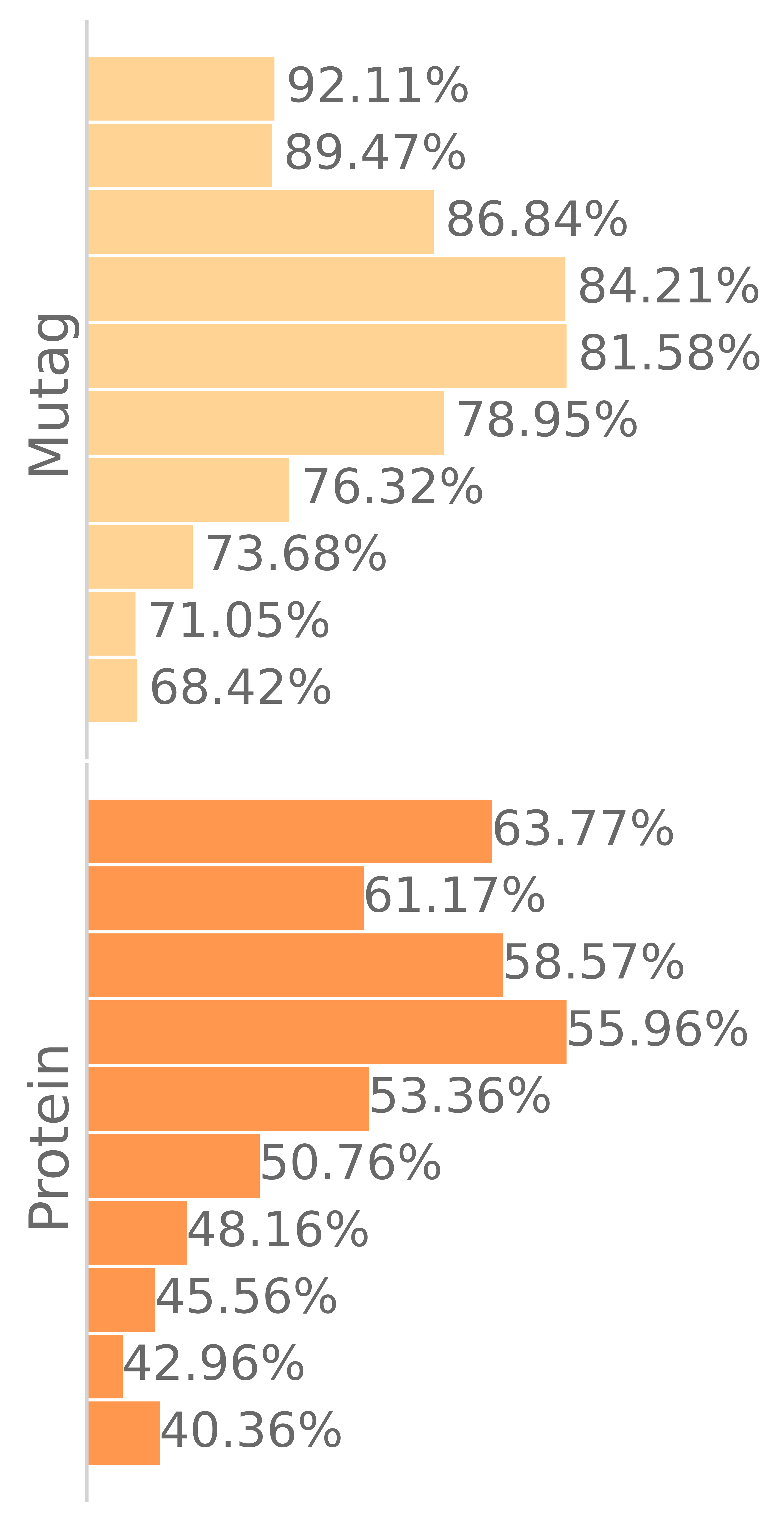}
\caption{(left)~Histograms of different hyperparameter configurations found by DeepHyper and (right)~the corresponding FM distribution.\label{fig:hist}}
\end{wrapfigure}
To observe the variance of the objective values corresponding to these hyperparemters selected by DH, we seek to understand what are the different sets of hyperparameters that provide good FM for these datasets. Typically, if DH has been executed for a large number of iterations, the hyperparameters providing the top 30\% quantile FM value would be high performing~(will provide extremely high values of FM). However, running DH for such a long number of iterations requires considerable high-performance computing  resources. To alleviate this need, we  fit a Gaussian copula~\cite{7796926}~(GC) to mimic the distribution of these hyperparameters~(those providing the top 30\% quantile in FM) and sample 1,000 new high-performing hyperparameters. We then evaluated all of them in the CL setting in parallel with a distinct random seed.  


In Table~\cref{tab:graph_classification} we record the best, worst, and  $mean\pm std$ values obtained by these 1,000 models. Furthermore we jointly trained a model using the hyperparameter configuration that provided the best FM values according to DH;  the score is recorded  in the first row in \cref{tab:graph_classification}. Note that for both  datasets, the best accuracy is close to the upper bound achieved by jointly training a model, while the worst $acc$ are  farther away from the mean. 

In \cref{fig:hist}~(right), we record the histogram of different FM values obtained by our approach when DH is utilized to find appropriate hyperparameters. Note  that for both  datasets, the FM values are skewed toward the mean~(the mean is coinciding with the mode). In the left and middle of \cref{fig:hist}, we plot the histograms of different hyperparameter values found by DH.  In addition to indicating  to the practitioner  which hyperparameters are good for using our methodology, they  illustrate several other insights. The histogram of hyperparameters for each dataset is distinct,  and the histograms for these hyperparameters are broad. For instance, the values of $\rho$ range for Mutag vary between 1 and 4000. Theses observations, together with the fact that the mean FM is skewed toward the mode, indicate that our approach is very stable to different hyperparameter configurations and different initializations of weights. In other words, even with varying hyperparameters and initial values of weights, our approach provides reasonable FMs. 

\begin{wraptable}[14]{l}{0.40\textwidth}
\scriptsize
\begin{center}
 \begin{tabular}{c|c|c}
    \toprule &\textbf{Mutag} &\textbf{Proteins}\\ \cline{2-3}                          &(1-FM)$\uparrow$        &(1-FM)$\uparrow$  \\ \cline{2-3}
     Joint training        &$93$        &$66$    \\ \hline \hline 
     Best                 &$94$        &$66$    \\
     Worst              &$68$        &$40$   \\
     Mean(std)            &$83\pm 6.1$ &$56 \pm 7$  \\
     \bottomrule
\end{tabular}
\caption{Scores on graph classification. The best, minimum, and mean(std) are evaluated based on DeepHyper's hyperparameter search.  \textit{Joint training} refers to the process of training a model on all the tasks together. \label{tab:graph_classification}}

\end{center}
\end{wraptable} Since the histogram of a particular hyperparameter for the two datasets is very distinct,  
 hyperparameter configurations cannot be commonly set for different datasets. Furthermore,  our results follow several commonly held notions with respect to setting hyperparameters. A larger number of layers appear to  favor the top 30\% quantile of FM. Furthermore, larger $\rho$ and large $\zeta$ are favored  (mean around $3000$ in both cases). The intuition is that within the two-player game, more  updates for both the outer loop and the inner loop lead to better convergence~(supports the validity of Theorem~2 result where asymptotic convergence is guaranteed). Peculiarly, larger dropout values are also favored. More investigation is required to analyze this behavior. A larger number of layers also appear to provide better performance, which also adheres to the intuition that better performance for deeper architectures.

\textbf{Ablation study:} Our cost function $\mathcal{H}_{\vect{k}}( \vect{u}^{(i)}, \vect{w}^{(j)} )$, which summarizes the game of two players $\vect{u}^{(i)}$ and $\vect{w}^{(j)}$, comprises four terms as in \cref{eq:eq_opt}. Here, the first term summarizes the forgetting and generalization cost. The second and third terms quantify the impact of simulated change in graphs. The fourth term quantifies the impact of simulated change in the parameters of the model. In this study we provide insight into the contribution of these different terms on the performance of a model~(a GAT with two dropout layers) for the Mutag dataset. Specifically, we initialize a model using high-performing hyperparameters~(chosen through DH) and perform three experiments. We perform CL on Mutag and record the FM score over 10 runs~(each run is executed with distinct random seed) when (1) all terms in $\mathcal{H}_{\vect{k}}( \vect{u}^{(i)}, \vect{w}^{(j)} )$ are utilized; (2) first, second, and third terms in  $\mathcal{H}_{\vect{k}}( \vect{u}^{(i)}, \vect{w}^{(j)} )$ are utilized; and (3) just the first term in $\mathcal{H}_{\vect{k}}( \vect{u}^{(i)}, \vect{w}^{(j)} )$ is utilized.

While the second scenario is equivalent to the case when we switch off the game, the third scenario is equivalent to regular experience replay strategy. We observe that the best performance on the Mutag dataset is observed when all the terms in the cost are utilized and the two-player game is played. We attain a performance of 89\%. When the game is switched off, we suffer a 14\% deterioration, and the deterioration further increases by switching off the regularization term where a 21\% deterioration in performance is observed. The study proves that the game actually does contribute to improved performance in the graph continual learning scenario.

\section{Conclusion}
We presented a new theoretical framework for graph continual learning. We modeled the stochastic process underlying the GCL problem as a vertex edge random graph. We formulated the GCL problem from an adaptive dynamic programming viewpoint and derived a partial differential equation to model the dynamics of GCL. We developed a theoretically sound two-player game-driven methodology for the GCL setting. We demonstrated that our proposed method achieved 44\% improvement compared with the state of the art on vertex classification benchmarks. We presented an ablation study, wherein we showed that the game performance improves  by 21\%. With a large-scale analysis we confirmed that our approach is stable to a variety of hyperparameters. 

Our future work will include (1) GCL for spatial-temporal data, (2) GCL for other classes of non-Euclidean data, (3) development of  problem-agnostic representation learning for GCL, and (4) applications to molecule property prediction tasks.


\section{Acknowledgement}
This work is funded by the Department of Energy under
the Integrated Computational and Data Infrastructure (ICDI) for Scientific Discovery, grant DE-SC0022328. 
We also acknowledge the support by the U.S. Department of Energy for the SCIDAC5-RAPIDS institute and the 
DOE Early Career Research Program award. This research used resources of the Argonne Leadership Computing Facility, 
which is a DOE Office of Science User Facility supported under Contract DE-AC02- 06CH11357.
\pagebreak



\section{Preliminaries}
The GCL task is given as
\begin{defn}[GCL task in the continuous sense]\label{defn:task_SM} For $t, \Delta t \in \vect{\Omega}$, define the interval $[t, t + \Delta t]$ and let $(\mathcal{G}_{V}(t), P_{\vect{x}(t) \times  \phi(t)})$  represent a VERG associated with GCL. Denote the GNN model as $g(., ., \vect{w}(t)): \vect{\Omega} \rightarrow \mathbb{R}^{n}$ with a loss function given as $\ell: \mathbb{R}^n \rightarrow \mathbb{R}.$ Let $J(\vect{x}([t, t+\Delta t]), \phi([t, t+\Delta t]), \vect{w}([t, t+\Delta t])) = \int_{\tau = t}^{t+\Delta t} \ell(\vect{x}_{\vect{V}}(\tau), \phi_{\vect{V}}(\tau), \vect{w}(\tau))]$ be the forgetting and generalization cost over the interval $[t, t+\Delta t].$  Then, a GCL task $\mathcal{T}([t, t+\Delta t])$ is described by the tuple \begin{multline} \Big( \vect{x}([t, t+\Delta t]), \phi([t, t+\Delta t]),  J(\vect{x}([t, t+\Delta t]), \phi([t, t+\Delta t]), \vect{w}([t, t+\Delta t]))  \Big)\end{multline} with $ \vect{x}([t, t+\Delta t]) = \{\vect{x}_{\vect{V}}(\tau) \forall \tau \in [t, t+\Delta t] \}_{\vect{V}(t)}$ and  $\phi([t, t+\Delta t]) = \{\vect{\phi}_{\vect{V}}(\tau) \forall \tau \in [t, t+\Delta t] \}_{\vect{V}(t)}$
\end{defn}
For simplicity of notations, we will denote the task as  $\mathcal{T}_{[ t, t+\Delta t]}$ which is described by the tuple  $\left( (\vect{x}, \phi)_{[ t, t+\Delta t]}, J_{[t, t+\Delta t]} (\vect{x}, \phi, \vect{w}) \right)$ where the subscript indicates the interval over which the task is defined. This notation, easily extends to a collection of tasks. For instance, all the tasks in the interval $[0,t]$ are collectively provided by $\mathcal{T}_{[ 0, t]}.$ As we use $\vect{t}$ to represent the interval $[0,t],$ it follows that $\mathcal{T}_{[ 0, t]}$ is rewritten as $\mathcal{T}_{\vect{t}}$ which represents all tasks in the interval $[0,t]$ where $\mathcal{T}_{\vect{t}} = \left( (\vect{x}, \phi)_{\vect{t}}, J_{\vect{t}} (\vect{x}, \phi, \vect{w}) \right) $ with $ J_{\vect{t}} (\vect{x}, \phi, \vect{w}) =  \int_{\tau = 0}^{t} \ell(\vect{x}(\tau), \phi(\tau), \vect{w}(\tau)).$ This notation naturally extends to the case when $\vect{\Omega}$ is comprised of discrete instance as well. In this case, we will set $\Delta t =1$ and replace $t$ by $k$ such that $[0,t] = [0,k] = [0,1,2,3,\cdots, k] = \vect{k}$ Furthermore, the collection of all tasks in the interval $[0,k]$ is given by $\mathcal{T}_{\vect{k}} = \left( (\vect{x}, \phi)_{\vect{k}},\ell_{\vect{k}} (\vect{x}_{\vect{V}}, \phi_{\vect{V}}, \vect{w}) \right)$ with $J_{\vect{k}}(\vect{x} , \phi, \vect{w})= \sum_{\tau = 0}^{k} \ell(\vect{x} (\tau), \phi (\tau), \vect{w}(\tau))$ where $g(\vect{x}(k), \phi(k))$ is the parametric map and $\ell(\vect{x} (k), \phi (k), \vect{w}(k))$ is the corresponding loss with the VERG defined by the probability space $(\mathcal{G}(k), P_{\vect{x}(k) \times  \phi(k)})), k \in \vect{\Omega}.$  

\newpage
\section{Dynamical System Modelling}
\begin{proposition}[Dynamics of the GCL problem]
Define a domain $\vect{\Omega}$ with $t \in \vect{\Omega}$ and a vertex set $\vect{V}(t):  \vect{\Omega} \rightarrow  \boldsymbol{\mathcal{V}} | \vect{V}(t) \subset \boldsymbol{\mathcal{V}} $ Then, define the CL task as in  \cref{defn:task_SM}. For each $\vect{t} \subset \vect{\Omega}$ let
\begin{equation}
    \begin{aligned}
      L^*_{\vect{t}} =  \underset{ \vect{w}(\vect{t}) }{min} \int_{\tau=t}^{\Omega}   J_{\boldsymbol{\tau} }(\vect{x}, \phi , \vect{w} )
    \end{aligned}
\end{equation}
as the GCL problem. Assume $  J_{\vect{t}}(\vect{x}, \phi , \vect{w} )$ to be smooth with respect to all its arguments. Under the assumption that $R(\vect{t})$ denotes all the higher order terms in a Taylor series expansion, the following is true
\begin{equation}
    \begin{aligned}
       - \left( \partial_{ \vect{t} } L^*_{\vect{t}} \right)^T \Delta_{\vect{t}} & =  \underset{ \vect{w}( \vect{t} ) }{min} \Big[ J_{\vect{t}}(\vect{x}, \phi , \vect{w} ) 
       + \left( \partial_{\vect{x}_{\vect{t}}} L^*_{\vect{t}} \right)^T \Delta_{\vect{x} }  \\ 
       & + \left( \partial_{ \phi_{\vect{t}} } L^*_{\vect{t}} \right)^T \Delta_{ \phi} 
       + \left( \partial_{\vect{w}_{\vect{t}} } L^*_{\vect{t}} \right)^T \Delta_{\vect{w} } + R_{\vect{t}}, \Big]
    \end{aligned}
\end{equation}
where $\Delta{x}$ refers to the first derivative of $x$ with respect to $\vect{t}.$
\end{proposition}

\begin{proof}
Let the GCL problem be given as 
\begin{equation}
    \begin{aligned}
     L^*_{\vect{t}}  = \underset{ \vect{w}_{\vect{t}}  }{min} \int_{\tau = t}^{\Omega}   J_{\tau}(\vect{x}, \phi , \vect{w} ).
    \end{aligned}
\end{equation}
Split the integral with the time interval $,t$rewrite the optimization problem as
\begin{equation}
    \begin{aligned}
        L_{\vect{t}}^{*} &=& \Big[ \underset{ \vect{w}_{\vect{t}}}{min}  \quad  J_{\vect{t}}(\vect{x}, \phi , \vect{w} ) + \int_{\tau = t+ \Delta t}^{\Omega} \underset{ \vect{w}_{\tau} }{min}  \quad  J_{\vect{\tau}}(\vect{x}, \phi , \vect{w} ) \Big]
    \end{aligned}
\end{equation}
Using the policy $\vect{w}_{\vect{t}} $ if we begin at $\vect{t},$ $L_{\vect{t}}^{*}$ provides the optimal cost over the complete interval. Since, $\underset{ \vect{w}_{\vect{t}}  }{min} \int_{ t}^{\Omega} $ $J_{\vect{t}}(\vect{x}, \phi , \vect{w} ).$ is $L_{\vect{t}}^{*}$, then it stands to reason that  $ \underset{ \vect{w}_{\vect{t}}  }{min} $ $\int_{\tau = t + \Delta t}^{\Omega}   J_{\vect{t}}(\vect{x}, \phi , \vect{w} ).$ is the optimal cost after $\vect{t}.$ That is, it is the optimal cost for the interval $\vect{\Omega}^C$ while following policy $\vect{w}(\vect{t}),$  which we denote as $L^*_{\vect{t}^+}$
This provides
\begin{equation}
    \begin{aligned} \label{eq:eq_split}
        L^*_{\vect{t}} =  \underset{ \vect{w}_{\vect{t}} }{min}  \Big[ J_{\vect{t}}(\vect{x}, \phi , \vect{w} ) +
        L^*_{\vect{t}^+}\Big].
    \end{aligned}
\end{equation}
Now, all the information about the future must be approximated using $\vect{t}.$ To do this approximation, we will write the  Taylor series expansion of $L^*_{\vect{t}^{+}}$ around $\vect{t}.$ Since \[V^{*}: \vect{\Omega} \rightarrow (\mathcal{X}, \Phi) \rightarrow \vect{G} \xrightarrow[]{\vect{w}} \mathbb{R}.\]  Taylor series expression is expanded around $$(\vect{t}, \vect{w}(\vect{t}), \vect{x}(\vect{t}), \phi( \vect{t}) ).$$ to obtain
\begin{multline}
       L^*_{\vect{t}^{+}}
       = L^*_{\vect{t}} + \left( \partial_{ \vect{t} }L^*_{\vect{t}} \right)^T \Delta_{\vect{t}}  + \left( \partial_{\vect{x}_{\vect{t}} } L^*_{\vect{t}} \right)^T \Delta_{\vect{x} } \\
       + \left( \partial_{ \phi_{\vect{t}} } L^*_{\vect{t}} \right)^T \Delta_{ \phi} \
       + \left( \partial_{\vect{w}_{\vect{t}}} L^*_{\vect{t}} \right)^T \Delta_{\vect{w}}
       + \cdots,
\end{multline}
Substitution into \eqref{eq:eq_split} reveals
\begin{multline}
        L^*_{\vect{t}} =  \underset{ \vect{w}_{\vect{t}} }{min} \Big[  J_{\vect{t}}(\vect{x}, \phi , \vect{w} ) +  L^*(\vect{t}) + \left( \partial_{ \vect{t} }L^*(\vect{t}) \right)^T \Delta_{\vect{t}} + \left( \partial_{\vect{x}_{\vect{t}} } L^*_{\vect{t}}\right)^T \Delta_{\vect{x} } \\ + \left( \partial_{ \phi_{\vect{t}} } L^*_{\vect{t}} \right)^T \Delta_{ \phi }+ \left( \partial_{\vect{w}_{\vect{t}}} L^*_{\vect{t}} \right)^T \Delta_{\vect{w} } + \cdots,  \Big]
\end{multline}
We will now cancel the common terms and write
\begin{align}
       0 &=  \underset{ \vect{w}_{\vect{t}} }{min} \Big[ J_{\vect{t}}(\vect{x}, \phi , \vect{w} )  + \left( \partial_{ \vect{t} }L^*_{\vect{t}} \right)^T \Delta_{\vect{t}} 
       + \left( \partial_{\vect{x}_{\vect{t}}} L^*_{\vect{t}} \right)^T \Delta_{\vect{x} } \nonumber \\ &+ \left( \partial_{ \phi_{\vect{t}} } L^*_{\vect{t}} \right)^T \Delta_{ \phi} 
       + \left( \partial_{\vect{w}_{\vect{t}} } L^*_{\vect{t}} \right)^T \Delta_{\vect{w} } + R_{\vect{t}}, \Big]
\end{align}
where $R(\vect{t})$ summarizes all the terms in the $\cdots$, i.e, the higher order terms from the Taylor series. Moving terms in the equation provides graph PDE as 
\begin{align}
          - \left( \partial_{ \vect{t} }L^*(\vect{t}) \right)^T \Delta_{\vect{t}} & =  \underset{ \vect{w}( \vect{t} ) }{min} \Big[ J_{\vect{t}}(\vect{x}, \phi , \vect{w} ) 
       + \left( \partial_{\vect{x}_{\vect{t}}} L^*_{\vect{t}} \right)^T \Delta_{\vect{x} } \nonumber \\ & + \left( \partial_{ \phi_{\vect{t}} } L^*_{\vect{t}} \right)^T \Delta_{ \phi} 
       + \left( \partial_{\vect{w}_{\vect{t}} } L^*_{\vect{t}} \right)^T \Delta_{\vect{w} } + R_{\vect{t}}, \Big]
    \end{align}
\end{proof}

\newpage
\section{Discrete Time Approximation}
\begin{proposition}
Let $\vect{k} \in \vect{\Omega}$ and define $\mathcal{W}, \vect{\mathcal{X}}, \vect{\Phi}$ such that $ \Delta_{\vect{w}}^{(i)} \times  \Delta^{(i)}_{\phi} \times \Delta^{(i)}_{\vect{x}} \in  \mathcal{W}  \times \vect{\Phi} \times \vect{\mathcal{X}}$ and assume that
\begin{subequations} \begin{align}
    sup_{\vect{\mathcal{X}} } L^*_{\vect{k}} &\leq inf_{\vect{\mathcal{X}}} J_{\vect{k}}(\vect{x} , \phi , \vect{w}^{(j)} ) \leq max_{\vect{\Delta}^{(i)}_{\vect{x}_{\vect{k}}} \in \vect{\mathcal{X}} } J_{\vect{k}}(\vect{x}+ \vect{\Delta}_{\vect{x}}^{(i)} , \phi , \vect{w}^{(j)} )      \\
    sup_{\vect{\Phi}} L^*_{\vect{k}} &\leq  inf_{\vect{\Phi}} J_{\vect{k}}(\vect{x}, \phi , \vect{w}^{(j)} ) \leq max_{\Delta^{(i)}_{\phi_{\vect{k}}} \in  \vect{\Phi} } J_{\vect{k}}(\vect{x} , \phi + \Delta^{(i)}_{\phi}  , \vect{w}^{(j)}) \\
     sup_{\mathcal{W}} L^*_{\vect{k}} &\leq inf_{\mathcal{W}} J_{\vect{k}}(\vect{x}, \phi , \vect{w}^{(j)} ) \leq max_{\Delta^{(i)}_{\vect{w}} \in \mathcal{W}} J_{\vect{k}}(\vect{x} , \phi , \vect{w}^{(j)} +\Delta^{(i)}_{\vect{w}}) \\ 
     inf_{\vect{\mathcal{X}} } L^*_{\vect{k}} &\geq 0,  inf_{\vect{\Phi}} L^*_{\vect{k}} \geq 0, 
     inf_{\mathcal{W}} L^*_{\vect{k}} \geq 0. \end{align} \end{subequations}
Then, define $\mathcal{H}_{}(\Delta^{(i)}_{\phi}, \Delta^{(i)}_{\vect{x} },\vect{x} , \phi  ,\vect{w}^{(j)} ) =  J_{\vect{k}}(\vect{x}, \phi , \vect{w}^{(j)} ) + \left( \partial_{\vect{x}_{\vect{k}} } L^*_{\vect{k}}\right)^T \Delta^{(i)}_{\vect{x}_{\vect{k}}}     \\ + \left( \partial_{ \phi_{\vect{k}} } L^*_{\vect{k}} \right)^T \Delta_{ \phi_{\vect{k}} } + \left( \partial_{\vect{w}^{(j)}_{\vect{k}}} L^*_{\vect{k}} \right)^T \Delta^{(i)}_{\vect{w}_{\vect{k}}}$ and the following approximation is true
\begin{equation}
	\begin{aligned}
	& H_{\vect{k}}(\Delta^{(i)}_{\phi}, \Delta^{(i)}_{\vect{x} },  \Delta^{(i)}_{\vect{w} },\vect{x} , \phi  ,\vect{w}^{(j)} )  \leq max_{\Delta_{\vect{w}_{\vect{k}}}^{(i)} \times  \Delta^{(i)}_{\phi_{\vect{k}}} \times \Delta^{(i)}_{\vect{x}_{\vect{k}}} \in  \mathcal{W}  \times \vect{\Phi} \times \vect{\mathcal{X}}} \\  
      &\Big[ J_{\vect{k}}(\vect{x}, \phi , \vect{w}^{(j)} ) + \beta_{1} J_{\vect{k}}(\vect{x}+ \vect{\Delta}_{\vect{x}}^{(i)} , \phi , \vect{w}^{(j)} )  + \beta_{2} J_{\vect{k}}(\vect{x}, \phi+\Delta^{(i)}_{\phi} , \vect{w}^{(j)} )  \\ 
      &+ \beta_3 J_{\vect{k}}(\vect{x}, \phi , \vect{w}^{(j)} +\Delta^{(i)}_{\vect{w}} ),  \Big] 
	\end{aligned}
\end{equation}
where $\beta_k  \in \mathbb{R} \cup [0,1], \forall k$ and $\zeta \in \mathbb{N}$ indicates finite difference updates.
\end{proposition}
\begin{proof}
\begin{remark}
    Note that the value function  is a function of all the arguments in the ODE. That is, $\Delta^{(i)}_{\phi}, \Delta^{(i)}_{\vect{x} },  \Delta^{(i)}_{\vect{w} },\vect{x} , \phi  ,\vect{w}^{(j)}$ but the arguments are not explicitly denoted for notational simplicity. Just for this  proposition, we shall indicate the arguments when required to describe the dependence through which the first derivatives shall exist. After the following proposition, we go back to the original notation of  $L_{\vect{k}}^{*}.$  
\end{remark}

By Euler's approximation, we obtain
\begin{equation}
	\begin{aligned}
     \left( \partial_{\vect{x}_{\vect{k}} } L^*_{\vect{k}}\right)^T \Delta^{(i)}_{\vect{x}}   &= \left( L^*_{\vect{k}}(\vect{x} + \Delta_{\vect{x}}^{(i)}) - L^*_{\vect{k}}(\vect{x}) \right) \frac{\left(\Delta_{\vect{x}}^{(i)}  \right)  }{\left( \Delta_{\vect{x}}^{(i)} \right)^T }\\
     &=\frac{L^*_{\vect{k}}(\vect{x} + \Delta_{\vect{x}}^{(i)}) - L^*_{\vect{k}} }{\Delta_{\vect{k}}}  \\
     &\leq \frac{sup_{\vect{\mathcal{X}}} L^*_{\vect{k}}  - inf_{\vect{\mathcal{X}}} L^*_{\vect{k}} }{\Delta_{\vect{k}}} \\
     & \leq \frac{1}{\Delta_{\vect{k}}} max_{\vect{\Delta}_{\vect{x}}^{(i)} \in \vect{\mathcal{X}} } J_{\vect{k}}(\vect{x}+ \vect{\Delta}_{\vect{x}}^{(i)} , \phi , \vect{w}^{(j)} ) \\
     & \leq \beta_{1} max_{\vect{\Delta}_{\vect{x}}^{(i)} \in \vect{\mathcal{X}} } J_{\vect{k}}(\vect{x}+ \vect{\Delta}_{\vect{x}}^{(i)} , \phi , \vect{w}^{(j)} ), \forall \beta_1 \in \mathbb{R}^+
	\end{aligned}
\end{equation}
Where the fourth and fifth inequality follows from assumption and $^{\dag} $ indicates the psuedo inverse. Similarly, we may write
\begin{equation} \begin{aligned}
  \left( \partial_{ \phi_{\vect{k}} } L^*_{\vect{k}} \right)^T \Delta_{ \phi }  &\leq  \beta_{2} max_{\vect{\Delta}^{(i)}_{\phi_{\vect{k}}} \in \vect{\Phi} } J_{\vect{k}}(\vect{x}, \phi+\Delta^{(i)}_{\phi}, \vect{w}^{(j)} )     \\
   \left( \partial_{\vect{w}^{(j)}_{\vect{k}}} L^*_{\vect{k}} \right)^T \Delta^{(i)}_{\vect{w}} &\leq \beta_3 max_{\vect{\Delta}_{\vect{w}_{\vect{k}}}^{(i)} \in \mathcal{W} } J_{\vect{k}}(\vect{x}, \phi , \vect{w}^{(j)} +\Delta^{(i)}_{\vect{w}} ) 
\end{aligned} \end{equation} 
with $\beta_2, \beta_3 \in \mathbb{R}^+,$ we obtain
\begin{equation} \begin{aligned}
     H_{\vect{k}}(\Delta^{(i)}_{\phi}, \Delta^{(i)}_{\vect{x} },\vect{x} , \phi  ,\vect{w}^{(j)} )  &\leq J_{\vect{k}}(\vect{x}, \phi , \vect{w}^{(j)} )     \\ & + \beta_{1} max_{\vect{\Delta}_{\vect{x}}^{(i)} \in \vect{\mathcal{X}} } J_{\vect{k}}(\vect{x}+ \vect{\Delta}_{\vect{x}}^{(i)} , \phi , \vect{w}^{(j)} )      \\ & + \beta_{2} max_{\vect{\Delta}_{\phi}^{(i)} \in \vect{\Phi} } J_{\vect{k}}(\vect{x}, \phi+\Delta^{(i)}_{\phi} , \vect{w}^{(j)} )  \\ &+ \beta_3 max_{\vect{\Delta}_{\vect{w}}^{(i)} \in \mathcal{W} } J_{\vect{k}}(\vect{x}, \phi , \vect{w}^{(j)} +\Delta^{(i)}_{\vect{w}} )  \nonumber
\end{aligned} \end{equation} 
Pulling the maximization outside provides the result.
\begin{equation} \begin{aligned}
     & H_{\vect{k}}(\Delta^{(i)}_{\phi}, \Delta^{(i)}_{\vect{x} },  \Delta^{(i)}_{\vect{w} },\vect{x} , \phi  ,\vect{w}^{(j)} )  \leq max_{\Delta_{\vect{w}_{\vect{k}}}^{(i)} \times  \Delta^{(i)}_{\phi_{\vect{k}}} \times \Delta^{(i)}_{\vect{x}_{\vect{k}}} \in  \mathcal{W}  \times \vect{\Phi} \times \vect{\mathcal{X}}}      \\ & \Bigg[ J_{\vect{k}}(\vect{x}, \phi , \vect{w}^{(j)} ) + \beta_{1} J_{\vect{k}}(\vect{x}+ \vect{\Delta}_{\vect{x}}^{(i)} , \phi , \vect{w}^{(j)} )       \\ &  + \beta_{2} J_{\vect{k}}(\vect{x}, \phi+\Delta^{(i)}_{\phi} , \vect{w}^{(j)} )  + \beta_3 J_{\vect{k}}(\vect{x}, \phi , \vect{w}^{(j)} +\Delta^{(i)}_{\vect{w}} )  \Bigg] 
\end{aligned} \end{equation} 
\end{proof}

\newpage
\section{Theoretical Analysis}
First, we will show that the gradients are bounded under the following assumption 
\begin{assumption} \label{ass:lip_bounded_SM}
The function $J_{\vect{k}}$ is Lipschitz continuous, that is  
\begin{subequations} \begin{align}
    \|\nabla_{ \vect{u}^{(i+1)}} J_{\vect{k}} -\nabla_{ \vect{u}^{(i)} } J_{\vect{k}}  \| &\leq M \|\vect{u}^{(i+1)} - \vect{u}^{(i)}\|      \\
    \|\nabla_{ \vect{w}^{(i+1)}} J_{\vect{k}} -\nabla_{ \vect{w}^{(i)} } J_{\vect{k}}  \| &\leq L_w\|\vect{w}^{(i+1)} - \vect{w}^{(i)}\|  \\
    \forall \vect{u}^{(i+1)}, \vect{u}^{(i)} \in \mathcal{U}, & \forall \vect{w}^{(i+1)}, \vect{w}^{(i)} \in \mathcal{W}. \nonumber
\end{align} \end{subequations} 
Furthermore, the gradient is bounded with respect to all its arguments.
\begin{equation} \begin{aligned}
    \|\nabla_{ \vect{u}_{0}^{(i)}} J_{\vect{k}} \|  \leq G_{\vect{x}} &, \|\nabla_{ \vect{u}_{1}^{(i)}} J_{\vect{k}} \| \leq G_{\phi} 
    \|\nabla_{ \vect{u}_{2}^{(i)}} J_{\vect{k}} \|  \leq G_{\vect{w}} &,   \|\nabla_{ \vect{w}^{(i)}} J_{\vect{k}} \| \leq G \\ 
    \forall \vect{u}^{(i)} \in \mathcal{U}, & \forall  \vect{w}^{(i)} \in \mathcal{W}.
\end{aligned} \end{equation} 
\end{assumption}
Thus providing the following lemma
\begin{lemma}
Let assumption \ref{ass:lip_bounded_SM} be true and let the size of $D^{P}_{k} \cup D^{N}_{k}$, be described by $N$ with the batch size given by $b.$ Assume that a minibatch is obtained by sampling uniformly from the dataset and define 
$H_{\vect{k}}( \vect{u}^{(i)}, \vect{w}^{(j)} ) = [ J_{\vect{k}}(\vect{w}^{(j)} )+ \beta_{1} J_{\vect{k}}(\vect{u}_{0}^{(i)}, \vect{w}^{(j)} )+ \beta_{2} J_{\vect{k}}(\vect{u}_{1}^{(i)} , \vect{w}^{(j)} ) + \beta_3 J_{\vect{k}}(\vect{u}_{2}^{(i)}, \vect{w}^{(j)}) ]$ with $\bar{G}=[ G_{\phi}+ G_{\vect{x}} +  G_{\vect{w}} ].$ Then,  the following inequalities are true.
\begin{align} 
\|\vect{g}^{(i)}_{\vect{u}} \|                         &\leq  \beta \bar{G}   & \| \vect{g}^{(j)}_{\vect{w}} \| &\leq  \Big[(1+3\beta) G   \Big] \nonumber \\ 
\mathbb{E}\left[\| \vect{g}^{(i)}_{\vect{u}} \|\right] &\leq  \beta \frac{b\bar{G}}{N} & \mathbb{E}\left[\| \vect{g}^{(i)}_{\vect{w}} \|\right]  &\leq  \beta \frac{b(1+3\beta) G}{N} \\ 
         Var(\vect{g}^{(i)}_{\vect{u}})  &\leq \frac{2b\beta^2\left(N^2+b^2\right)}{N^3} \bar{G}^2 &  Var(\vect{g}^{(j)}_{\vect{w}}) & \leq \left[\frac{bN^2+b^3}{N^3} \right] \left[ G^2 (1+3\beta)^2\right].  \nonumber
\end{align}
 \end{lemma}
\begin{proof}
We begin by stating the cost function as 
\begin{equation} \begin{aligned}
    H_{\vect{k}}( \vect{u}^{(i)}, \vect{w}^{(j)} ) &= \Big[ J_{\vect{k}}(\vect{w}^{(j)} )+ \beta_{1} J_{\vect{k}}(\vect{u}_{0}^{(i)}, \vect{w}^{(j)} ) \\ &+ \beta_{2} J_{\vect{k}}(\vect{u}_{1}^{(i)} , \vect{w}^{(j)} ) + \beta_3 J_{\vect{k}}(\vect{u}_{2}^{(i)}, \vect{w}^{(j)})  \Big]
\end{aligned} \end{equation} 
Taking derivative with respect to $\vect{u}$ both sides provides 
\begin{subequations} \begin{align}
    \vect{g}^{(i)}_{\vect{u}} & =  \nabla_{\vect{u}^{(i)}} H_{\vect{k}}( \vect{u}^{(i)}, \vect{w}^{(j)} ) \\ & = \nabla_{\vect{u}^{(i)}} \Big[ J_{\vect{k}}(\vect{w}^{(j)} )+ \beta_{1} J_{\vect{k}}(\vect{u}_{0}^{(i)}, \vect{w}^{(j)} )  + \beta_{2} J_{\vect{k}}(\vect{u}_{1}^{(i)} , \vect{w}^{(j)} ) \\ &+ \beta_3 J_{\vect{k}}(\vect{u}_{2}^{(i)}, \vect{w}^{(j)})\Big], \nonumber
\end{align} \end{subequations} 
which leads to
\begin{subequations} \begin{align}
    \vect{g}^{(i)}_{\vect{u}} &  = \nabla_{\vect{u}^{(i)}} \Big[ \beta_{1} J_{\vect{k}}(\vect{u}_{0}^{(i)}, \vect{w}^{(j)} )    + \beta_{2} J_{\vect{k}}(\vect{u}_{1}^{(i)} , \vect{w}^{(j)} ) \\ 
    &+ \beta_3 J_{\vect{k}}(\vect{u}_{2}^{(i)}, \vect{w}^{(j)})   \Big] \nonumber \\ 
    &  =  \Big[ \beta_{1} \nabla_{\vect{u}^{(i)}} J_{\vect{k}}(\vect{u}_{0}^{(i)}, \vect{w}^{(j)} )    + \beta_{2} \nabla_{\vect{u}^{(i)}} J_{\vect{k}}(\vect{u}_{1}^{(i)} , \vect{w}^{(j)} ) \\ &+ \beta_3 \nabla_{\vect{u}^{(i)}} J_{\vect{k}}(\vect{u}_{2}^{(i)}, \vect{w}^{(j)})   \Big] \nonumber 
\end{align} \end{subequations} 
whence the boundedness assumption on the gradients along with the $\beta_1, \beta_2, \beta_3 \leq \beta$ provides
\begin{equation} \begin{aligned}
    \| \vect{g}^{(i)}_{\vect{u}} \| & \leq  \beta \bar{G} \\
    \mathbb{E}\left[\| \vect{g}^{(i)}_{\vect{u}} \|\right] &\leq  \beta \sum_{b} \frac{1}{N} \bar{G} \leq \beta  \frac{b\bar{G}}{N}
\end{aligned} \end{equation} 
Similarly, we obtain 
\begin{equation} \begin{aligned}
    \| \vect{g}^{(j)}_{\vect{w}} \| \leq  \Big[(1+3\beta) G    \Big]
\end{aligned} \end{equation} 
with 
\begin{equation} \begin{aligned}
    E\left[\| \vect{g}^{(j)}_{\vect{w}} \|\right] & \leq  \sum_{b} \frac{1}{N} \left[(1+3\beta) G    \right]
                                       & \leq  \frac{b}{N} \left[(1+3\beta) G    \right]
\end{aligned} \end{equation} 
where $b$ refers to the batch size and $N$ refers to the total number of samples with $\vect{g}^{(j)}_{\vect{w}}(j)$ indicating gradient with respect to the $j^{th}$ datapoint in the batch. The estimators of variance is provided as
\begin{equation} \begin{aligned}
   Var(\vect{g}^{(i)}_{\vect{u}}) & = \sum_b \frac{1}{N} \|\vect{g}^{(i)}_{\vect{u}} - \hat{\vect{g}}^{(i)}_{\vect{u}} \|^2  \\
   & \leq \sum_b \frac{1}{N} \left[ \|\vect{g}^{(i)}_{\vect{u}} \|^2 + \|\hat{\vect{g}}^{(i)}_{\vect{u}} \|^2  + 2 \|\vect{g}^{(i)}_{\vect{u}}\| \|\hat{\vect{g}}^{(i)}_{\vect{u}} \|  \right]   \\
   & \leq \sum_b \frac{1}{N} \left[ 2\|\vect{g}^{(i)}_{\vect{u}} \|^2 + 2\|\hat{\vect{g}}^{(i)}_{\vect{u}} \|^2\right]  \\
    & \leq  \sum_b \frac{2}{N}  \left[ (\beta)^2 + (\frac{b \beta}{N})^2 \right]\bar{G}^2\\ 
    & \leq \sum_b \frac{2\beta^2}{N} \left(1+\frac{b^2}{N^2}\right) \bar{G}^2 \\
    & \leq \frac{2b\beta^2\left(N^2+b^2\right)}{N^3} \bar{G}^2,
\end{aligned} \end{equation} 
where the third inequality is obtained by applying the Young's inequality. Similarly, we obtain
\begin{equation} \begin{aligned}
   Var(\vect{g}^{(j)}_{\vect{w}}) & = \sum_b \frac{1}{N} \|\vect{g}^{(j)}_{\vect{w}} - \hat{\vect{g}}^{(i)}_{\vect{w}} \|^2  \\ & \leq \sum_b \frac{1}{N}  \|\vect{g}^{(j)}_{\vect{w}} \|^2 + \|\hat{\vect{g}}^{(i)}_{\vect{w}} \|^2  +  2 \| \hat{\vect{g}}^{(i)}_{\vect{w}}\| \|\vect{g}^{(j)}_{\vect{w}}\|     \\
   & \leq \sum_b \frac{1}{N} \left[ 2\|\vect{g}^{(j)}_{\vect{w}} \|^2 + 2\|\hat{\vect{g}}^{(i)}_{\vect{w}} \|^2  \right]   \\
   & \leq \sum_{b} \frac{1}{N}  \left[ G^2 (1+3\beta)^2 +\frac{b^2 G^2 (1+3\beta)^2}{N^2} \right] \\
   & \leq \left[\frac{bN^2+b^3}{N^3} \right] \left[ G^2 (1+3\beta)^2\right]
\end{aligned} \end{equation} 
\end{proof}


\newpage 
In this work, the goal is to prove that a equillibrium point for the two player game exists and can be reached. Since, the two player game is sequential, we seek a local min-max point or Stackleberg equillibrium. The following definitions are adapted from \cite{jin2020local} which provides approximate conditions for sequential minmax games as 
\begin{defn}\label{def:stack_eq_SM}.
Let there be two compact sets $\mathcal{U}$ and $\mathcal{W}$ and assume $\mathcal{H}_{\vect{k}}$ to be twice differentiable, then  $(\vect{u}^{*}, \vect{w}^{*}) \in  \mathcal{U} \times \mathcal{W}$ is said to be a local minimax point or at Stackleberg equillibrium  for $\mathcal{H}_{\vect{k}}$, if 
the following is true
\begin{equation} \begin{aligned} \label{cond_stack_eq_SM}
    \norm{ \mathbb{E}[ \mathcal{H}_{\vect{k}}(\vect{u}^{*}, \vect{w}^{*}) - \mathcal{H}_{\vect{k}}(\vect{u}^{(i)}, \vect{w}^{*})] }&\leq \epsilon(\delta_u),       \\
\norm{ \mathbb{E} \left[ \mathcal{H}_{\vect{k}}(\vect{u}^{*}, \vect{w}^{*} ) - \underset{\vect{u}^{(i)} \in \mathcal{\vect{U}} }{max}\mathcal{H}_{\vect{k}}(\vect{u}^{(i)}, \vect{w}^{(j)}) \right] } &\leq \epsilon(\delta_w, \delta_u).
\end{aligned} \end{equation}
for every $(\vect{u}^{(i)}, \vect{w}^{(j)} ) \in   \mathcal{U} \times \mathcal{W} $ such that for any  $\delta_u, \delta_w  \in \mathbb{R}^+$ with $ \mathbb{E}[\|\vect{w}^{(j)} - \vect{w}^{*}] \| \leq \delta_w, \quad \mathbb{E}[\|\vect{u}^{(i)} - \vect{u}^{*}] \| \leq \delta_u, $ with $\epsilon(\delta_u, \delta_w),\epsilon(\delta_u)  \in \mathbb{R^{+}}.$
\end{defn}
In what follows, we will show that a local min-max point~(which is equivalent to the Stackleberg equilibrium in a two-player setting~\cite{jin2020local}) exists~( definition \cref{def:stack_eq_SM}) and that the algorithm converges.

\section{Proof of Theorem 1, Existence of the minmax point}
\begin{theorem*}[Existence of an Equilibrium Point]
For each task $k,$  fix $\vect{w}^{*} \in\mathcal{W}$ and construct $\mathcal{M}= \{\vect{\mathcal{U}},\vect{w}^{*} \}.$  Let assumption \cref{ass:lip_bounded_SM} be true, define a dataset $D_{P} \cup D_{N}$ of size $N>0$ and sample uniformly a mini-batch  of size $b.$ Next, define
$ \mathcal{H}_{\vect{k}}( \vect{u}^{(i)}, \vect{w}^{(j)} ) =  J_{\vect{k}}(\vect{w}^{(j)} )+ \beta_{1} J_{\vect{k}}(\vect{u}_{0}^{(i)}, \vect{w}^{(j)} ) + \beta_{2} J_{\vect{k}}(\vect{u}_{1}^{(i)} , \vect{w}^{(j)} ) + \beta_3 J_{\vect{k}}(\vect{u}_{2}^{(i)}, \vect{w}^{(j)}) $ with $\beta_1, \beta_2, \beta_3 \leq \beta > 0.$ Let the inequalities from \cref{lem:bound_grad} be true. Under assumption that $\alpha^{(i)}_{u}>0,  b>0, \beta>0, N>0,$ then, there conditions in \cref{cond_stack_eq_SM} are satisfied with  \begin{subequations}
    \begin{align}  \epsilon(\delta_u) & = \frac{M+1}{2}\delta_u^2 + \bar{G}^2 \left( \frac{1}{2}\left( \frac{b\bar{G}}{N}\right)^2 + \frac{2b\beta^2\left(N^2+b^2\right)}{N^3} \right), \\
    \epsilon(\delta_u, \delta_w) &= (\frac{L+1}{2})\delta_w^2 +  G^2 \left( \frac{\left[ (1+3\beta)^2\right]}{2} +  \left[\frac{bN^2+b^3}{N^3} \right] \left[ (1+3\beta)^2\right] \right) 
\\ &+ \frac{M+1}{2}\delta_u^2 + \bar{G}^2 \left( \frac{1}{2}\left( \frac{b}{N}\right)^2 + \frac{2b\beta^2\left(N^2+b^2\right)}{N^3} \right) \nonumber.\end{align}
\end{subequations} 
Here $\| \vect{u}^{(i)} - \vect{u}^{*}\| \leq \delta_u, \forall i$  and $\| \vect{w}^{(j)} - \vect{w}^{*}\| \leq \delta_w, \forall j.$ Furthermore, there exists a $(\vect{u}^{*}, \vect{w}^{*}) \in \mathcal{M} \cup \mathcal{N}$ such that  $(\vect{u}^{*}, \vect{w}^{*})$ is a local minimax point or a Stackleberg equillibrium point according to definition \cref{def:stack_eq_SM}
\end{theorem*}
\begin{proof}
The proof of this theorem has two specific inequalities to prove. These inequilities make up the condition that guarantees Stackleberg equilibrium as detailed in definition \ref{def:stack_eq_SM}. We prove the first inequality, i.e., $\mathbb{E} \left[ \mathcal{H}_{\vect{k}}(\vect{u}^{(i)}, \vect{w}^{*})  -\mathcal{H}_{\vect{k}}(\vect{u}^{*}, \vect{w}^{*}) \right] \leq \epsilon(\delta_u)$
and determine the bound for $\epsilon(\delta_u)$ through the next lemma.
\begin{lemma*}
For each task $k,$  fix $\vect{w}^{*} \in\mathcal{W}$ and construct $\mathcal{M}= \{\vect{\mathcal{U}},\vect{w}^{*} \}.$ Let assumption \cref{ass:lip_bounded_SM} be true and sample uniformly a minibatch of size $b$ from the dataset $D = D_{P} \cup D_{N}$ of size $N$. Define 
$ \mathcal{H}_{\vect{k}}( \vect{u}^{(i)}, \vect{w}^{(j)} ) =  J_{\vect{k}}(\vect{w}^{(j)} )+ \beta_{1} J_{\vect{k}}(\vect{u}_{0}^{(i)}, \vect{w}^{(j)} ) + \beta_{2} J_{\vect{k}}(\vect{u}_{1}^{(i)} , \vect{w}^{(j)} ) + \beta_3 J_{\vect{k}}(\vect{u}_{2}^{(i)}, \vect{w}^{(j)}) $ with $\beta_1, \beta_2, \beta_3 \leq \beta > 0.$ Then, $  \norm{ \mathbb{E} \left[ \mathcal{H}_{\vect{k}}(\vect{u}^{*}, \vect{w}^{*}) - \mathcal{H}_{\vect{k}}(\vect{u}^{(i)}, \vect{w}^{*}) \right] } \leq  \epsilon(\delta_u)$ with \begin{align} \epsilon(\delta_u) = \frac{M+1}{2}\delta_u^2 + \bar{G}^2 \left( \frac{2b\beta^2}{N}  + \frac{1 b^2}{2N^2}+ + \frac{2b^3\beta^2}{N^3} \right).\end{align}
\end{lemma*}
\begin{proof}
For each task $k,$  fix $\vect{w}^{*} \in\mathcal{W}$ and construct $\mathcal{M}= \{\vect{\mathcal{U}},\vect{w}^{*} \}.$ Therefore, for $ (\vect{u}^{ (i+1 )}, \vect{w}^{*}) , ( \vect{u}^{(i)}, \vect{w}^{*}) \in \mathcal{M}$  assuming L-smoothness of $\mathcal{H}_{\vect{k}}( \vect{u}^{(i)}, \vect{w}^{*} ) $ we write
\begin{subequations}
 \begin{align}
 \mathcal{H}_{\vect{k}}(\vect{u}^{*}, \vect{w}^{*})    
  &\leq  \mathcal{H}_{\vect{k}}(\vect{u}^{(i)}, \vect{w}^{*}) + \langle \vect{g}^{(i)}_{\vect{u}}, \vect{u}^{*} -\vect{u}^{(i)} \rangle  \\ & +   \frac{M}{2}\| \vect{u}^{(i)} - \vect{u}^{*}\|^2, \nonumber \\
  &\leq \mathcal{H}_{\vect{k}}(\vect{u}^{(i)}, \vect{w}^{*}) + \frac{1}{2}\| \vect{g}^{(i)}_{\vect{u}}\|^2  \\ & +  \frac{M+1}{2}\|\vect{u}^{(i)} - \vect{u}^{*}\|^2,  \Bigg| \text{\textit{By Young's Inequality}} \nonumber
\end{align}
\end{subequations}
 Take conditional expectation that conditioned on $\vect{u}^{(i)}$ to obtain
\begin{subequations} \begin{align}
  \mathbb{E} \left[ \mathcal{H}_{\vect{k}}(\vect{u}^{*}, \vect{w}^{*})     | \vect{u}^{(i)} \right]
  & \leq \mathbb{E} \left[  \mathcal{H}_{\vect{k}}(\vect{u}^{(i)}, \vect{w}^{*})  | \vect{u}^{(i)} \right] 
  \nonumber \\ &+ \mathbb{E} \left[ \frac{1}{2}\| \vect{g}^{(i)}_{\vect{u}}\|^2 | \vect{u}^{(i)} \right] 
  \\ &+  \mathbb{E} \left[ \frac{M+1}{2}\|\vect{u}^{(i)} - \vect{u}^{*}\|^2| \vect{u}^{(i)} \right],  
  \nonumber \\ & \quad \quad \Bigg| \text{\textit{$Var(x) = E[x^2] - E[x]^2$ }} \nonumber
  \\ \nonumber \\ 
  &\leq  \mathbb{E} \left[  \mathcal{H}_{\vect{k}}(\vect{u}^{(i)}, \vect{w}^{*})  | \vect{u}^{(i)} \right]  + \frac{1}{2}\| \mathbb{E} \left[  \vect{g}^{(i)}_{\vect{u}}  | \vect{u}^{(i)} \right] \|^2 
  \\&+   \frac{M+1}{2}\| \vect{u}^{(i)} - \vect{u}^{*}\|^2 + Var\left(  \vect{g}^{(i)}_{\vect{u}}  | \vect{u}^{(i)} \right),  \nonumber
\end{align} \end{subequations} 
Integrate out the $\vect{u}^{(i)} $ by law of total expectation, substitute the bounds and rearrange to write
\begin{equation} \begin{aligned}
  \mathbb{E} \left[ \mathcal{H}_{\vect{k}}(\vect{u}^{*}, \vect{w}^{*}) - \mathcal{H}_{\vect{k}}(\vect{u}^{(i)}, \vect{w}^{*}) \right]
  &\leq  \frac{1}{2}\left( \frac{b\bar{G}}{N}\right)^2 +   \frac{M+1}{2}\| \vect{u}^{(i)} - \vect{u}^{*}\|^2 \\& + \frac{2b\beta^2\left(N^2+b^2\right)}{N^3} \bar{G}^2. \\ \\ 
\end{aligned} \end{equation} 
Under the assumption that $\| \vect{u}^{(i)} - \vect{u}^{*}\| \leq \delta_u,$ we obtain
\begin{equation} \begin{aligned}
  \norm{ \mathbb{E} \left[ \mathcal{H}_{\vect{k}}(\vect{u}^{*}, \vect{w}^{*}) - \mathcal{H}_{\vect{k}}(\vect{u}^{(i)}, \vect{w}^{*}) \right] }
  &\leq  \epsilon(\delta_u), \\ \\ 
\end{aligned} \end{equation} 
with \[ \epsilon(\delta_u) = \frac{M+1}{2}\delta_u^2 + \bar{G}^2 \left( \frac{2b\beta^2}{N}  + \frac{1 b^2}{2N^2}+ + \frac{2b^3\beta^2}{N^3} \right).\]
\end{proof}

Now, we prove the second inequality which is 
\[ \norm{ \mathbb{E} \left[ \mathcal{H}_{\vect{k}}(\vect{u}^{*}, \vect{w}^{*} ) - \underset{\vect{u}^{(i)} \in \mathcal{\vect{U}} }{max}\mathcal{H}_{\vect{k}}(\vect{u}^{(i)}, \vect{w}^{(j)}) \right]} \leq \epsilon(\delta_w, \delta_u)\]
\begin{lemma*}
For each task $k,$  fix $\vect{w}^{*} \in\mathcal{W}$ and construct $\mathcal{M}= \{\vect{\mathcal{U}},\vect{w}^{*} \}.$ Let assumption 2 be true and sample uniformly a minibatch of size $b$ from the dataset $D$ of size $N$. Define 
$ \mathcal{H}_{\vect{k}}( \vect{u}^{(i)}, \vect{w}^{(j)} ) =  J_{\vect{k}}(\vect{w}^{(j)} )+ \beta_{1} J_{\vect{k}}(\vect{u}_{0}^{(i)}, \vect{w}^{(j)} ) + \beta_{2} J_{\vect{k}}(\vect{u}_{1}^{(i)} , \vect{w}^{(j)} ) + \beta_3 J_{\vect{k}}(\vect{u}_{2}^{(i)}, \vect{w}^{(j)}) $ with $\beta_1, \beta_2, \beta_3 \leq \beta > 0.$ Let the inequalities from Lemma \ref{lem:bound_grad} be true. Under assumption that $\alpha^{(i)}_{u}>0,  b>0, \beta>0, N>0$ and  
\begin{align} \underset{\begin{array}{c} \vect{u}^{(i)} \in \mathcal{U} \\ \| \vect{u}^{(i)} - \vect{u}^{*}\| \leq \delta_u \forall i \end{array}}{max}  \mathcal{H}_{\vect{k}}(\vect{u}^{(i)}, \vect{w}^{(j)})  = \mathcal{H}_{\vect{k}}(\vect{u}^{(\zeta)}, \vect{w}^{(j)}). \end{align}
Then, the second inequality is true with \begin{align} \epsilon(\delta_u, \delta_w) &= (\frac{L+1}{2})\delta_w^2 +  (1+3\beta)^2 G^2 \left( \frac{1}{2} +  \frac{b}{N} + \frac{b^3}{N^3}  \right) + \frac{M+1}{2}\delta_u^2 \\ & + \bar{G}^2 \left( \frac{2b\beta^2}{N}  + \frac{1 b^2}{2N^2}+ + \frac{2b^3\beta^2}{N^3} \right). \nonumber  \end{align}
\end{lemma*}
\begin{proof}
For each task $k,$ construct $\mathcal{N}= \{\mathcal{U}, \vect{\mathcal{W}}\}.$ Therefore, for $ (\vect{u}^{(\zeta)}, \vect{w}^{(j)}) ,$  $ (\vect{u}^{*}, \vect{w}^{*}) \in \mathcal{N}$  assuming L-smoothness of $\mathcal{H}_{\vect{k}}(\vect{u}^{*}, \vect{w}^{(j)})$ and under the assumption  we write
\begin{equation} \begin{aligned}
 \mathcal{H}_{\vect{k}}(\vect{u}^{*}, \vect{w}^{*}) &\leq  \mathcal{H}_{\vect{k}}(\vect{u}^{(\zeta)}, \vect{w}^{(j)})   
       + \langle \vect{g}^{(j)}_{\vect{w}}, \vect{w}^{*}-\vect{w}^{ (j)} \rangle +   \frac{L}{2}\| \vect{w}^{(j)} - \vect{w}^{*}\|^2 
   \\ &+ \langle \vect{g}^{(\zeta)}_{\vect{u}}, \vect{u}^{*}-\vect{u}^{(\zeta)} \rangle +   \frac{M}{2}\| \vect{u}^{(\zeta)} - \vect{u}^{*}\|^2, \\ \\ 
      &\leq \mathcal{H}_{\vect{k}}(\vect{u}^{*}, \vect{w}^{*}) + \frac{1}{2}\| \vect{g}^{(i)}_{\vect{w}}\|^2 +  \frac{L+1}{2}\|\vect{w}^{(i)} - \vect{w}^{*}\|^2 
      \\ & + \frac{1}{2}\| \vect{g}^{(\zeta)}_{\vect{u}}\|^2 +  \frac{M+1}{2}\|\vect{u}^{(\zeta)} - \vect{u}^{*}\|^2, \Bigg| \text{\textit{By Young's Inequality}}  \\
\end{aligned} \end{equation} 
where the second inequality is achieved through Young's inequality. Take conditional expectation that conditioned on $\vect{w}^{(j)}$ and $\vect{u}^{(\zeta)}$  to obtain
\begin{equation} \begin{aligned}
\mathcal{H}_{\vect{k}}(\vect{u}^{*}, \vect{w}^{*}) &\leq \mathbb{E} \left[  \mathcal{H}_{\vect{k}}(\vect{u}^{(\zeta)}, \vect{w}^{(j)})  | \vect{w}^{(j)}, \vect{u}^{(\zeta)} \right]  + \mathbb{E} \left[ \frac{1}{2}\| \vect{g}^{(j)}_{\vect{w}}\|^2 | \vect{w}^{(j)}, \vect{u}^{(\zeta)}  \right] 
  \\&+  \mathbb{E} \left[ \frac{L+1}{2}\|\vect{w}^{(j)} - \vect{w}^{*}\|^2| \vect{w}^{(j)} \right] + \mathbb{E} \left[ \frac{1}{2}\| \vect{g}^{(\zeta)}_{\vect{u}}\|^2| \vect{w}^{(j)}, \vect{u}^{(\zeta)}  \right] 
  \\ & + \mathbb{E} \left[ \frac{M+1}{2}\|\vect{u}^{(\zeta)} - \vect{u}^{*}\|^2|  \vect{u}^{(\zeta)}  \right] , 
  \\ \\ 
  &\leq \mathbb{E} \left[  \mathcal{H}_{\vect{k}}(\vect{u}^{(\zeta)}, \vect{w}^{(j)})  | \vect{w}^{(j)}, \vect{u}^{(\zeta)} \right] \\& +  \frac{1}{2}\| \mathbb{E} \left[ \vect{g}^{(j)}_{\vect{w}}| \vect{w}^{(j)}, \vect{u}^{(\zeta)}  \right] \|^2 \\&  +  \frac{1}{2}\| \mathbb{E} \left[ \vect{g}^{(\zeta)}_{\vect{u}}| \vect{w}^{(j)}, \vect{u}^{(\zeta)}  \right] \|^2  \\ &  + \mathbb{E} \left[ \frac{L+1}{2}\|\vect{w}^{(j)} - \vect{w}^{*}\|^2| \vect{w}^{(j)} \right] 
   \\&  + \mathbb{E} \left[ \frac{M+1}{2}\|\vect{u}^{(\zeta)} - \vect{u}^{*}\|^2|  \vect{u}^{(\zeta)}  \right] \\ & + Var \left(\vect{g}^{(j)}_{\vect{w}}| \vect{w}^{(j)}, \vect{u}^{(\zeta)}  \right) \\&  + Var \left( \vect{g}^{(\zeta)}_{\vect{u}}| \vect{w}^{(j)}, \vect{u}^{(\zeta)}  \right) \quad \Bigg| \text{ Since \textit{$Var(x) = E[x^2] - E[x]^2$ }} 
\end{aligned} \end{equation} 
Integrate out the $\vect{w}^{(j)}$ and $\vect{u}^{(\zeta)}$ to get by law of total expectation,
\begin{equation} \begin{aligned}
\mathbb{E} \left[  \mathcal{H}_{\vect{k}}(\vect{u}^{*}, \vect{w}^{*})\right] - \mathbb{E} \left[  \mathcal{H}_{\vect{k}}(\vect{u}^{(\zeta)}, \vect{w}^{(j)}) \right]
  &\leq  \frac{1}{2}\| \mathbb{E} \left[ \vect{g}^{(j)}_{\vect{w}}  \right] \|^2  \\&  +  \frac{1}{2}\| \mathbb{E} \left[ \vect{g}^{(\zeta)}_{\vect{u}}\right] \|^2  \\ &  + \mathbb{E} \left[ \frac{L+1}{2}\|\vect{w}^{(j)} - \vect{w}^{*}\|^2\right]  \\&  + \mathbb{E} \left[ \frac{M+1}{2}\|\vect{u}^{(\zeta)} - \vect{u}^{*}\|^2 \right]  
  \\ & + Var \left(\vect{g}^{(j)}_{\vect{w}} \right) + Var \left( \vect{g}^{(\zeta)}_{\vect{u}} \right)
\end{aligned} \end{equation}
Since, $\mathbb{E} \left[  \mathcal{H}_{\vect{k}}(\vect{u}^{(\zeta)}, \vect{w}^{(j)}) \right] =  \underset{\begin{array}{c} \vect{u}^{(i)} \in \mathcal{U} \\ \| \vect{u}^{(i)} - \vect{u}^{*}\| \leq \delta_u \end{array}}{max}  \mathbb{E} \left[  \mathcal{H}_{\vect{k}}(\vect{u}^{(i)}, \vect{w}^{(j)}) \right].$ We obtain by substituting the bounds.
\begin{equation} \begin{aligned}
  &  \mathbb{E} \left[  \mathcal{H}_{\vect{k}}(\vect{u}^{*}, \vect{w}^{*})\right] - \underset{\begin{array}{c} \vect{u}^{(i)} \in \mathcal{U} \\ \| \vect{u}^{(i)} - \vect{u}^{*}\| \leq \delta_u \end{array}}{max} \mathbb{E} \left[  \mathcal{H}_{\vect{k}}(\vect{u}^{(i)}, \vect{w}^{(j)}) \right]\\ & \leq (\frac{L+1}{2})\delta_w^2 +  G^2 \left( \frac{\left[ (1+3\beta)^2\right]}{2} +  \left[\frac{bN^2+b^3}{N^3} \right] \left[ (1+3\beta)^2\right] \right) 
  \\ & \frac{M+1}{2}\delta_u^2 + \bar{G}^2 \left( \frac{1}{2}\left( \frac{b}{N}\right)^2 + \frac{2b\beta^2\left(N^2+b^2\right)}{N^3} \right)\\ \\ 
\end{aligned} \end{equation} 
which leads into
\begin{equation} \begin{aligned}
  & \norm{ \mathbb{E} \left[   \mathcal{H}_{\vect{k}}(\vect{u}^{*}, \vect{w}^{*}) - \underset{\begin{array}{c} \vect{u}^{(i)} \in \mathcal{U} \\ \| \vect{u}^{(i)} - \vect{u}^{*}\| \leq \delta_u \end{array}}{max}  \mathcal{H}_{\vect{k}}(\vect{u}^{(i)}, \vect{w}^{(j)})  \right] }\leq \epsilon(\delta_w, \delta_u)
\end{aligned} \end{equation} 
with \begin{align}
    \epsilon(\delta_u, \delta_w) &= (\frac{L+1}{2})\delta_w^2 +  (1+3\beta)^2 G^2 \left( \frac{1}{2} +  \frac{b}{N} + \frac{b^3}{N^3}  \right)  \nonumber \\ &+ \frac{M+1}{2}\delta_u^2 + \bar{G}^2 \left( \frac{2b\beta^2}{N}  + \frac{1 b^2}{2N^2}+ + \frac{2b^3\beta^2}{N^3} \right).
\end{align} 
We have pulled the expected value outside the maximum under the assumption that the function is smooth.
\end{proof}
By the two inequalities proved in the preceeding lemmas, we obtain for  $(\vect{u}^{*}, \vect{w}^{*}) \in \mathcal{M} \cup \mathcal{N}.$ 
\begin{equation} \begin{aligned}
   \norm{ \mathbb{E}  \left[   \mathcal{H}_{\vect{k}}(\vect{u}^{*}, \vect{w}^{*}) - \underset{\begin{array}{c} \vect{u}^{(i)} \in \mathcal{U} \\ \| \vect{u}^{(i)} - \vect{u}^{*}\| \leq \delta_u \end{array}}{max}  \mathcal{H}_{\vect{k}}(\vect{u}^{(i)}, \vect{w}^{(j)})  \right] } & \leq \epsilon(\delta_w, \delta_u) \\ \\ 
    \norm{ \mathbb{E} \left[ \mathcal{H}_{\vect{k}}(\vect{u}^{*}, \vect{w}^{*}) - \mathcal{H}_{\vect{k}}(\vect{u}^{(i)}, \vect{w}^{*}) \right] } & \leq  \epsilon(\delta_u), 
\end{aligned} \end{equation} 
Then, $(\vect{u}^{*}, \vect{w}^{*})$ is a local minimax point according to definition \ref{def:stack_eq_SM}. Thus, we conclude the proof of Theorem 1.
\end{proof}

\newpage
\section{Proof of Theorem 2, Convergence to the equilibrium point}
The proof of this theorem requires us to first prove that the maximizing player converges. Next, we show that, provided the maximizing player provides a strategy, the minimizing player converges. 
In our algorithm, we perform $\zeta$ gradient ascent updates for each $\vect{w}^{(j)}.$ Therefore, we will first show that the with many gradient ascent steps, our gradient goes to zero. 
\begin{theorem}[Gradient of $\mathcal{H}_{\vect{k}}$ with respect to $\vect{u}^{(i)}$ converges to zero] \label{theorem: grad_max_SM}
For each task $k,$  fix $\vect{w}^{(j)} \in\mathcal{W}$ and construct $\mathcal{M}= \{\vect{w}^{(j)}, \vect{\mathcal{U}} \}.$ Let assumption 2 be true and sample uniformly a minibatch of size $b$ from the dataset $D$ of size $N$. Define 
$ \mathcal{H}_{\vect{k}}( \vect{u}^{(i)}, \vect{w}^{(j)} ) =  J_{\vect{k}}(\vect{w}^{(j)} )+ \beta_{1} J_{\vect{k}}(\vect{u}_{0}^{(i)}, \vect{w}^{(j)} ) + \beta_{2} J_{\vect{k}}(\vect{u}_{1}^{(i)} , \vect{w}^{(j)} ) + \beta_3 J_{\vect{k}}(\vect{u}_{2}^{(i)}, \vect{w}^{(j)}) $ with $\beta_1, \beta_2, \beta_3 \leq \beta > 0.$ Let the inequalities from Lemma \ref{lem:bound_grad} be true. Choose $\alpha^{(i)}_{u} = \frac{\alpha_{u}}{\sqrt{\zeta}},$ then $\sum_{i} \alpha^{(i)}_{u} = \sum_{i} \frac{\alpha_{u}}{ \sqrt{\zeta}} = \alpha_{u} \sqrt{\zeta}.$ Similarly, $ \sum_{i} (\alpha^{(i)}_{u})^2 =  \alpha^2_{u}$ such that
$\sum_{i} (\alpha^{(i)}_{u}-\frac{M (\alpha^{(i)}_{u})^2 }{2}) = \frac{ 2\alpha_{u} \sqrt{\zeta} - M\alpha^2_{u} }{2} = S_n$. Now, denote $\Delta_{(i)} = \mathcal{H}_{\vect{k}}(\vect{u}^{(i+1)}, \vect{w}^{(j)})-\mathcal{H}_{\vect{k}}(\vect{u}^{*}, \vect{w}^{(j)}),$ then the following is true  
\[ \underset{i=0,\cdots, \zeta}{min} E\left[ \| \vect{g}^{(i)}_{\vect{u}} \|^2 \right] \leq  \frac{2 \mathbb{E}[\| \vect{u}^{(\zeta)} - \vect{u}^{*}\|] }{2\alpha_{u} \sqrt{\zeta} - M\alpha^2_{u} } + \frac{2 M(\alpha_{u})^2 b\beta^2\bar{G}^2 )}{N\left(2\alpha_{u} \sqrt{\zeta} - M\alpha^2_{u}\right) }+\frac{2M(\alpha_{u})^2(b^3\beta^2 \bar{G}^2 )}{N^3\left( 2\alpha_{u} \sqrt{\zeta} - M\alpha^2_{u} \right) }. .\] Provided $M\alpha^2_{u} << \alpha_{u} \sqrt{\zeta}$, we have $\lim_{\zeta \rightarrow \infty } \underset{i=0,\cdots, \zeta}{min} E\left[ \| \vect{g}^{(i)}_{\vect{u}} \|^2 \right]   \rightarrow 0$ with the rate $\frac{1}{\sqrt{\zeta}}.$
\end{theorem}
\begin{proof}
For each task $k,$  fix $\vect{w}^{(j)} \in\mathcal{W}$ and construct $\mathcal{M}= \{\vect{w}^{(j)}, \vect{\mathcal{U}} \}.$ Therefore, for $ (\vect{u}^{ (i+1 )}, \vect{w}^{(j)}) , ( \vect{u}^{(i)}, \vect{w}^{(j)}) \in \mathcal{M}$  assuming L-smoothness of $\mathcal{H}_{\vect{k}}( \vect{u}^{(i)}, \vect{w}^{(j)} ) $ we write
\begin{equation} \begin{aligned}
  \mathcal{H}_{\vect{k}}(\vect{u}^{(i+1)}, \vect{w}^{(j)})  \geq \mathcal{H}_{\vect{k}}(\vect{u}^{(i)}, \vect{w}^{(j)}) + \langle \vect{g}^{(i)}_{\vect{u}}, \vect{u}^{ (i+1 )}-\vect{u}^{ (i)} \rangle -   \frac{M}{2}\| \vect{u}^{(i+1)} - \vect{u}^{(i)}\|^2, 
\end{aligned} \end{equation} 
where we get upon substitution of the update rule $\alpha^{(i)}_{u} \vect{\hat{g}}^{(i)}_{\vect{u}}$
\begin{equation} \begin{aligned} 
  \mathcal{H}_{\vect{k}}(\vect{u}^{(i+1)}, \vect{w}^{(j)})  & \geq  \mathcal{H}_{\vect{k}}(\vect{u}^{(i)}, \vect{w}^{(j)}) + \alpha^{(i)}_{u} \langle \vect{g}^{(i)}_{\vect{u}}, \vect{\hat{g}}^{(i)}_{\vect{u}} \rangle -   \frac{M ( \alpha^{(i)}_{u} )^2}{2}\|\vect{\hat{g}}^{(i)}_{\vect{u}}\|^2.
\end{aligned} \end{equation} 
Take conditional expectation that is conditioned on $\vect{u}^{(i)}$
  \begin{equation} \begin{aligned}
  \mathbb{E}[ \mathcal{H}_{\vect{k}}(\vect{u}^{(i+1)}, \vect{w}^{(j)}) | \vect{u}^{(i)}]  & \geq  \mathbb{E}[  \mathcal{H}_{\vect{k}}(\vect{u}^{(i)}, \vect{w}^{(j)}) | \vect{u}^{(i)}]  \\ &  + \alpha^{(i)}_{u} \langle \vect{g}^{(i)}_{\vect{u}}, \mathbb{E}[  \vect{\hat{g}}^{(i)}_{\vect{u}} | \vect{u}^{(i)}]  \rangle     \\ & -   \frac{M ( \alpha^{(i)}_{u} )^2}{2} \mathbb{E}[  \|\vect{\hat{g}}^{(i)}_{\vect{u}}\|^2| \vect{u}^{(i)}],     \\   &\Bigg| \text{ Since \textit{$Var(x) = E[x^2] - E[x]^2$ }}    \\
  & \geq \mathbb{E}[  \mathcal{H}_{\vect{k}}(\vect{u}^{(i)}, \vect{w}^{(j)}) | \vect{u}^{(i)}]   \\ &  + \alpha^{(i)}_{u}  \langle \vect{g}^{(i)}_{\vect{u}}, \mathbb{E}[  \vect{\hat{g}}^{(i)}_{\vect{u}} | \vect{u}^{(i)}]  \rangle     \\ &  -   \frac{M ( \alpha^{(i)}_{u} )^2}{2} ( \|\mathbb{E}[  \vect{\hat{g}}^{(i)}_{\vect{u}}| \vect{u}^{(i)}] \|^2 + Var(\vect{g}^{(i)}_{\vect{u}}) )    
  \end{aligned} \end{equation} 
  
 Assume $\mathbb{E}[  \vect{\hat{g}}^{(i)}_{\vect{u}} | \vect{u}^{(i)}]  = \vect{g}^{(i)}_{\vect{u}}$ \cite{https://doi.org/10.5281/zenodo.4638695}  to write 
    \begin{equation} \begin{aligned}
 \mathbb{E}[ \mathcal{H}_{\vect{k}}(\vect{u}^{(i+1)}, \vect{w}^{(j)}) | \vect{u}^{(i)}]  & \geq  \mathbb{E}[  \mathcal{H}_{\vect{k}}(\vect{u}^{(i)}, \vect{w}^{(j)}) | \vect{u}^{(i)}]  + \alpha^{(i)}_{u} \| \vect{g}^{(i)}_{\vect{u}} \|^2     \\ &  -  \frac{M ( \alpha^{(i)}_{u} )^2}{2} \| \vect{g}^{(i)}_{\vect{u}} \|^2  \\ & - \frac{M ( \alpha^{(i)}_{u} )^2}{2}( Var(\vect{g}^{(i)}_{\vect{u}}))   \\     \\ 
  & \geq  \mathbb{E}[  \mathcal{H}_{\vect{k}}(\vect{u}^{(i)}, \vect{w}^{(j)}) | \vect{u}^{(i)}]  \\ & + \alpha^{(i)}_{u}(1-\frac{M \alpha^{(i)}_{u} }{2}) \| \vect{g}^{(i)}_{\vect{u}} \|^2     \\ &  - \frac{M ( \alpha^{(i)}_{u} )^2}{2}( Var(\vect{g}^{(i)}_{\vect{u}})).
\end{aligned} \end{equation} 
Take expectation over all possible values of $\vect{u}^{(i)}$ and applying Law of total expectation provides and substitute the bounds from Lemma \ref{lem:bound_grad} to obtain
  \begin{equation} \begin{aligned}
  \mathbb{E}[ \mathcal{H}_{\vect{k}}(\vect{u}^{(i+1)}, \vect{w}^{(j)})] 
  & \geq  \mathbb{E}[  \mathcal{H}_{\vect{k}}(\vect{u}^{(i)}, \vect{w}^{(j)}) ]   \\ & + \alpha^{(i)}_{u}(1-\frac{M \alpha^{(i)}_{u} }{2}) \| \vect{g}^{(i)}_{\vect{u}} \|^2     \\ &  - \frac{M ( \alpha^{(i)}_{u} )^2}{2}( Var(\vect{g}^{(i)}_{\vect{u}})),
  \\ \\ & \geq  \mathbb{E}[  \mathcal{H}_{\vect{k}}(\vect{u}^{(i)}, \vect{w}^{(j)}) ]  \\ &  + \alpha^{(i)}_{u}(1-\frac{M \alpha^{(i)}_{u} }{2}) E \left[ \| \vect{g}^{(i)}_{\vect{u}} \|^2 \right]    \\ &  - \frac{M ( \alpha^{(i)}_{u} )^2}{2}(Var(\vect{g}^{(i)}_{\vect{u}})).
\end{aligned} \end{equation} 
 Add and subtract $\mathcal{H}_{\vect{k}}(\vect{u}^{*}, \vect{w}^{(j)})$ from both sides to obtain
\begin{equation} \begin{aligned}
  \mathbb{E}[ \mathcal{H}_{\vect{k}}(\vect{u}^{(i+1)}, \vect{w}^{(j)}) - \mathcal{H}_{\vect{k}}(\vect{u}^{*}, \vect{w}^{(j)}) ]  & \geq  \mathbb{E}[  \mathcal{H}_{\vect{k}}(\vect{u}^{(i)}, \vect{w}^{(j)})  - \mathcal{H}_{\vect{k}}(\vect{u}^{*}, \vect{w}^{(j)}) ]  \\ &
  + \alpha^{(i)}_{u}(1+\frac{M \alpha^{(i)}_{u} }{2}) E \left[ \| \vect{g}^{(i)}_{\vect{u}} \|^2 \right]   \\ &  + \frac{M ( \alpha^{(i)}_{u} )^2}{2}(Var(\vect{g}^{(i)}_{\vect{u}}) ), \\ & \Bigg| \text{ Denote $\Delta_{(i)} = \mathcal{H}_{\vect{k}}(\vect{u}^{(i+1)}, \vect{w}^{(j)})-\mathcal{H}_{\vect{k}}(\vect{u}^{*}, \vect{w}^{(j)}) $} \\ 
   \mathbb{E}[ \Delta^{i+1} ]   & \geq  \mathbb{E}[ \Delta^{i} ]\\ & + \alpha^{(i)}_{u}(1-\frac{M \alpha^{(i)}_{u} }{2}) E \left[ \| \vect{g}^{(i)}_{\vect{u}} \|^2 \right]  \\ & - \frac{M ( \alpha^{(i)}_{u} )^2}{2}(Var(\vect{g}^{(i)}_{\vect{u}})). 
\end{aligned} \end{equation} 
Thus, we obtain 
\begin{equation} \begin{aligned}
 - \alpha^{(i)}_{u}(1-\frac{M \alpha^{(i)}_{u} }{2}) E \left[ \| \vect{g}^{(i)}_{\vect{u}} \|^2 \right]  & \geq  \mathbb{E}[ \Delta^{i} ]  - \mathbb{E}[ \Delta^{i+1} ]   - \frac{M ( \alpha^{(i)}_{u} )^2}{2}(Var(\vect{g}^{(i)}_{\vect{u}})).  \\ \\ 
 \alpha^{(i)}_{u}(1-\frac{M \alpha^{(i)}_{u} }{2}) E \left[ \| \vect{g}^{(i)}_{\vect{u}} \|^2 \right]  & \leq  \mathbb{E}[ \Delta^{i+1} ] - \mathbb{E}[ \Delta^{i} ] + \frac{M ( \alpha^{(i)}_{u} )^2}{2}(Var(\vect{g}^{(i)}_{\vect{u}})).  \\ \\  
\end{aligned} \end{equation} 
Since, we take $\zeta$ updates, let us sum both sides for $\zeta$ updates
\begin{equation} \begin{aligned}
 \sum_{i} \alpha^{(i)}_{u}(1-\frac{M \alpha^{(i)}_{u} }{2}) E \left[ \| \vect{g}^{(i)}_{\vect{u}} \|^2 \right]  & \leq  \underset{\text{Telescopic sum}}{\underbrace{\sum_{i} \left( \mathbb{E}[ \Delta^{i+1} ] -\mathbb{E}[ \Delta^{i} ] \right)}} \\ & + \sum_{i} \frac{M ( \alpha^{(i)}_{u} )^2}{2}(\frac{2b\beta^2\left(N^2+b^2\right)}{N^3} \bar{G}^2 ),  \\ \\ 
   \sum_{i} \alpha^{(i)}_{u}(1-\frac{M \alpha^{(i)}_{u} }{2}) E \left[ \| \vect{g}^{(i)}_{\vect{u}} \|^2 \right]  & \leq    \mathbb{E}[ \Delta^{\zeta} ] -\mathbb{E}[\Delta^{0}]  \\ & + \sum_{i} \frac{M  ( \alpha^{(i)}_{u} )^2}{2}(\frac{2b\beta^2\left(N^2+b^2\right)}{N^3} \bar{G}^2 ).
\end{aligned} \end{equation} 
Now choose, $\alpha^{(i)}_{u} = \frac{\alpha_{u}}{\sqrt{\zeta}},$ then $\sum_{i} \alpha^{(i)}_{u} = \sum_{i} \frac{\alpha_{u}}{ \sqrt{\zeta}} = \alpha_{u} \sqrt{\zeta}.$ Similarly, $ \sum_{i} (\alpha^{(i)}_{u})^2 =  \alpha^2_{u}.$ Therefore, we obtain 
\[ \sum_{i} (\alpha^{(i)}_{u}-\frac{M (\alpha^{(i)}_{u})^2 }{2}) = \frac{ 2\alpha_{u} \sqrt{\zeta} - M\alpha^2_{u} }{2}  = S_n.\]
Use the idea that for a set of random variables $x_1, \cdots, x_m,$  we have $\underset{i=0,\cdots,m}{min} \mathbb{E}[x] \leq \mathbb{E}[\underset{i=0,\cdots, m}{min} x] $ and the fact that the minimum value of $\mathbb{E}[\Delta^{0}]$ is zero.
\begin{equation} \begin{aligned}
 \sum_{i} (\alpha^{(i)}_{u}-\frac{M (\alpha^{(i)}_{u})^2 }{2}) \underset{i=0,\cdots, \zeta}{min}E\left[ \| \vect{g}^{(i)}_{\vect{u}} \|^2 \right]  & \leq   \mathbb{E}[ \Delta^{K}] -\mathbb{E}[\Delta^{0}]  \\ & + \sum_{i} \frac{M  ( \alpha^{(i)}_{u} )^2}{2}(\frac{2b\beta^2\left(N^2+b^2\right)}{N^3} \bar{G}^2 ). \\ \\ 
 \underset{i=0,\cdots, \zeta}{min} E\left[ \| \vect{g}^{(i)}_{\vect{u}} \|^2 \right]    & \leq  \frac{ \mathbb{E}[ \Delta^{K} ] }{S_n}+ \frac{M(\alpha_{u})^2}{2 S_n}(\frac{2b\beta^2\left(N^2+b^2\right)}{N^3} \bar{G}^2 ). \\ \\ 
\end{aligned} \end{equation} 
Thus providing the upperbound on the minimum value of the gradients over all the update steps as 
\begin{equation} \begin{aligned}
 \underset{i=0,\cdots, \zeta}{min} E\left[ \| \vect{g}^{(i)}_{\vect{u}} \|^2 \right] & \leq  \frac{\mathbb{E}[ \Delta^{K} ]}{\left(  \frac{ 2\alpha_{u} \sqrt{\zeta} - M\alpha^2_{u} }{2}  \right) } + \frac{2 M(\alpha_{u})^2 b\beta^2\bar{G}^2 )}{2N\left(\frac{ 2\alpha_{u} \sqrt{\zeta} - M\alpha^2_{u} }{2}  \right) }+\frac{2M(\alpha_{u})^2(b^3\beta^2 \bar{G}^2 )}{2N^3\left( \frac{ 2\alpha_{u} \sqrt{\zeta} - M\alpha^2_{u} }{2} \right) }. \\ \\ 
\end{aligned} \end{equation} 
 As a consequence of lemma \ref{lem:bound_grad}, we may write 
\begin{equation} \begin{aligned}
 \underset{i=0,\cdots, \zeta}{min} E\left[ \| \vect{g}^{(i)}_{\vect{u}} \|^2 \right] & \leq  \frac{2 \mathbb{E}[\| \vect{u}^{(\zeta)} - \vect{u}^{*}\|] }{2\alpha_{u} \sqrt{\zeta} - M\alpha^2_{u} } + \frac{2 M(\alpha_{u})^2 b\beta^2\bar{G}^2 )}{N\left(2\alpha_{u} \sqrt{\zeta} - M\alpha^2_{u}\right) }+\frac{2M(\alpha_{u})^2(b^3\beta^2 \bar{G}^2 )}{N^3\left( 2\alpha_{u} \sqrt{\zeta} - M\alpha^2_{u} \right) }. \\ \\ 
\end{aligned} \end{equation} 
Under the assumption that $\| \vect{u}^{(\zeta)} - \vect{u}^{*}\| \leq \delta_u$ we obtain our result as
\begin{equation} \begin{aligned}
 \underset{i=0,\cdots, \zeta}{min} E\left[ \| \vect{g}^{(i)}_{\vect{u}} \|^2 \right] & \leq  \frac{2 \mathbb{E}[ \delta_u] }{2\alpha_{u} \sqrt{\zeta} - M\alpha^2_{u} } + \frac{4 M(\alpha_{u})^2 b\beta^2\bar{G}^2 )}{N\left(2\alpha_{u} \sqrt{\zeta} - M\alpha^2_{u}\right) } \\ &+\frac{4M(\alpha_{u})^2(b^3\beta^2 \bar{G}^2 )}{N^3\left( 2\alpha_{u} \sqrt{\zeta} - M\alpha^2_{u} \right) }.
\end{aligned} \end{equation} 
Provided $\alpha^2_{u} << \alpha_{u} \sqrt{\zeta}$,  $\underset{i=0,\cdots, \zeta}{min} E\left[ \| \vect{g}^{(i)}_{\vect{u}} \|^2 \right] $ we have  $\lim_{\zeta \rightarrow \infty }\mathbb{E}[\| \vect{g}^{(i)}_{\vect{u}} \|^2]  \rightarrow 0.$ with the rate $\frac{1}{ \sqrt{\zeta}}$
\end{proof}
With the result that the gradient will converge to zero for the maximizing player, we are now ready to show the convergence of our algoritheorem in the sense of gradients. Thus, providing the result to our main theorem in the paper--Theorem 2. For the following result, we will assume that, for each $j$ the maximizing player has already played and provides with a $\vect{u}^{(\zeta)}.$
\begin{theorem*}[Convergence in gradients]
For each task $k,$  construct $\mathcal{N}= \{\vect{\mathcal{U}}, \vect{\mathcal{W}} \}.$  Let assumption 2 be true and define a dataset $D$ of size $N>0.$ Assume that a minibatch of size $b$ is obtained by uniformly sampling from $D$ and define 
$ \mathcal{H}_{\vect{k}}( \vect{u}^{(i)}, \vect{w}^{(j)} ) =  J_{\vect{k}}(\vect{w}^{(j)} )+ \beta_{1} J_{\vect{k}}(\vect{u}_{0}^{(i)}, \vect{w}^{(j)} ) + \beta_{2} J_{\vect{k}}(\vect{u}_{1}^{(i)} , \vect{w}^{(j)} ) + \beta_3 J_{\vect{k}}(\vect{u}_{2}^{(i)}, \vect{w}^{(j)}) $ with $\beta_1, \beta_2, \beta_3 \leq \beta > 0$ Consider the updates for $\vect{u}^{(i)}$ as $\alpha^{(i)}_{u} \vect{\hat{g}}^{(i)}_{\vect{u}}$ and the updates for updates for $\vect{w}^{(j)}$ as $\alpha^{(j)}_{w} \vect{\hat{g}}^{(i)}_{\vect{w}}$ and let the inequalities from Lemma \ref{lem:bound_grad} provides the bounds on the variance and expected values of these gradients. By theorem \ref{theorem: grad_max_SM}, we obtain that the maximizing player $\vect{u}^{(i)}$ converges to $\vect{u}^{(\zeta)}$ Furthermore, assume that $\alpha^{(i)}_{u}>0,  b>0, \beta>0, N>0, \norm{\vect{u}^{(\zeta)}-\vect{u}^{*}}^2 \leq \delta_w^2,$ and that $\underset{j = 1, 2, 3,   \cdots \rho}{max} Var(\vect{g}^{(j)}_{\vect{w}} )$ and $\underset{j = 1, 2, 3,   \cdots \rho}{max} Var(\vect{g}^{(j)}_{\vect{w}} )$ are upper bounded by the bound provided by lemma \ref{lem:bound_grad}. Furthermore, choose, $\alpha^{(i)}_{w} = \frac{\alpha_{w}}{ \sqrt{\rho}},$ then $\sum_{i} (\alpha^{(i)}_{w} = \sum_{i} \frac{\alpha_{w}}{\sqrt{\rho}} = \alpha_{w} \sqrt{\rho}.$ Similarly, $ \sum_{i} (\alpha^{(i)}_{w})^2 =  \alpha^2_{w}.$ Then the minimum value of the gradient is bounded as 
\begin{equation}  \begin{aligned}
   \underset{j = 1, 2, 3,  \cdots \rho}{min}  \mathbb{E}\left[  \| \vect{g}^{(j)}_{\vect{w}} \|^2\right]  &\leq  \frac{ 2\rho \mathbb{E} \left[\delta_u^2\right] (M+1) + 2(1+3\beta)G \delta_{w}2}{2\alpha_{w} \sqrt{\rho} -L_w\alpha^2_{w} } + \frac{L_w( \alpha_{\vect{w}})^2 b}{ N (2 \alpha_{w} \sqrt{\rho} -L_w\alpha^2_{w} )} \\ & +  \frac{ \beta  \rho b^2\bar{G}^2}{N^2 (2 \alpha_{w} \sqrt{\rho} -L_w\alpha^2_{w}) } + \frac{G^2 (1+3\beta)^2L_w(\alpha_{\vect{w}})^2 b^3}{(N^3 2\alpha_{w} \sqrt{\rho} - L_w\alpha^2_{w} ) }. 
\end{aligned} \end{equation}
where $G$ and $\bar{G}$ are provided by lemma \ref{lem:bound_grad} with 
and the gradient converges asymptotically to zero with the rate $\frac{1}{\sqrt{\rho}}$ under the assumption that $2\alpha_{w} \sqrt{\rho} >>L_w\alpha^2_{w}$ . 
\end{theorem*}
\begin{proof}
For each task $k,$  construct $\mathcal{N}= \{\mathcal{U}, \mathcal{W}\}.$ Therefore, for $ (\vect{u}^{(\zeta)}, \vect{w}^{(j+1)}) ,$ $( \vect{u}^{(\zeta)}, \vect{w}^{(j)}) \in \mathcal{N}$  assuming L-smoothness of $\mathcal{H}_{\vect{k}}( \vect{u}^{(\zeta)}, \vect{w}^{(j)} ) $ we write
\begin{equation} \begin{aligned}
  \mathcal{H}_{\vect{k}}(\vect{u}^{(\zeta)}, \vect{w}^{(j+1)}) &\leq \mathcal{H}_{\vect{k}}(\vect{u}^{*}, \vect{w}^{(j)}) \\ & + \langle \vect{g}^{(j)}_{\vect{w}}, \vect{w}^{(j+1)}-\vect{w}^{(j)} \rangle + \frac{L}{2}\|\vect{w}^{(j+1)} - \vect{w}^{(j)}\|^2
    \\ & + \langle \vect{g}^{(i)}_{\vect{u}}, \vect{u}^{(\zeta)}-\vect{u}^{*} \rangle +   \frac{M}{2}\|\vect{u}^{(\zeta)}-\vect{u}^{*}\|^2, 
\end{aligned} \end{equation} 
where we get upon substitution of the update rule $-\alpha^{(j)}_{w} \vect{\hat{g}}^{(j)}_{\vect{w}}$
\begin{equation} \begin{aligned} 
  \mathcal{H}_{\vect{k}}(\vect{u}^{(\zeta)}, \vect{w}^{(j+1)})  & \leq \mathcal{H}_{\vect{k}}(\vect{u}^{*}, \vect{w}^{(j)})  \\ & -\alpha^{(j)}_{w} \langle \vect{g}^{(j)}_{\vect{w}}, \vect{\hat{g}}^{(j)}_{\vect{w}}\rangle +   \frac{L_w( \alpha^{(j)}_{w} )^2}{2}\|\vect{\hat{g}}^{(j)}_{\vect{w}}\|^2 \\ & +\langle \vect{g}^{(\zeta)}_{\vect{u}}, \vect{u}^{(\zeta)}-\vect{u}^{*} \rangle  +\frac{M}{2}\|\vect{u}^{(\zeta)}-\vect{u}^{*}\|^2, 
\end{aligned} \end{equation} 
Take expectation that is conditioned on $\vect{w}^{(j)}, \vect{u}^{(\zeta)}$
  \begin{equation} \begin{aligned}
   \mathbb{E} \left[  \left( \mathcal{H}_{\vect{k}}(\vect{u}^{(\zeta)}, \vect{w}^{(j+1)}) - \mathcal{H}_{\vect{k}}(\vect{u}^{*}, \vect{w}^{(j)}) \right) | \vect{w}^{(j)}, \vect{u}^{(\zeta)} \right]     
   & \leq    - \alpha^{(j)}_{\vect{w}} \mathbb{E} \left[ \langle \vect{g}^{(j)}_{\vect{w}},  \vect{\hat{g}}^{(j)}_{\vect{w}} \rangle  | \vect{w}^{(j)}, \vect{u}^{(\zeta)}  \right] \\ & + \frac{L_w( \alpha^{(j)}_{\vect{w}} )^2}{2} \mathbb{E} \left[  \|\vect{\hat{g}}^{(j)}_{\vect{w}}\|^2| \vect{u}^{(\zeta)} , \vect{w}^{(j)} \right]
  \\ & + \mathbb{E} \left[ \langle \vect{g}^{(i)}_{\vect{u}}, \vect{u}^{(\zeta)}-\vect{u}^{*} \rangle | \vect{w}^{(j)}, \vect{u}^{(\zeta)} \right] \\ & +  \mathbb{E} \left[ \frac{M}{2}\|\vect{u}^{(\zeta)}-\vect{u}^{*}\|^2 | \vect{w}^{(j)}, \vect{u}^{(\zeta)}  \right], 
\end{aligned} \end{equation} 
Under the assumption that $\vect{\hat{g}}^{(j)}_{\vect{w}}$ and $\vect{\hat{g}}^{(i)}_{\vect{u}}$ are  unbiased estimators of the respective gradients, apply the Young's inequality to obtain 
\begin{equation} \begin{aligned}
  \mathbb{E} \left[  \left( \mathcal{H}_{\vect{k}}(\vect{u}^{(\zeta)}, \vect{w}^{(j+1)}) - \mathcal{H}_{\vect{k}}(\vect{u}^{*}, \vect{w}^{(j)}) \right) | \vect{w}^{(j)}, \vect{u}^{(\zeta)} \right]     
   & \leq  -\left( \alpha^{(j)}_{\vect{w}}  - \frac{L_w( \alpha^{(j)}_{\vect{w}} )^2}{2} \right)   \\ & \mathbb{E}\left[  \| \vect{g}^{(j)}_{\vect{w}} \|^2 | \vect{w}^{(j)}, \vect{u}^{(\zeta)}\right] \\ & +  \frac{L_w( \alpha^{(j)}_{\vect{w}} )^2}{2} Var(\vect{g}^{(j)}_{\vect{w}}  | \vect{u}^{(\zeta)} ) 
  \\ & + \frac{1}{2}\mathbb{E} \left[ \| \vect{g}^{(i)}_{\vect{u}} \|^2  | \vect{u}^{(\zeta)} \right] \\ & + \left(\frac{M+1}{2} \right) \times \\ & \mathbb{E} \left[ \norm{\vect{u}^{(\zeta)}-\vect{u}^{*}}^2 | \vect{w}^{(j)}, \vect{u}^{(\zeta)} \right], 
\end{aligned} \end{equation}
Integrate out $\vect{u}^{(\zeta)}$ and $\vect{w}^{(j)}$ to obtain by law of total expectation
\begin{equation} \begin{aligned}
   \mathbb{E} \left[  \left( \mathcal{H}_{\vect{k}}(\vect{u}^{(\zeta)}, \vect{w}^{(j+1)}) - \mathcal{H}_{\vect{k}}(\vect{u}^{*}, \vect{w}^{(j)}) \right)  \right]  & \leq  -\left( \alpha^{(j)}_{\vect{w}}  - \frac{L_w( \alpha^{(j)}_{\vect{w}} )^2}{2} \right)  \mathbb{E}\left[  \| \vect{g}^{(j)}_{\vect{w}} \|^2 \right] \\ & +  \frac{L_w( \alpha^{(j)}_{\vect{w}} )^2}{2} Var(\vect{g}^{(j)}_{\vect{w}})  
  \\ & + \frac{1}{2}\mathbb{E} \left[ \| \vect{g}^{(i)}_{\vect{u}} \|^2  \right] \\ & + \left(\frac{M+1}{2} \right)\mathbb{E} \left[ \norm{\vect{u}^{(\zeta)}-\vect{u}^{*}}^2  \right], 
\end{aligned} \end{equation}
Rearrange to obtain
\begin{equation} \begin{aligned}
  \left( \alpha^{(j)}_{\vect{w}}  - \frac{L_w( \alpha^{(j)}_{\vect{w}} )^2}{2} \right)  \mathbb{E}\left[  \| \vect{g}^{(j)}_{\vect{w}} \|^2 \right] & \leq \mathbb{E} \left[  \left( \mathcal{H}_{\vect{k}}(\vect{u}^{*}, \vect{w}^{(j)}) -\mathcal{H}_{\vect{k}}(\vect{u}^{(\zeta)}, \vect{w}^{(j+1)}) \right)  \right]   \\ &+  \frac{L_w( \alpha^{(j)}_{\vect{w}} )^2}{2} Var(\vect{g}^{(j)}_{\vect{w}} )  \\ & + \frac{1}{2}\mathbb{E} \left[ \| \vect{g}^{(i)}_{\vect{u}} \|^2  \right] \\ &+ \left(\frac{M+1}{2} \right)\mathbb{E} \left[ \norm{\vect{u}^{(\zeta)}-\vect{u}^{*}}^2  \right], 
\end{aligned} \end{equation}
Since $\vect{u}^{*}$ is the maximizer, we have from Theorem 2 that $\mathcal{H}_{\vect{k}}(\vect{u}^{*}, \vect{w}^{(j)}) \geq  \mathcal{H}_{\vect{k}}(\vect{u}^{(\zeta)}, \vect{w}^{(j)}).$ We thus obtain by adding and subtracting $\mathcal{H}_{\vect{k}}(\vect{u}^{*}, \vect{w}^{*})$
\begin{equation} \begin{aligned}
  \left( \alpha^{(j)}_{\vect{w}}  - \frac{L_w( \alpha^{(j)}_{\vect{w}} )^2}{2} \right)  \mathbb{E}\left[  \| \vect{g}^{(j)}_{\vect{w}} \|^2 \right] & \leq \mathbb{E} \Bigg[  \left( \mathcal{H}_{\vect{k}}(\vect{u}^{*}, \vect{w}^{(j)}) -\mathcal{H}_{\vect{k}}(\vect{u}^{(\zeta)}, \vect{w}^{(j+1)}) \right)  \\ &+ \mathcal{H}_{\vect{k}}(\vect{u}^{*}, \vect{w}^{*})- \mathcal{H}_{\vect{k}}(\vect{u}^{*}, \vect{w}^{*}) \Bigg]   \\ &+  \frac{L_w( \alpha^{(j)}_{\vect{w}} )^2}{2} Var(\vect{g}^{(j)}_{\vect{w}} ) \\ &  + \frac{1}{2}\mathbb{E} \left[ \| \vect{g}^{(i)}_{\vect{u}} \|^2  \right] \\ &+ \left(\frac{M+1}{2} \right)\mathbb{E} \left[ \norm{\vect{u}^{(\zeta)}-\vect{u}^{*}}^2  \right], 
\end{aligned} \end{equation}
Denote $\mathcal{H}_{\vect{k}}(\vect{u}^{(\zeta)}, \vect{w}^{(j+1)}) - \mathcal{H}_{\vect{k}}(\vect{u}^{*}, \vect{w}^{*})$ as $\Delta^{j+1}$ to obtain
\begin{equation} \begin{aligned}
  \left( \alpha^{(j)}_{\vect{w}}  - \frac{L_w( \alpha^{(j)}_{\vect{w}} )^2}{2} \right)  \mathbb{E}\left[  \| \vect{g}^{(j)}_{\vect{w}} \|^2 \right] & \leq \mathbb{E} \Bigg[ \Delta^{j} - \Delta^{j+1} \Bigg]  \\ &  +  \frac{L_w( \alpha^{(j)}_{\vect{w}} )^2}{2} Var(\vect{g}^{(j)}_{\vect{w}} )  \\ &+ \frac{1}{2}\mathbb{E} \left[ \| \vect{g}^{(i)}_{\vect{u}} \|^2  \right] \\ & + \left(\frac{M+1}{2} \right)\mathbb{E} \left[ \norm{\vect{u}^{(\zeta)}-\vect{u}^{*}}^2  \right], 
\end{aligned} \end{equation}
Sum both sides from $j = 1, 2, 3,   \cdots \rho$ to write
\begin{equation} \begin{aligned}
 \underset{j = 1, 2, 3, \cdots \rho}{\sum} \left( \alpha^{(j)}_{\vect{w}}  - \frac{L_w( \alpha^{(j)}_{\vect{w}} )^2}{2} \right) & \mathbb{E}\left[  \| \vect{g}^{(j)}_{\vect{w}} \|^2 \right]  \leq  \underset{j = 1, 2, 3, \cdots \rho}{\sum} \Bigg( \mathbb{E} \Bigg[  \Delta^{j} - \Delta^{j+1} \Bigg]   \\ & +  \frac{L_w( \alpha^{(j)}_{\vect{w}} )^2}{2} Var(\vect{g}^{(j)}_{\vect{w}} ) \\ &+ \frac{1}{2}\mathbb{E} \left[ \| \vect{g}^{(i)}_{\vect{u}} \|^2  \right] \\ & + \left(\frac{M+1}{2} \right)\mathbb{E} \left[ \norm{\vect{u}^{(\zeta)}-\vect{u}^{*}}^2  \right], \Bigg)
\end{aligned} \end{equation}

Therefore, we may simplify since the first term on the right hand side is the telescopic sum to write 
\begin{equation} \begin{aligned}
 \underset{j = 1, 2, 3,   \cdots \rho}{\sum} &\left( \alpha^{(j)}_{\vect{w}}  - \frac{L_w( \alpha^{(j)}_{\vect{w}} )^2}{2} \right)  \mathbb{E}\left[  \| \vect{g}^{(j)}_{\vect{w}} \|^2 \right]  \leq  \mathbb{E} \Bigg[ \Delta^{0}-\Delta^{\rho} \Bigg]  \\ & +   \underset{j = 1, 2, 3,   \cdots \rho}{\sum}  \frac{L_w( \alpha^{(j)}_{\vect{w}} )^2}{2} Var(\vect{g}^{(j)}_{\vect{w}} )  \\ &+  \underset{j = 1, 2, 3, \cdots \rho}{\sum}  \frac{1}{2}\mathbb{E} \left[ \| \vect{g}^{(i)}_{\vect{u}} \|^2  \right] \\ & +  \underset{j = 1, 2, 3,  \textbf{} \cdots \rho}{\sum}  \left(\frac{M+1}{2} \right)\mathbb{E} \left[ \norm{\vect{u}^{(\zeta)}-\vect{u}^{*}}^2  \right], \\ \\ 
  \underset{j = 1, 2, 3,  \cdots \rho}{min} & \mathbb{E}\left[  \| \vect{g}^{(j)}_{\vect{w}} \|^2\right] \underset{j = 1, 2, 3, \cdots \rho}{\sum} \left( \alpha^{(j)}_{\vect{w}}  - \frac{L_w( \alpha^{(j)}_{\vect{w}} )^2}{2} \right)    \leq  \mathbb{E} \Bigg[ \Delta^{0}-\Delta^{\rho} \Bigg]  \\ & +   \underset{j = 1, 2, 3,  \cdots \rho}{max} Var(\vect{g}^{(j)}_{\vect{w}} ) \underset{j = 1, 2, 3,  \cdots \rho}{\sum}  \frac{L_w( \alpha^{(j)}_{\vect{w}} )^2}{2}   \\ &+  \underset{j = 1, 2, 3,  \cdots \rho}{max} \mathbb{E} \left[ \| \vect{g}^{(i)}_{\vect{u}} \|^2  \right]  \underset{j = 1, 2, 3, \cdots \rho}{\sum}  \frac{1}{2} \\ & +  \underset{j = 1, 2, 3,  \textbf{} \cdots \rho}{\sum}  \left(\frac{M+1}{2} \right)\mathbb{E} \left[ \norm{\vect{u}^{(\zeta)}-\vect{u}^{*}}^2  \right], 
\end{aligned} \end{equation}
Here, we may bring in some assumptions. First, $\norm{\vect{u}^{(\zeta)}-\vect{u}^{*}}^2 \leq \delta_w^2,$ then, we assume that the $\underset{j = 1, 2, 3,   \cdots \rho}{max} Var(\vect{g}^{(j)}_{\vect{w}} )$ and $\underset{j = 1, 2, 3,   \cdots \rho}{max} Var(\vect{g}^{(j)}_{\vect{w}} )$ is upper bounded by the bound on these quantities within in the compact set provided by Lemma \ref{lem:bound_grad}. Furthermore, choose, $\alpha^{(i)}_{w} = \frac{\alpha_{w}}{ \sqrt{\rho}},$ then $\sum_{i} (\alpha^{(i)}_{w} = \sum_{i} \frac{\alpha_{w}}{\sqrt{\rho}} = \alpha_{w} \sqrt{\rho}.$ Similarly, $ \sum_{i} (\alpha^{(i)}_{w})^2 =  \alpha^2_{w}.$ Therefore, we obtain 
\[ \sum_{i} (\alpha^{(i)}_{w}-\frac{L_w(\alpha^{(i)}_{w})^2 }{2}) = \frac{2 \alpha_{w} \sqrt{\rho} - L_w\alpha^2_{w}}{2} = S_n\].
Use the idea that for a set of random variables $x_1, \cdots, x_m,$  we have $\underset{i=0,\cdots,m}{min} \mathbb{E}[x] \leq \mathbb{E}[\underset{i=0,\cdots, m}{min} x] $ and the fact that the minimum value of $\mathbb{E}[\Delta^{\rho}]$ is zero, we obtain 
\begin{equation} \begin{aligned}
  \underset{j = 1, 2, 3,  \cdots \rho}{min}  \mathbb{E}\left[  \| \vect{g}^{(j)}_{\vect{w}} \|^2\right]  &\leq \frac{1}{ S_n} \mathbb{E} \Bigg[ \Delta^{0}\Bigg]  +   \left[\frac{bN^2+b^3}{ S_n N^3} \right] \left[ G^2 (1+3\beta)^2\right] \frac{L_w( \alpha_{\vect{w}})^2}{2}   \\ &+   \left( \beta  \frac{b\bar{G}}{N} \right)^2 \frac{\rho}{2 S_n} + \left(\rho \frac{M+1}{2 S_n} \right)\mathbb{E} \left[\delta_u^2\right], \\ \\
  &\leq \frac{\mathbb{E} [ \Delta^{0}] }{ S_n}  + \left[\frac{L_w( \alpha_{\vect{w}})^2 b}{ 2N S_n} \right] + \left[\frac{L_w( \alpha_{\vect{w}})^2 b^3}{ 2 S_n N^3} \right] \left[ G^2 (1+3\beta)^2\right]    \\ &+  \beta  \frac{\rho b^2\bar{G}^2}{2 S_n N^2} + \frac{ \rho \mathbb{E} \left[\delta_u^2\right] (M+1)}{2 S_n} , 
\end{aligned} \end{equation}
As a consequence of lemma \ref{lem:bound_grad} we may obtain $\Delta^{0} \leq (1+3\beta)G \delta_{w}.$
\begin{equation} \begin{aligned}
  \underset{j = 1, 2, 3,  \cdots \rho}{min}  \mathbb{E}\left[  \| \vect{g}^{(j)}_{\vect{w}} \|^2\right]  &\leq  \frac{ \rho \mathbb{E} \left[\delta_u^2\right] (M+1) + (1+3\beta)G \delta_{w}2}{ (\frac{2 \alpha_{w} \sqrt{\rho} - L_w\alpha^2_{w}}{2} ) } + \frac{L_w( \alpha_{\vect{w}})^2 b}{ 2N (\frac{2 \alpha_{w} \sqrt{\rho} - L_w\alpha^2_{w}}{2} )} \\ & +  \frac{ \beta  \rho b^2\bar{G}^2}{2  N^2 (\frac{2 \alpha_{w} \sqrt{\rho} - \alpha^2_{w}}{2} ) } + \frac{G^2 (1+3\beta)^2 L_w(\alpha_{\vect{w}})^2 b^3}{ 2 (\frac{2 \alpha_{w} \sqrt{\rho} -L_w\alpha^2_{w}}{2} ) N^3}. 
\end{aligned} \end{equation}
Simplify to obtain
\begin{equation} \begin{aligned}
  \underset{j = 1, 2, 3,  \cdots \rho}{min}  \mathbb{E}\left[  \| \vect{g}^{(j)}_{\vect{w}} \|^2\right]  &\leq  \frac{ 2\rho \mathbb{E} \left[\delta_u^2\right] (M+1) + 2(1+3\beta)G \delta_{w}2}{2\alpha_{w} \sqrt{\rho} - \alpha^2_{w} } + \frac{L_w( \alpha_{\vect{w}})^2 b}{ N (2 \alpha_{w} \sqrt{\rho} - \alpha^2_{w} )} \\ & +  \frac{ \beta  \rho b^2\bar{G}^2}{N^2 (2 \alpha_{w} \sqrt{\rho} - \alpha^2_{w}) } + \frac{G^2 (1+3\beta)^2 L_w(\alpha_{\vect{w}})^2 b^3}{(N^3 2\alpha_{w} \sqrt{\rho} - \alpha^2_{w} ) }. 
\end{aligned} \end{equation}
Under the condition that $2\alpha_{w} \sqrt{\rho} >> \alpha^2_{w},$  $\underset{j = 1, 2, 3,  \cdots \rho}{min}  \mathbb{E}\left[  \| \vect{g}^{(j)}_{\vect{w}} \|^2\right] $ converges to zero with the rate $\frac{1}{\sqrt{\rho}}$
\end{proof}

\bibliography{references}

\begin{thebibliography}{10}

\bibitem{verg}
{\sc E.~Beer, J.~A. Fill, S.~Janson, and E.~R. Scheinerman}, {\em On vertex, edge, and vertex-edge random graphs}, 2008, \url{https://doi.org/10.48550/ARXIV.0812.1410}, \url{https://arxiv.org/abs/0812.1410}.

\bibitem{bronstein2017geometric}
{\sc M.~M. Bronstein, J.~Bruna, Y.~LeCun, A.~Szlam, and P.~Vandergheynst}, {\em Geometric deep learning: going beyond euclidean data}, IEEE Signal Processing Magazine, 34 (2017), pp.~18--42.

\bibitem{cai2022multimodal}
{\sc J.~Cai, X.~Wang, C.~Guan, Y.~Tang, J.~Xu, B.~Zhong, and W.~Zhu}, {\em Multimodal continual graph learning with neural architecture search}, in Proceedings of the ACM Web Conference 2022, 2022, pp.~1292--1300.

\bibitem{carpenter1987massively}
{\sc G.~A. Carpenter and S.~Grossberg}, {\em A massively parallel architecture for a self-organizing neural pattern recognition machine}, Computer vision, graphics, and image processing, 37 (1987), pp.~54--115.

\bibitem{dantzig1956constructive}
{\sc G.~B. Dantzig}, {\em Constructive proof of the min-max theorem.}, Pacific Journal of Mathematics, 6 (1956), pp.~25--33.

\bibitem{egele2021agebo}
{\sc R.~Egele, P.~Balaprakash, I.~Guyon, V.~Vishwanath, F.~Xia, R.~Stevens, and Z.~Liu}, {\em Agebo-tabular: joint neural architecture and hyperparameter search with autotuned data-parallel training for tabular data}, in Proceedings of the International Conference for High Performance Computing, Networking, Storage and Analysis, 2021, pp.~1--14.

\bibitem{febrinanto2023graph}
{\sc F.~G. Febrinanto, F.~Xia, K.~Moore, C.~Thapa, and C.~Aggarwal}, {\em Graph lifelong learning: A survey}, IEEE Computational Intelligence Magazine, 18 (2023), pp.~32--51.

\bibitem{galke2021lifelong}
{\sc L.~Galke, B.~Franke, T.~Zielke, and A.~Scherp}, {\em Lifelong learning of graph neural networks for open-world node classification}, in 2021 International Joint Conference on Neural Networks (IJCNN), IEEE, 2021, pp.~1--8.

\bibitem{glorot2010understanding}
{\sc X.~Glorot and Y.~Bengio}, {\em Understanding the difficulty of training deep feedforward neural networks}, in Proceedings of the Thirteenth International Conference on Artificial Intelligence and Statistics, 2010, pp.~249--256.

\bibitem{grover2016node2vec}
{\sc A.~Grover and J.~Leskovec}, {\em node2vec: Scalable feature learning for networks}, in Proceedings of the 22nd ACM SIGKDD international conference on Knowledge discovery and data mining, 2016, pp.~855--864.

\bibitem{hamilton2017inductive}
{\sc W.~Hamilton, Z.~Ying, and J.~Leskovec}, {\em Inductive representation learning on large graphs}, Advances in neural information processing systems, 30 (2017).

\bibitem{jin2020local}
{\sc C.~Jin, P.~Netrapalli, and M.~Jordan}, {\em What is local optimality in nonconvex-nonconcave minimax optimization?}, in International Conference on Machine Learning, PMLR, 2020, pp.~4880--4889.

\bibitem{kipf2016semi}
{\sc T.~N. Kipf and M.~Welling}, {\em Semi-supervised classification with graph convolutional networks}, arXiv preprint arXiv:1609.02907,  (2016).

\bibitem{krishnan2020meta}
{\sc R.~Krishnan and P.~Balaprakash}, {\em Meta continual learning via dynamic programming}, https://arxiv.org/abs/2008.02219,  (2020).

\bibitem{lewis2012optimal}
{\sc F.~L. Lewis, D.~Vrabie, and V.~L. Syrmos}, {\em Optimal control}, John Wiley \& Sons, 2012.

\bibitem{lin1992self}
{\sc L.-J. Lin}, {\em Self-improving reactive agents based on reinforcement learning, planning and teaching}, Machine Learning, 8 (1992), pp.~293--321.

\bibitem{liu2021overcoming}
{\sc H.~Liu, Y.~Yang, and X.~Wang}, {\em Overcoming catastrophic forgetting in graph neural networks}, in Proceedings of the AAAI Conference on Artificial Intelligence, vol.~35, 2021, pp.~8653--8661.

\bibitem{morris2020tudataset}
{\sc C.~Morris, N.~M. Kriege, F.~Bause, K.~Kersting, P.~Mutzel, and M.~Neumann}, {\em Tudataset: A collection of benchmark datasets for learning with graphs}, arXiv preprint arXiv:2007.08663,  (2020).

\bibitem{7796926}
{\sc N.~{Patki}, R.~{Wedge}, and K.~{Veeramachaneni}}, {\em The synthetic data vault}, in 2016 IEEE International Conference on Data Science and Advanced Analytics (DSAA), Oct 2016, pp.~399--410, \url{https://doi.org/10.1109/DSAA.2016.49}.

\bibitem{perozzi2014deepwalk}
{\sc B.~Perozzi, R.~Al-Rfou, and S.~Skiena}, {\em Deepwalk: Online learning of social representations}, in Proceedings of the 20th ACM SIGKDD international conference on Knowledge discovery and data mining, 2014, pp.~701--710.

\bibitem{krishnan2021continual}
{\sc K.~Raghavan and P.~Balaprakash}, {\em Formalizing the generalization-forgetting trade-off in continual learning}, Advances in Neural Information Processing Systems, 34 (2021), pp.~17284--17297.

\bibitem{riemer2018learning}
{\sc M.~Riemer, I.~Cases, R.~Ajemian, M.~Liu, I.~Rish, Y.~Tu, and G.~Tesauro}, {\em Learning to learn without forgetting by maximizing transfer and minimizing interference}, arXiv preprint arXiv:1810.11910,  (2018).

\bibitem{scarselli2008graph}
{\sc F.~Scarselli, M.~Gori, A.~C. Tsoi, M.~Hagenbuchner, and G.~Monfardini}, {\em The graph neural network model}, IEEE transactions on neural networks, 20 (2008), pp.~61--80.

\bibitem{tang2020graph}
{\sc B.~Tang and D.~S. Matteson}, {\em Graph-based continual learning}, in International Conference on Learning Representations, 2020.

\bibitem{trivedi2018representation}
{\sc R.~Trivedi, M.~Farajtabar, P.~Biswal, and H.~Zha}, {\em Representation learning over dynamic graphs}, arXiv preprint arXiv:1803.04051,  (2018).

\bibitem{https://doi.org/10.5281/zenodo.4638695}
{\sc G.~Turinici}, {\em The convergence of the stochastic gradient descent (sgd) : a self-contained proof}, tech. report, 2021, \url{https://doi.org/10.5281/ZENODO.4638695}, \url{https://zenodo.org/record/4638695}.

\bibitem{velivckovic2017graph}
{\sc P.~Veli{\v{c}}kovi{\'c}, G.~Cucurull, A.~Casanova, A.~Romero, P.~Lio, and Y.~Bengio}, {\em Graph attention networks}, arXiv preprint arXiv:1710.10903,  (2017).

\bibitem{wang2020lifelong}
{\sc C.~Wang, Y.~Qiu, D.~Gao, and S.~Scherer}, {\em Lifelong graph learning}, arXiv preprint arXiv:2009.00647,  (2020).

\bibitem{wang2022streaming}
{\sc J.~Wang, W.~Zhu, G.~Song, and L.~Wang}, {\em Streaming graph neural networks with generative replay}, in Proceedings of the 28th ACM SIGKDD Conference on Knowledge Discovery and Data Mining, 2022, pp.~1878--1888.

\bibitem{wu2019simplifying}
{\sc F.~Wu, A.~Souza, T.~Zhang, C.~Fifty, T.~Yu, and K.~Weinberger}, {\em Simplifying graph convolutional networks}, in International conference on machine learning, PMLR, 2019, pp.~6861--6871.

\bibitem{9046288}
{\sc Z.~Wu, S.~Pan, F.~Chen, G.~Long, C.~Zhang, and P.~S. Yu}, {\em A comprehensive survey on graph neural networks}, IEEE Transactions on Neural Networks and Learning Systems, 32 (2021), pp.~4--24, \url{https://doi.org/10.1109/TNNLS.2020.2978386}.

\bibitem{xu2018powerful}
{\sc K.~Xu, W.~Hu, J.~Leskovec, and S.~Jegelka}, {\em How powerful are graph neural networks?}, arXiv preprint arXiv:1810.00826,  (2018).

\bibitem{yuan2023continual}
{\sc Q.~Yuan, S.-U. Guan, P.~Ni, T.~Luo, K.~L. Man, P.~Wong, and V.~Chang}, {\em Continual graph learning: A survey}, arXiv preprint arXiv:2301.12230,  (2023).

\bibitem{zhang2022hierarchical}
{\sc X.~Zhang, D.~Song, and D.~Tao}, {\em Hierarchical prototype networks for continual graph representation learning}, IEEE Transactions on Pattern Analysis and Machine Intelligence,  (2022).

\bibitem{zhou2021overcoming}
{\sc F.~Zhou and C.~Cao}, {\em Overcoming catastrophic forgetting in graph neural networks with experience replay}, in Proceedings of the AAAI Conference on Artificial Intelligence, vol.~35, 2021, pp.~4714--4722.

\end{thebibliography}

\begin{flushright}
  \scriptsize
  \framebox{\parbox{\textwidth}{
  The submitted manuscript has been created by UChicago Argonne, LLC, Operator of Argonne National Laboratory (“Argonne”).
  Argonne, a U.S. Department of Energy Office of Science laboratory, is operated under Contract No. DE-AC02-06CH11357.
  The U.S. Government retains for itself, and others acting on its behalf, a paid-up nonexclusive, irrevocable worldwide
  license in said article to reproduce, prepare derivative works, distribute copies to the public, and perform publicly
  and display publicly, by or on behalf of the Government.  The Department of Energy will provide public access to these
  results of federally sponsored research in accordance with the DOE Public Access Plan.
  \url{http://energy.gov/downloads/doe-public-access-plan}
  }}
  \normalsize
  \end{flushright}

\end{document}